%% file: main.tex
\documentclass{article}

\usepackage{arxiv}

\usepackage[utf8]{inputenc} 
\usepackage[T1]{fontenc}    
\usepackage{hyperref}       
\usepackage{url}            
\usepackage{booktabs}       
\usepackage{amsfonts}       
\usepackage{nicefrac}       
\usepackage{microtype}      
\usepackage{lipsum}
\usepackage{graphicx}
\usepackage{makecell}
\usepackage{xcolor}
\usepackage{amsmath}
\usepackage{amssymb}
\usepackage{mathtools}
\usepackage{amsthm}
\usepackage{subcaption}
\usepackage{wrapfig}
\usepackage{tikz}
\usetikzlibrary{arrows.meta, calc, decorations.pathreplacing}
\graphicspath{ {./images/} }

\theoremstyle{plain}
\newtheorem{theorem}{Theorem}[section]

\newtheorem{lemma}[theorem]{Lemma}

\theoremstyle{definition}

\newtheorem{assumption}[theorem]{Assumption}
\theoremstyle{remark}

\title{Bridging the Gap Between Average and Discounted TD Learning}

\author{
Haoxing Tian \\
Department of Electrical and Computer Engineering\\
Boston University \\
\texttt{tianhx@bu.edu} \\
\And
Zaiwei Chen \\
Edwardson School of Industrial Engineering \\
Purdue University \\
\texttt{chen5252@purdue.edu} \\
\And 
Ioannis Ch. Paschalidis \\
Department of Electrical and Computer Engineering\\ 
Boston University \\
\texttt{yannisp@bu.edu} \\
\And 
Alex Olshevsky \\
Department of Electrical and Computer Engineering\\ 
Boston University \\
\texttt{alexols@bu.edu} 
}

\begin{document}
\maketitle
\begin{abstract}
The analysis of Temporal Difference (TD) learning in the average-reward setting faces notable theoretical difficulties because the Bellman operator is not contractive with respect to any norm. This complicates standard analyses of stochastic updates that are effective in discounted settings. Although a considerable body of literature addresses these challenges, existing theoretical approaches come with limitations. We introduce a novel algorithm designed explicitly for policy evaluation in the average-reward setting, utilizing sampling from two Markovian trajectories. Our proposed method overcomes previous limitations by guaranteeing convergence to the unique solution of a properly defined projected Bellman equation. Notably, and in contrast to earlier work, our convergence analysis is uniformly applicable to both linear function approximation and tabular settings and does not involve explicit dimension-dependent terms in its convergence bounds. These results align with what is known to hold in the discounted setting. Furthermore, our algorithm achieves improved dependence on the problem's condition number, reducing the sample complexity from quartic, as in prior literature, to quadratic scaling, and thus matching the efficiency seen in the discounted setting.
\end{abstract}

\section{Introduction}
\label{sec: intro}

Reinforcement learning with an average-reward objective is well-suited for applications that focus on the long-term performance over an infinite horizon. This framework has proved valuable in a variety of domains, including control systems \cite{prieto2008ergodic,krishnamurthy2016partially}, telecommunications \cite{bertsekas2021data,altman2021constrained}, and production environments \cite{feinberg2012handbook}. While the discounted setting—emphasizing shorter-term returns—has received extensive theoretical attention, especially for Temporal-Difference (TD) learning~\cite{bhandari2018finite,chen2019performance,tian2023performance,srikant2019finite}, analogous understanding of TD methods under the average-reward criterion is comparatively less developed~\cite{tsitsiklis1999average,zhang2021finite,blaser2024almost}.

Technically, compared to the discounted setting the  challenge in the average-reward setting is that the Bellman operator for policy evaluation is not a contraction mapping with respect to any norm. Consequently, the solution to the corresponding Bellman equation is not unique. This significantly complicates the theoretical analysis, particularly in the model-free setting where the stochastic dynamics of the Markov Decision Process (MDP) are unknown and the agent can only obtain random samples through interaction with the environment.

We now make the above remarks more concrete by surveying existing results on average-reward TD and comparing them to their counterparts in the discounted setting. The first theoretical analysis of average-reward TD, due to \cite{tsitsiklis1999average}, addressed the challenges discussed above by introducing an assumption that guarantees uniqueness of the solution to the Bellman equation. In the setting where the true value function $V_{\pi}(s)$ of a policy $\pi$ is approximated as $V_{\pi}(s) \approx \phi(s)^\top \theta$, that is, as a linear combination of computable features stacked into the vector $\phi(s)$, \cite{tsitsiklis1999average} assumed that for any scalar $c \in \mathbb{R}$ and vector $\theta \in \mathbb{R}^d$, we have $\Phi \theta \neq c e$, where $\Phi$ is the matrix whose rows are the feature vectors $\phi(s)^\top$, and $e$ is the all-ones vector. While this assumption enabled the first convergence guarantee for average-reward TD, it is not satisfied even in the tabular case, where $\{\phi(s)\}$ form the canonical basis and $\Phi$ is the identity matrix. Several subsequent works \cite{yu2009convergence, zhang2021average, li2024stochastic} adopted the same assumption.

A more recent study \cite{blaser2024almost} established that the value function under a particular update rule converges to a sample-path-dependent fixed point without requiring the assumption $\Phi \theta \neq c e$ for all $\theta$, by leveraging the theory of Stochastic Krasnoselskii--Mann (SKM) iterations \cite{bravo2024stochastic}. However, since the limit is sample-path-dependent, it may differ across independent runs of the algorithm. A similar result was obtained in the linear function approximation setting in \cite{zhang2021finite}, but without a guarantee that the process converges to a point. A more recent paper \cite{haque2024stochastic} establishes convergence to a point, but the convergence rate includes explicit dependence on the dimension $d$ of the parameter vector $\theta$, a dependence that does not explicitly appear in the discounted case. 

The recent paper \cite{li2024stochastic} on actor-critic methods includes some results on policy evaluation in the linear approximation setting, using a fairly intricate nested-loop algorithm based on variance reduction. However, when specialized to the policy evaluation problem, that work also makes the $\Phi \theta \neq c e$ assumption, rendering it inapplicable to the tabular case. 
Moreover, we note that their sample complexity is stated in terms of the averaged iterate, whereas the other papers discussed above provide guarantees for the last iterate.

To summarize, while a number of papers have developed a convergence theory for average-reward TD, the existing literature consistently includes some mix of caveats compared to the discounted case, ranging from assumptions that exclude the tabular setting, to a lack of point-wise convergence guarantees for the underlying iterates, or explicit dimension dependence in the convergence bounds. Additionally, we note that the state-of-the-art results from \cite{zhang2021average, haque2024stochastic, chen2025non} exhibit \emph{quartic} dependence on the condition number, in contrast to the \emph{quadratic} scaling achieved in the discounted setting \cite{bhandari2018finite}. Alternatively, the result from \cite{li2024stochastic} does achieve quadratic scaling, but it relies on the assumption \(\Phi \theta \neq c e\) which prevents it from being applied to the tabular case.

In this work, we take a complementary approach: observing that the average-reward value function must satisfy a certain steady-state constraint, we formulate the solution to the Bellman equation as a constrained optimization problem. To solve this problem, we propose a new algorithm that leverages sampling from two independent Markov chains in each iteration. We provide a finite-sample analysis for this method using a relatively new technique known as ``gradient splitting'' \cite{liu2021temporal}. Finally, building upon ideas from the Gradient TD (GTD) \cite{sutton2008convergent}, we give a version of our algorithm which only uses a single Markov chain.

\begin{table*}[t] \label{table:comparison}
\centering
\caption{Comparison with previous work. We highlight several key distinctions between our work and prior research: \\ 
(1) In the table below, {\bfseries Linear Only} refers to the assumption that $e \not \in {\rm span}(\Phi)$, which rules out the tabular case. \\ 
(2) A complicating factor in comparing results is that different papers tend to have slightly different notions of condition numbers; this is why the table contains $\eta_1, \eta_2, \eta_3$ (see definitions in Eq. (\ref{eq:eta1firstdef}) to Eq. (\ref{eq:eta3firstdef})). A detailed discussion and precise definitions can be found in Appendix \ref{appendix: condition number}. We remark that our condition number $\eta_1$ is at least as good as the widely used $\eta_3$, i.e., 
$\eta_1 \geq \Omega(\eta_3),$ which means bounds based on $\eta_1^{-1}$ are at least as good as bounds based on $\eta_3^{-1}$;  and if we assume the stationary probability $\mu$ is not too far from uniform, we also have 
$\eta_1 \geq \Omega(\eta_2).$ See Appendix \ref{appendix: condition number} for details. \\ 
(3) We consider a convergence time to be independent of the dimension $d$ if its dependence on $d$ appears only through $\|\theta_0\|$ or $\|\theta^*\|$. \\ 
(4) We make the standard assumption that $\|\phi(s)\|\leq 1$ for all states $s$, which can be achieved by rescaling. If we instead make the assumption that every entry of $\phi(s)$ is  $O(1)$ (so that $\|\phi(s)\|=O(\sqrt{d})$), then our scaling with dimension would be $O(d)$. This modified assumption matches more closely the assumption made in \cite{haque2024stochastic}, which, unlike this work, analyzes the infinite-dimensional case. Our results are an improvement compared to the larger $\Omega\left(d^2\right)$ scaling in that work. \\
(5) \cite{li2024stochastic} includes a scaling with approximation error (\(\mathbb{E}\|W^* - \Phi \theta^*\|_D^2\)) that no other paper has. In addition, we also note that they use a variance reduction method, not an analogue of plain TD; we do not implement variance reduction and our scaling with condition number is similar to their method.\\
(6) In \cite{li2024stochastic}, the condition number is actually defined as \(\min_{y \neq e, \|y\|_D=1} y^\top D (I-P)y\). However, our analysis in Appendix \ref{appendix: condition number} suggests that \(\eta_3\) is actually greater than their condition number, offering an optimistic approximation of their sample complexity. \\ 
(7) In \cite{kim2025implicit}, the sample complexity is actually \(\tilde{O}(\epsilon^{-1} R_{\rm proj}^2 \eta_2^{-3})\), where \(R_{\rm proj}\) is the projection radius. However, the authors did not provide the choice of this radius. Our analysis in Appendix \ref{appendix: condition number} shows that \(R_{\rm proj}^2\) is actually \(O(1/\eta_3).\) \\ 
(8) In the discounted case, $\eta_{\rm discounted}$ can be defined as $(1-\gamma) \sigma_{\rm min}(\Phi^T D \Phi)$ where $\gamma$ is the discount factor \cite{bhandari2018finite}}.
\label{tb:comparison}
\begin{scriptsize}
\begin{sc}
\resizebox{\linewidth}{!}{\begin{tabular}{lcccccc}
\toprule
Reference & Setting & \makecell{Converges To A \\ Sample Independent Point}  & \makecell{Sample \\ Complexity} & Scaling with $d$ & Iterate & Method\\
\midrule
\cite{tsitsiklis1999average} & Linear Only & Yes & $\times$ & $\times$ & last & \makecell{coupled\\SA}\\
\cite{blaser2024almost} & Tabular Only & No & $\times$ & $\times$ & last & \makecell{coupled\\SA}\\
\cite{zhang2021finite} & Tabular \& Linear & No & $\tilde{O}(\epsilon^{-1}\eta_2^{-4})$ & No & last & \makecell{coupled\\SA}\\
\cite{haque2024stochastic} & Tabular \& Linear & Yes & $\tilde{O}(\epsilon^{-1}\eta_2^{-4})$& Yes & last & \makecell{coupled\\SA}\\
\cite{kim2025implicit} & Tabular \& Linear  & No & \(\tilde{O}(\epsilon^{-1} \eta_2^{-3}\eta_3^{-1}) \) & No & last & \makecell{coupled\\SA}\\
\cite{chen2025non} & Tabular \& Linear & No & $\tilde{O}(\epsilon^{-1}\eta_2^{-4})$& No & last & \makecell{coupled\\SA}\\
\cite{li2024stochastic} & Linear  Only & Yes & $\tilde{O}(\epsilon^{-1} \eta_3^{-2})$& No & average & \makecell{Variance\\Reduction}\\
\addlinespace[0.5em]
\makecell[l]{\textcolor{blue}{Our Two-Chain} \\ \textcolor{blue}{Algorithm}} & \textcolor{blue}{Tabular \& Linear} & \textcolor{blue}{Yes} & \textcolor{blue}{$\tilde{O}(\epsilon^{-1} \eta_1^{-2})$} & \textcolor{blue}{No} & \textcolor{blue}{last} & \makecell{\textcolor{blue}{coupled}\\\textcolor{blue}{SA}}\\
\addlinespace[0.25em]
\makecell[l]{{Discounted Case} \\ {(e.g., \cite{bhandari2018finite})}} & Tabular \& Linear & Yes & $\tilde{O}(\epsilon^{-1} \eta_{\rm discounted}^{-2})$ & No & last & \makecell{SA}\\
\bottomrule
\end{tabular}}
\end{sc}
\end{scriptsize}
\vskip -0.1in
\end{table*}

Our work improves upon the state of the art by simultaneously having all of the following features.

\begin{itemize}

    \item \textbf{Provably unique, sample-independent fixed point.}  
Our analysis shows that the iterate sequence $\{\theta_t\}$ converges almost surely to a single, deterministic solution $\theta^\star$ that does not depend on the random trajectory or on initialization.  

    \item \textbf{Tabular + Linear Function Approximation}: Our analysis is applicable to both the tabular and linear function approximation cases.
    In particular, we do not assume $e \notin {\rm span}(\Phi)$. 
        

    \item \textbf{Dependence on Dimensionality}: Our convergence bound does not have any explicit factors of $d$, the dimension of $\theta$. While all algorithms have terms like $\|\theta^*\|$ that might implicitly scale with dimension, our algorithm has no terms scaling with $d$ in addition to those. 

    \item \textbf{Last Iterate:} our results are based on the last iterate rather than an iterate averaging, matching the corresponding results in the discounted case. 

    \item \textbf{Scaling with Condition Number}: In \cite{zhang2021finite} and \cite{haque2024stochastic}, the dependence of the convergence time on the condition number is quartic as $O\left(\eta_2^{-4}\right)$, where $\eta_2$ is defined as
    \begin{equation} \label{eq:eta2firstdef} \eta_{2} = \min_{\|x\|=1, x^\top e = 0} \|\Phi x\|_{\rm Dir}^2,
    \end{equation}
    where 
    $\|\cdot\|$ is the Euclidean norm and $\|\cdot\|_{\rm Dir}$ denotes the Dirichlet seminorm, formally defined later in Section \ref{section: norms}. Similarly, in \cite{kim2025implicit}, the dependence of the convergence time on the condition number is also quartic as $O\left(\eta_2^{-3}\eta_3^{-1}\right)$, where $\eta_3$ is defined as 
    \begin{equation}\label{eq:eta3firstdef} 
     \begin{small}
    \begin{aligned}
    \eta_3 = & \left(\min_{\|x\|=1} x^\top \Phi^\top D \Phi x \right) \cdot \left( \min_{\langle y, e \rangle_D = 0, \|y\|_D=1} y^\top D (I-P)y \right),    
    \end{aligned}
        \end{small}
    \end{equation} 
    \noindent where $D$ is the diagonal matrix with the stationary distribution $\mu$ of the policy on the diagonal. These contrast unfavorably with the corresponding results in the discounted case, which scale with the square rather than fourth power of the condition number \cite{bhandari2018finite}. That being said, the definition of condition number is different within the discounted case, and it may be that a slightly different notion of condition number is needed for the average-reward setting. 
    
    Indeed, that is just what we show: for our algorithm convergence time scales quadratically as  $O\left(1/\eta_1^2\right)$, where $\eta_1$ is defined as 
    \begin{equation} \label{eq:eta1firstdef} \eta_1 = \min_{x:\, \Vert x \Vert = 1} \Vert \Phi x \Vert_{\rm Dir}^2 + (\mu^\top \Phi x)^2, \end{equation} 
    and $\mu$ is the stationary distribution of the policy. 

\item \textbf{No variance reduction techniques needed.}  
While variance-reduction schemes are well-known to improve sample complexity, they typically require multi-level structure (e.g., nested loops, periodic full-batch or long-trajectory reference estimates, and additional bookkeeping) and therefore differ substantially from the ``plain TD'' template.  
In contrast, our quadratic condition-number scaling is achieved with a single-timescale, simple stochastic-approximation update without using any variance reduction techniques, matching the algorithmic simplicity of discounted TD.
\end{itemize}

\section{Preliminaries}

This section introduces the necessary background on average-reward reinforcement learning to support the algorithm design and convergence analysis of TD learning presented in the subsequent sections.

\subsection{Markov Decision Processes (MDP)}

We consider an MDP defined by the tuple $(S, A, P_{\rm env}, r)$, where (i) $S$ is the finite state space, (ii) $A$ is the finite action space, (iii) $P_{\rm env} = (P_{\rm env}(s' \mid s, a))_{s, s' \in S,\, a \in A}$ is the transition probability kernel, and (iv) $r: S \times A \to \mathbb{R}$ is the reward function. Let $r_{\max} := \max_{s \in S,\, a \in A} |r(s,a)|$, which is finite since the state-action space is finite. 

A policy $\pi: S \times A \to \mathbb{R}$ is a function where $\pi(a \mid s)$ represents the probability of the agent taking action $a$ in state $s$. Throughout this paper, we focus exclusively on the policy evaluation problem and therefore assume the policy $\pi$ to be fixed and known. Under this fixed policy, we define the induced transition matrix $P$ as $P = (P(s' \mid s))_{s, s' \in S}$, where $P(s' \mid s) = \sum_{a \in A} P_{\rm env}(s' \mid s, a)\, \pi(a \mid s)$.

We make the following assumption regarding the policy $\pi$, which is standard in TD learning \cite{liu2021temporal, bhandari2018finite}.

\begin{assumption}
\label{a: markov chain}
The Markov chain with transition matrix $P$ is irreducible and aperiodic.
\end{assumption}

Under the above assumption, the Markov chain with transition matrix $P$ has a unique stationary distribution, denoted by $\mu$, which satisfies $\mu_{\min} := \min_{s} \mu(s) > 0$ \cite{levin2017markov}. Moreover, according to Theorem 4.9 in \cite{levin2017markov}, there exist constants $C>1$ and $\beta \in [0,1)$ such that 
\begin{equation}
\label{eq: mixing}
\|p_{\tau}(\cdot|s) - \mu\|_1 \le C \beta^\tau, \quad \forall\, \tau \ge 0,  s \in S,    
\end{equation}
where $p_{\tau}$ is the probability distribution of the state of this Markov chain after $\tau$ steps starting at \(s\). 


\subsection{The Long-Term Average Reward}
\label{subsec:average_reward}

We now discuss value functions within the average-reward framework, highlighting their role in policy evaluation. Let $r_s = \sum_{a} \pi(a \mid s)\, r(s,a)$ denote the expected reward in state $s$ under policy $\pi$. We also define $p_{s_0s}^{k}$ to be the probability that the agent is at state $s$ after $k$ steps starting from $s_0$. Then the value function $v_{s_0}(t)$ reflects the expected cumulative reward starting from state $s_0$ after $t$ transitions:
\begin{align*}
v_{s_0}(t) 
= r_{s_0} + \sum_{s} p_{s_0s} r_{s} + \cdots + \sum_{s} p_{s_0s}^{t-1} r_{s}.
\end{align*}
Defining $V_t = [v_s(t)]_{s \in S}$ as the vector stacking up the value function and $R = [r_s]_{s \in S}$ as the vector stacking up the expected rewards, the above equation can be compactly written as $V_t = \sum_{k=0}^{t-1} P^k R$. Under Assumption \ref{a: markov chain}, we have $\lim_{n \to \infty} P^n = e \mu^\top$ \cite{levin2017markov}. Let $g = \mu^\top R$ be the steady-state reward per unit of time. Then, the relative value function, denoted by $W^*$, is defined as
\begin{equation}
\label{eq: w*}
W^* = \lim_{t \to \infty} V_t - t g e = \lim_{t \to \infty} \sum_{k=0}^{t-1} (P^k - e \mu^\top) R.    
\end{equation}
Intuitively, the relative value function $W^*$ quantifies the long-term expected cumulative reward differences across states relative to the steady-state reward. It is a central quantity for policy evaluation in the average-reward setting, as it effectively centers the rewards to focus purely on differences due to transient dynamics. 

\subsection{Useful Norms}
\label{section: norms}
In this subsection, we introduce several useful norms which will play an important role in our analysis. 
The so-called $D$-norm and Dirichlet semi-norm have been previously shown to be very useful in TD-like analysis \cite{ollivier2018approximate,liu2021temporal}. Given a vector $f$ with the same number of entries as the number of states in the MDP, its $D$-norm is defined as 
\begin{equation}
\label{eq: d norm}    
\Vert f \Vert_D^2 = \langle f, f \rangle_D = \sum_{s} \mu(s) f(s)^2
\end{equation}
and its Dirichlet semi-norm is defined as 
\begin{equation}
\label{eq: dir norm}
\Vert f \Vert _{\rm Dir}^2 = \frac{1}{2} \sum_{s, s'} \mu(s) P(s'|s) (f(s)-f(s'))^2.    
\end{equation}
Intuitively, the Dirichlet seminorm measures the difference between $f$ and the all-ones vector $e$, but in a way that is adapted to the Markov chain with transition matrix $P$. Finally, throughout the paper, we will use $\Vert \cdot \Vert$ to denote the standard Euclidean norm. 

\subsection{Markov Noise}
\label{sec: mixing}

Let $ s_t $ denote the state at time step $ t $. Following standard practice, we consider two distinct sampling scenarios in this paper: (1) \textit{i.i.d.\ sampling}, where each state $ s_t $ is independently drawn from the stationary distribution $ \mu $; and (2) \textit{Markov sampling}, where the Markov chain starts at \(s_0\) and evolves according to the policy. Under Markov sampling, the states $ s_t $ remain marginally distributed according to $ \mu $, but exhibit temporal correlation across time steps.

\subsection{Linear Function Approximation}

In practical reinforcement learning applications, the state space $S$ is often extremely large, making it impractical to maintain a vector whose dimension scales with the number of states. To address this challenge, it is common to incorporate function approximation, in particular, a linear function approximator of the form $W = \Phi \theta$, where $\Phi \in \mathbb{R}^{n \times d}$ is the feature matrix and $\theta \in \mathbb{R}^d$ is the parameter vector. Additionally, denote the $s$-th row of the feature matrix by $\phi(s)^\top$. We assume, without loss of generality, that (i) the features are normalized so that $\max_{s} \Vert \phi(s) \Vert \le 1$, and (ii) the columns of $\Phi$ are linearly independent.

We further define 
\begin{equation}
\label{eq: eta}
\eta = \min_{x:\, \Vert x \Vert = 1} \Vert \Phi x \Vert_{\rm Dir}^2 + (\mu^\top \Phi x)^2, \quad \text{where } x \in \mathbb{R}^d
\end{equation} 
Intuitively, $\eta$ measures how close to zero $\Phi x$ can get: the first term measures the distance between $\Phi x$ and the all-ones vector, whereas the second term measures the (squared) distance between the weighted average $\mu^\top \Phi x$ and zero. In particular, under the assumption that the columns of $\Phi$ are linearly independent, we immediately have $\eta > 0$. The quantity $\eta$ will act as a condition number in our algorithms for average reward TD. 

\section{Algorithms}

We now introduce the two algorithms studied in this paper. The first, called the \textbf{double-chain algorithm}, uses two independent Markov chains. The second, the \textbf{single-chain algorithm}, uses only one.

\subsection{Double-Chain Algorithm: Motivation and Derivation}

Recall that in the discounted setting, TD-learning is designed to solve the projected Bellman equation \cite{tsitsiklis1996analysis}. To motivate our analysis, we next introduce a natural analogue of the projected Bellman equation in the average-reward setting.

Our goal is to compute \(W^*\) as defined in Eq.(\ref{eq: w*}). It is well known \cite{gallager1997discrete} that the relative value function $W^*$ satisfies two properties: 
\[
W^* + g e = P W^* + R \quad \text{ and } \quad \mu^\top W^* = 0.
\]
Define $\Pi = I - e \mu^\top$ as a projection onto the subspace $\{x \mid \langle x, e \rangle_{D} = 0\}$ in the inner product $\langle \cdot, \cdot \rangle_{D}$ and the Bellman operator $T_\pi$ such that $T_\pi W = R + PW$. 

We can write the two properties into an equivalent way:
\begin{equation} 
\label{eq:projected_bellman_first}
W^* = \Pi T_\pi W^*.
\end{equation}
It is known that   $W^*$ is the unique solution to this equation \cite{gallager1997discrete}. Since the matrix $D=\text{diag}(\mu)$ is invertible, the previous equation is equivalent to 
\begin{align*}
D \left( \Pi (R + P W^*) - W^* \right) =0,
\end{align*}
which can be further written as
\begin{equation} \label{eq:property of W*}
\begin{aligned}
D \left( R + P W^* - W^* \right) - \mu \mu^\top (R + W^*)=0
\end{aligned}
\end{equation}
using the explicit definition of $\Pi$.

A natural approach to solve $W^*$ from Eq.(\ref{eq:property of W*}) is to recursively perform the update
\begin{small}
\begin{align}\label{eq: tabular expected averaged reward update}
    W_{t+1} 
=  W_t + \alpha_t \left( D \left( R+PW_{t} - W_t \right) -  \mu \mu^\top(R+W_{t}) \right).  
\end{align}
\end{small}
Although the above iterative algorithm seems promising, it cannot be implemented directly, since the transition matrix $P$ and the reward function $R$ are unknown. In the remainder of this section, we develop a data-driven stochastic version of the algorithm presented in Eq. (\ref{eq: tabular expected averaged reward update}). Before delving into the details, we first introduce some notation.

We use $1(s = s_0) \in \mathbb{R}^{|S|}$ to denote the vector whose entries are $0$ except a \(1\) at position $s = s_0$. We also use $\mathbb{E}_\mu$ as expectation assuming that the state $s_t$ is drawn from stationary distribution $\mu$ while $s_{t+1}$ is still drawn according to the MDP with action taken according to policy $\pi$.

With this notation in place, we now describe the intuition behind the equations that we will write down. Keeping in mind that our goal is to provide a stochastic version of Eq. (\ref{eq: tabular expected averaged reward update}), the straightforward approach is to replace the unknown transition matrix $P$ with something depending on samples that has expectation $P$. This works, but it is the second term in Eq. (\ref{eq: tabular expected averaged reward update}) that causes some trouble: it is surprisingly not straightforward to find a quantity such that its expectation is $\mu \mu^\top(R+W_{t})$. 

Indeed, to form a stochastic estimator for $\mu \mu^\top(R+W_{t})$, observe that while $\mathbb{E}_\mu [1(s=s_t)] = \mu$ and $\mathbb{E}_\mu [r_{s_t}+W_t(s_t)] = \mu^\top (R+W_t)$, we cannot multiply these two estimators to obtain the result we want because $\mathbb{E}[XY] \neq \mathbb{E}[X] \mathbb{E}[Y]$ for random variables $X$ and $Y$ if they are not independent. This is known as the double sampling issue \cite{sutton2008convergent}. A natural and simple way to solve this issue is to sample two independent Markov Chains and base the two estimates on independent samples. 

Denoting the state of these two chains by $\{s_t\}$ and $\{\hat{s}_t\}$, respectively, we therefore consider the following update:
\begin{equation}
\label{eq: tabular stochastic averaged reward update}    
W_{t+1} = W_t + \alpha_t \left(f(s_t, \hat{s}_t, W_t) + g(s_t,s_t', W_t)\right),
\end{equation}
where 
\begin{equation*}  
\begin{aligned}
f(s_t, \hat{s}_t, W_t)[s] = & -1(s=\hat{s}_t)(r_{s_t}+W_t(s_t)), \quad g(s_t,s_t', W_t)[s] = & 1(s=s_t) \left( r_{s_t}+W_t(s_t')-W_t(s_t) \right),
\end{aligned}
\end{equation*}
 It is then indeed immediate that
\begin{align*}
\mathbb{E}_{\mu}  f(s_t, \hat{s}_t, W_t) = & -\mu \mu^\top (R+W_t) ,
\quad
\mathbb{E}_{\mu} g(s_t,s_t', W_t) = & D(R + P W_t - W_t),     
 \end{align*}
and therefore Eq. (\ref{eq: tabular stochastic averaged reward update}) is a stochastic version of Eq. (\ref{eq: tabular expected averaged reward update}). 

With linear function approximation, the natural generalization of Eq. (\ref{eq:projected_bellman_first}) becomes 
\begin{equation} \label{eq:projected_bellman_second}
\Phi \theta^* = \Pi_D \Pi T_{\pi} \Phi \theta^*
\end{equation}
where $\Pi_D = \Phi (\Phi^\top D \Phi)^{-1} \Phi^\top D$ is the projection onto the subspace spanned by the columns of $\Phi$ in $D$-inner product. Compared to Eq. (\ref{eq:projected_bellman_first}), this equation uses $\Phi \theta$ as an approximation and adds a projection to the column space of $\Phi$. We rewrite the above equation in the following form, which will be easier to implement:
\begin{equation}
\label{eq: fix point double chain}
\Phi^\top D (I - \Pi P) \Phi \theta^* = \Phi^\top D \Pi R.    
\end{equation}
The following lemma establishes the existence and uniqueness of $\theta^*$ whose proof can be found in Appendix \ref{appendix: contraction}. 
\begin{lemma}
\label{l: existence and uniqueness}    
The solution to the linear system $
\Phi^\top D (I - \Pi P) \Phi \theta = \Phi^\top D \Pi R$ exists and is unique. 
\end{lemma}

Based on this equation, a natural generalization from the tabular to the linear approximation case is therefore
\begin{equation}
\label{eq: linear stochastic averaged reward update}    
\theta_{t+1} = \theta_t + \alpha_t \left(f(s_t, \hat{s}_t, \theta_t) + g(s_t,s_t', \theta_t)\right),
\end{equation}
where
\begin{equation}
\label{eq: definition of f and g}
\begin{aligned}
f(s_t, \hat{s}_t, \theta_t) = & -(r_{s_t}+\phi(s_t)^\top \theta_t) \phi(\hat{s}_t) \\
g(s_t,s_t', \theta_t) = & \left( r_{s_t}+\phi(s_t')^\top \theta_t-\phi(s_t)^\top \theta_t \right) \phi(s_t).   
\end{aligned}
\end{equation}

\subsection{Single-Chain Algorithm}

It is natural to wonder whether we can perform the update using  only a single Markov chain. Inspired by the GTD method \cite{sutton2008convergent}, we propose a solution that does so. 

Our single-chain algorithm is based on the following observation. Our two-chain algorithm uses the term $\phi(\hat{s}_t)$ -- but what if, instead, we replace that by an estimate of the expectation $E[ \phi(\hat{s}_t)]$? Because $s_t$ is sampled from $\mu$, this expectation equals $\Phi^\top \mu$. We will therefore introduce a new variable $w_t$ which will converge to $\Phi^\top \mu$ and use it in place of $\phi(\hat{s}_t)$. 

Our algorithm is thus as follows:
\begin{equation}
\label{eq: onechain stochastic averaged reward update}    
\begin{aligned}
w_{t+1} = \textbf{Proj}_{R_w} \{w_t + \beta_t f(s_t, w_t)\},
\quad
\theta_{t+1} = \textbf{Proj}_{R_\theta}\{\theta_t + \alpha_t g(s_t, s_t', w_t, \theta_t)\},    
\end{aligned}
\end{equation}
where
\begin{equation}
\label{eq: definition of f and g in onechain case}    
\begin{aligned}
f(s_t, w_t) = & \phi(s_t) - w_t, \\
g(s_t, s_t', w_t, \theta_t) = & (r_{s_t} + \phi(s_t')^\top \theta_t - \phi(s_t)^\top \theta_t) \phi(s_t) - (r_{s_t} + \phi(s_t)^\top \theta_t) w_t.    
\end{aligned}
\end{equation}
Note that the first line clearly drives $w_t$ to $\Phi^\top \mu$ while the second line is identical to the double-chain method except $\phi(\hat{s}_t)$ has been replaced by $w_t$. 

To see where \(\{w_t\}\) and \(\{\theta_t\}\) will converge to in this case, we notice that, under i.i.d. sampling, 
\begin{align*}
\mathbb{E}_{s_t \sim \mu} f(s_t, w_t) = & \Phi^\top \mu - w_t  \\
\mathbb{E}_{s_t \sim \mu} g(s_t, s'_t, w_t, \theta_t) = & \Phi^\top D (R + P \Phi \theta_t - \Phi \theta_t) - \mu^\top (R + \Phi \theta_t) w_t.
\end{align*}
Therefore, we can define \(w^*\) and \(\theta^*\) such that
\begin{align*}
& \Phi^\top \mu - w^* = 0 \\
& \Phi^\top D (R + P \Phi \theta^* - \Phi \theta^*) - \mu^\top (R + \Phi \theta_t) w^* = 0.   
\end{align*}
The intuition behind this method is that  while $w_t$ is derived from the same trajectory as $s_t$, it is an average over past features. We can therefore expect that it will be essentially de-correlated from the instantaneous value $s_t$, so we will be able to estimate $E[XY] \approx E[X]E[Y]$ up to some error for the product of $X=w_t$ and $Y=r_{s_t} + \phi(s_t)^\top \theta_t$. Naturally, this will come at the cost of an increased convergence time because of the additional error incurred. 

\section{Main Results}

We now present our main results. We will consider both i.i.d. and Markov sampling, as well as both our double-chain and single-chain algorithms. 

\subsection{Convergence Results for the Double-Chain Method}
Our first result assumes that the two Markov chains $\{s_t\}$ and $\{\hat{s}_t\}$ are sampled i.i.d. from the stationary distribution $\mu$. We first consider the case where the stepsizes are constant \(\alpha_t = \alpha\). We define \(\tau_{\rm mix}\) as the mixing time \(\tau_{\rm mix} = \tau_{\rm mix}(\alpha)\). We also denote \(\delta_t = \theta_t - \theta^*\). Our first theorem considers the double-chain method.  
\begin{theorem}[Double-chain, i.i.d.\ sampling]
\label{t: linear iid}
Suppose Assumption~\ref{a: markov chain} holds. Consider the double-chain algorithm~\eqref{eq: linear stochastic averaged reward update} with i.i.d.\ sampling and constant stepsize $\alpha \le \eta/18$. Then 
\[
\begin{small}
\mathbb{E} \| \theta_T - \theta^* \|^2 
\;\le\; e^{-\alpha \eta T} \| \theta_0 - \theta^* \|^2 
\;+\; O\!\left( \frac{\alpha (r_{\max} + \|\theta^*\|)^2}{\eta} \right).
\end{small}
\] 
In particular, choosing $\alpha = \tilde{\Theta}(1/(\eta T))$ yields a convergence rate of $\tilde{O}(1/T)$ with sample complexity $\tilde{O}(\epsilon^{-1} \eta^{-2})$.
\end{theorem}

Our next theorem generalizes this result to Markov sampling, i.e., when $s_t, \hat{s}_t$ are sampled from two independent Markov chains. This requires further information on the mixing time of the Markov Chain as mentioned in Section \ref{sec: mixing}. We still choose a constant stepsize. 
We will require the constants \begin{small}
\begin{align*}
\label{eq: B and G def}    
B \;=\; 2\|\theta_0 - \theta^*\| + r_{\max} + \|\theta^*\|,
\quad
G \;=\; 42 B^2 + 30(r_{\max} + \|\theta^*\|)^2.
\end{align*}
\end{small} 

\begin{theorem}[Double-chain, Markov sampling, constant stepsize]
\label{t: linear markov}
Suppose Assumption~\ref{a: markov chain} holds. Consider the double-chain algorithm~\eqref{eq: linear stochastic averaged reward update} with Markov sampling and constant stepsize
\[
\alpha \;\le\; \frac{\eta B^2}{(3\tau_{\rm mix} + 1) G}.
\]
Then for all $T \ge \tau_{\rm mix}$,
\[
\mathbb{E} \| \theta_T - \theta^* \|^2 
\;\le\; e^{-2\alpha \eta (T - \tau_{\rm mix})} B^2
\;+\; \frac{\alpha G (3\tau_{\rm mix} + 1)}{2\eta}.
\]
Choosing $\alpha = \tilde{\Theta}(1/(\eta T))$ yields a convergence rate of $\tilde{O}(1/T)$ with sample complexity $\tilde{O}(\epsilon^{-1} \eta^{-2})$.
\end{theorem}

The rates of both of these theorems  match the state-of-the-art for TD learning in the discounted case \cite{bhandari2018finite} in terms of the scaling with the various parameters. 

A similar result can be obtained with a decaying stepsize $\alpha_t = a/(t+c_0)^\xi$. In this case, the mixing time is defined as \(\tau_{\rm mix} = \tau_{\rm mix}(\alpha_T)\). The next theorem formally handles this case. 

\begin{theorem}[Double-chain, Markov sampling, decaying stepsize]
\label{t: linear markov with decaying stepsize}
Suppose Assumption~\ref{a: markov chain} holds. Consider the double-chain algorithm~\eqref{eq: linear stochastic averaged reward update} with Markov sampling and stepsize $\alpha_t = a/(t + c_0)^\xi$ for $\xi \in (0,1]$, $a > 0$, and $c_0$ sufficiently large. Then for all $T \ge \tau_{\rm mix}$:
\begin{enumerate}

    \item If $\xi = 1$, then
    \begin{eqnarray*}
    \mathbb{E}\|\theta_T - \theta^*\|^2 
    & \;\le\; & O(1) \cdot \left( \frac{\tau_{\rm mix} + c_0}{T + c_0} \right)^{a\eta} \;+\; \tilde{O}\!\left( \frac{a^2}{(T + c_0)^{\min\{1, a\eta\}}} \right).
    \end{eqnarray*}

    \item If $\xi \in (0,1)$, let \(\Delta_T = (T+c_0)^{1-\xi}-(\tau_{\rm mix}+c_0)^{1-\xi}\) then
    \begin{eqnarray*}
    \mathbb{E}\|\theta_T - \theta^*\|^2
    & \;\le\; & O \left( \exp\!\left(-\frac{\eta a}{1-\xi}
\Delta_T\right) \right) \;+\; O\!\left( \frac{a}{ \eta (T + c_0)^\xi} \right).
    \end{eqnarray*}
\end{enumerate}
Choosing $a = \Theta(1/\eta)$ yields sample complexity $\tilde{O}(\epsilon^{-1/\xi} \eta^{-2/\xi})$ which reduces to $\tilde{O}(\epsilon^{-1} \eta^{-2})$ when $\xi=1$. The $O(1)$ terms depend polynomially on $\|\theta_0 - \theta^*\|$, $r_{\max}$, and $\|\theta^*\|$.
\end{theorem}

This is once again consistent with the state-of-the-art results in the discounted setting \cite{bhandari2018finite}, where an additional $O(1/\eta)$ factor is introduced when the distance is measured in the parameter space.

\subsection{Convergence Results for the Single-Chain Method}

For the single-chain algorithm given by Eq. (\ref{eq: onechain stochastic averaged reward update}), we can also establish similar results, albeit with worse scaling with respect to the condition number. We denote \(\delta_t^\theta = \theta_t - \theta^*\) and \(\delta_t^w = w_t - w^*\). We define \(\tau_{\rm mix}\) as the mixing time \(\tau_{\rm mix} = \tau_{\rm mix}(\min\{\alpha, \beta\})\).

\begin{theorem}[Single-chain, constant stepsize, Markov sampling]
\label{t: one chain constant}
Consider the single-chain algorithm in Eq.~\eqref{eq: onechain stochastic averaged reward update} with Markov sampling.
Assume constant stepsizes $\alpha_t\equiv \alpha>0$, $\beta_t\equiv \beta>0$, and let $\rho_0:=\beta/\alpha\le 1$.
Let $\tau_{\rm mix}:=\tau_{\rm mix}(\min\{\alpha,\beta\})$ denote the mixing time at accuracy level $\min\{\alpha,\beta\}$. Let $0<\alpha<\frac{1}{2\zeta}$ and 
let $\lambda>0$ satisfy
\begin{align*}
 & 0<\lambda^2<\frac{2\eta}{r_{\max}+2R_\theta}, \\   
 & \zeta:=\eta-\frac{\lambda^2(r_{\max}+2R_\theta)}{2}>0,
\end{align*}
and define
\[
\kappa:=1-2\alpha\zeta\in(0,1),
\qquad
G_1:=e^{-2\rho_0\alpha}\in(0,1),
\] where $
\rho:=\max\{\kappa,G_1\}\in(0,1)$.
Then for all $t\ge \tau_{\rm mix}$,
\begin{align*}
\mathbb{E}\|\theta_t-\theta^*\|^2
\le\;&
\frac{4\alpha (r_{\max}+2R_\theta)\,(t-\tau_{\rm mix})\,\rho^{\,t-\tau_{\rm mix}}}{\lambda^2} 
\;+\; 4R_\theta^2\,\kappa^{\,t-\tau_{\rm mix}}
\;+\;\frac{\alpha\,G_{\rm const}}{2\zeta},
\end{align*}
and where 
\begin{align*}
G_{\rm const}
:=\;& \frac{(16\tau_{\rm mix}+6)(r_{\max}+2R_\theta)}{\lambda^2}
\;+\; (88\tau_{\rm mix}+4)\,(r_{\max}+3R_\theta)^2.
\end{align*}
\end{theorem}

\paragraph{Sample Complexity.}
As discussed in Appendix~\ref{appendix: projection radius}, the projection radii can be chosen so that
$R_\theta = O(\eta'^{-1/2})$ and $R_w=O(1)$, where $\eta'=\lambda_{\min}(\Phi^\top D \Phi)$.
Moreover, Appendix~\ref{appendix: one chain constant} shows one may select $\lambda^2=\Theta(\eta'^{1/2}\eta)$, which implies
$\zeta=\Theta(\eta)$ (e.g., by taking $\lambda^2$ a fixed fraction of $2\eta/(r_{\max}+2R_\theta)$).
With these choices, the dominant part of $G_{\rm const}$ comes from the $(r_{\max}+2R_\theta)/\lambda^2$ term and scales as
$G_{\rm const}= \tilde O\!\big(\tau_{\rm mix}/(\eta'\eta)\big)$ (up to polynomial factors in $r_{\max}$ and $R_\theta$),
so the steady-state error term satisfies
\[
\frac{\alpha\,G_{\rm const}}{\zeta}
=\tilde O\!\left(\frac{\alpha\,\tau_{\rm mix}}{\eta'\eta^2}\right).
\]
Thus, choosing $\alpha=\tilde\Theta(1/(\eta T))$ yields a $\tilde O(1/T)$ rate (as in the two-chain case), but the overall
condition-number dependence is worse than quadratic: plugging $\alpha=\tilde\Theta(1/(\eta T))$ into the steady-state term gives
$\tilde O\!\big(\tau_{\rm mix}/(\eta'\eta^3 T)\big)$, which becomes quartic $\tilde O(1/(\eta^4 T))$ in the common regime where
$\eta'=\Theta(\eta)$ (the precise relation between $\eta'$ and $\eta$ is discussed in Appendix~\ref{appendix: condition number}).

Finally, we can also give a similar theorem for the single-chain method with decaying step-size 
\[ \alpha_t = \frac{a}{(t+c_0)^{\xi}},\] which parallels Theorem~\ref{t: linear markov with decaying stepsize} for the double-chain algorithm but with the $\eta^{-2}$ scaling replaced with $\eta^{-4}$. For reasons of space,  {\bfseries we state this in the appendix as  Theorem~\ref{t: one chain}}.

\noindent {\bf Remark:} Some previous works add an additional  $(r_t - g)^2$ to the error measure, e.g., \cite{zhang2021finite, haque2024stochastic}.  We observe that using  $r_t = (\sum_{i=0}^{t-1} r_{S_i}) / t $ as a stochastic estimate of $g$ allows us to obtain  
\begin{equation}
\label{eq: reward estimation}
\mathbb{E} \left[ (r_t - g)^2 \right] = O(1/T). 
\end{equation} 
\begin{wrapfigure}{l}{0.2\textwidth}
  \begin{minipage}{0.2\textwidth}
    \centering
\begin{tikzpicture}[scale=3, >=Stealth]
\draw[->, thick] (0,0) -- (1,0);
\draw[->, thick] (0,0) -- (0,1);
\draw[->, thick] (0,0) -- (0.25,0.75) node[pos=1, right] {$\nabla f(\theta)$};
\draw[->, thick] (0,0) -- (0.3,0.3) node[pos=1, above right] {$\theta - \theta^*$};
\draw[->, thick] (0,0) -- (0.72,0.25) node[pos=1, right] {$2h(\theta)$};
\draw[thick, purple] 
  (20:0.2) arc[start angle=20, end angle=45, radius=0.2];
\draw[thick, green!70!black] 
  (45:0.2) arc[start angle=45, end angle=70, radius=0.2];
\end{tikzpicture}
\caption{Illustration of key property of the gradient splitting. The splitting vector $h(\theta)$ has the same inner product (up to a factor of $2$) as the true gradient with the vector $\theta-\theta^*$.}
\label{fig:gsplit}
  \end{minipage}
\end{wrapfigure}
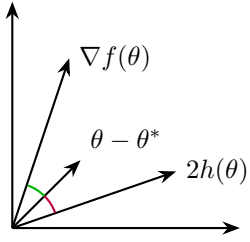
Thus, this quantity could easily be estimated separately without affecting the results of this section.
A detailed discussion is in Appendix \ref{appendix: condition number}. 

\subsection{Proof Idea}

Our main observation is that the methods we propose can be analyzed in much the same way as standard {\em Stochastic Gradient Descent (SGD)}, which is usually more tractable than stochastic approximation-based methods. The critical tool enabling this perspective is the notion of a \emph{gradient splitting}, introduced by \cite{liu2021temporal}.

To illustrate the idea, let us consider a convex quadratic function $f(\theta) \;=\; (\theta - \theta^*)^T \,A\, (\theta - \theta^*)$, where $A$ is symmetric and positive definite; and let us also consider a linear function $h(\theta) \;=\; B\,(\theta - \theta^*)$. We will say $h(\theta)$ is a gradient splitting of $f(\theta)$ if $B + B^T=2A$. In other words, each $(i,j)$-entry of $A$ can be split between the $(i,j)$ and $(j,i)$-entries of the (generally non-symmetric) matrix $B$. This decomposition is not unique, since there are many ways to split the entries of $A$.

An immediate implication of this definition is that for any $\theta$, $(\theta^* - \theta)^T \nabla f(\theta)=2 (\theta^* - \theta)^Th(\theta)$. Hence, updating in the direction $-h(\theta)$ behaves just like an update in the direction of the true gradient $-\nabla f(\theta)$. An illustration is given in Figure \ref{fig:gsplit}. One can further conclude
\begin{align*}
2 (\theta - \theta^*)^Th(\theta) 
= & (\theta^* - \theta)^T \nabla f(\theta) \\
= & (\theta^* - \theta)^T (\nabla f(\theta) - \nabla f(\theta^*)) \\
= & 2 f(\theta), 
\end{align*}
where the last equality holds because $f(\theta)$ is a quadratic function.

Concretely, the update rule
\begin{equation}
\label{eq:gs_update}
\theta_{t+1} \;=\; \theta_t 
\;-\; \alpha_t \,\bigl[h(\theta_t) + w_t\bigr],
\end{equation}
where $w_t$ is zero-mean i.i.d. noise, exhibits convergence properties analogous to classical SGD. In fact, using a Taylor expansion, 
\begin{equation*}
E\!\bigl[\|\theta_{t+1} - \theta^*\|_2^2 \big| \theta_t \bigr]
=
\|\theta_t - \theta^*\|_2^2
-\alpha_t\,\nabla f(\theta_t)^T(\theta_t - \theta^*)
+ O(\alpha_t^2).
\end{equation*}
Notice that the middle term is precisely what one would get with a standard gradient step, even though \eqref{eq:gs_update} does not explicitly use $\nabla f(\theta_t)$. Thus, standard SGD analysis applies, with adjustments for the different higher-order $O(\alpha_t^2)$ terms.

Building on these ideas, our contribution in this paper is to show that an algorithm for the average reward case can also be written as a gradient splitting. Specifically, consider Eq. (\ref{eq: linear stochastic averaged reward update}), where the expected update direction is $E_{\mu} \bigl[f(s_t, \hat{s}_t)\bigr]+g(s_t, s_t')$, with notation as in Eq. (\ref{eq: linear stochastic averaged reward update}). We can rewrite this expected update as $\Phi^TD \Bigl(R + P\Phi\theta - \Phi\theta\Bigr)-\Phi^T\mu\mu^T \bigl(R + \Phi\theta\bigr)$; our key technical observation is to show that  it serves as a gradient splitting for the composite function \[ \|\Phi(\theta - \theta^*)\|_{\mathrm{Dir}}^2+\bigl(\mu^T \Phi(\theta - \theta^*)\bigr)^2. \]  Once this gradient-splitting viewpoint is in place, we can leverage the standard SGD descent recursion, but closing the argument requires new bounds that control the discrepancy between our stochastic update and the true gradient—specifically, we must show the resulting 
$O\left(\alpha_t^2\right)$ and bias terms remain uniformly small/summable under the chosen stepsizes.

\section{Numerical Results} 

We compare our algorithms with prior work on fifteen tabular MDPs from OpenAI Gym and MO-Gymnasium. We focus on tabular environments because the exact solution can be computed, allowing us to directly quantify how accurately each method approximates the true solution. This enables a fair comparison between algorithms that converge to a single point and those that converge only to a set, since we evaluate all methods using their approximation error relative to the true solution rather than properties of their iterates. 

Across all fifteen environments, we evaluate both our Double-Chain and Single-Chain algorithms, together with the baselines listed in Table~\ref{tb:comparison}. Overall, we observe that our methods outperform the prior literature in most settings. Additional experimental details are provided in Appendix~\ref{sec:numerical}. 

\section{Conclusion}

The theoretical analysis of average reward TD  has traditionally lagged behind the discounted setting due to  mathematical difficulties, notably the non-contractive nature of the Bellman operator. This disparity manifested in prior works through various limitations: assumptions incompatible with the tabular case, convergence guarantees only to sets or sample-dependent points, explicit dimension scaling, and slower (quartic) convergence rates with respect to condition numbers -- which, unlike in the discounted case, could only be removed with more intricate algorithms like variance reduction.

In this work, we employed the gradient splitting technique to  provide a finite-sample analysis that closes the gap with discounted TD theory. Our methods guarantee convergence to a unique, well-defined solution for both tabular and linear approximation cases. Furthermore, the convergence bounds are dimension-free (in the standard sense) and exhibit quadratic scaling with the relevant condition number, mirroring the performance characteristics known for discounted TD. This contribution removes the persistent caveats associated with average reward TD analysis.

\bibliographystyle{unsrt}  
\bibliography{references}  

\newpage
\appendix
\include{sections/radius}
\include{sections/comparison}
\include{sections/double_constant}
\include{sections/double_decaying}

\include{sections/one_constant}
\include{sections/one_decaying}
\include{sections/numerical_result}

\end{document}

%% file: sections/radius.tex
\section{Bellman Operator and Projection}
\label{appendix: contraction}

In this section, we establish the contractivity of $\Pi T_{\pi}$ and prove the existence and uniqueness of the solution to Eq.\eqref{eq:projected_bellman_first}. Based on the contraction factor, we subsequently determine the choice of the projection radius used in the single-chain algorithm.

\subsection{Contraction of $\Pi T_{\pi}$}

In this section, we prove that $\Pi T_{\pi}$ is a contractive operator (recall that $\Pi = I - e \mu^T)$. 

\begin{lemma}
\label{l: contraction}
The operator $\Pi T_\pi$ is a contractive operator satisfying
\begin{align*}
\Vert \Pi T_{\pi} W_1 - \Pi T_{\pi} W_2 \Vert_{D} \le  \omega \Vert W_1 - W_2 \Vert_D,
\end{align*}
where $\omega = \sqrt{\max_{\langle z, e \rangle_{D} = 0, \Vert z \Vert_D = 1} z^\top P^\top D P z} < 1$. 
\end{lemma}

\begin{proof}[Proof of Lemma \ref{l: contraction}]
First, we notice that, since $P$ is irreducible according to Assumption \ref{a: markov chain}, the eigenvector whose eigenvalue is $1$ is unique (up to a constant factor) and must be the all-one vector.    

For any vector $\Delta = W_1 - W_2$, it can be decomposed into $\Delta = \Delta_{\parallel} + \Delta_{\perp}$ where $\langle \Delta_{\perp}, e \rangle_{D} = 0$ and $\langle \Delta_{\parallel}, e \rangle_{D} = \Vert \Delta_{\parallel} \Vert_D \cdot \Vert e \Vert_D$. With such decomposition, 
\begin{align*}
\Vert \Pi T_{\pi} W_1 - \Pi T_{\pi} W_2 \Vert_{D} = \Vert \Pi  P \Delta_{\perp} \Vert_D.
\end{align*}
Therefore, 
\begin{align*}
\Vert \Pi T_{\pi} W_1 - \Pi T_{\pi} W_2 \Vert_{D}^2 \le \max_{\langle \Delta_{\perp}, e \rangle_{D} = 0} \Delta_{\perp}^\top P^\top D P \Delta_{\perp} \le \max_{\langle z, e \rangle_{D} = 0, \Vert z \Vert_D = 1} z^\top P^\top D P z \cdot \Vert \Delta \Vert_D^2,
\end{align*}
which indicates the contraction factor is 
\[
\omega := \sqrt{\max_{\langle z, e \rangle_{D} = 0, \Vert z \Vert_D = 1} z^\top P^\top D P z}. 
\]
It is easy to see that this factor cannot be larger than $1$. If it is exactly $1$, then there exists $z$ such that 
\begin{align*}
\Vert z \Vert_D = 1 = \max_{\langle z, e \rangle_{D} = 0, \Vert z \Vert_D = 1} z^\top P^\top D P z = \Vert P z \Vert_D. 
\end{align*}
This indicates that $z$ must be a multiple of all-one vector, which contradicts with $\langle z, e \rangle_{D} = 0$. 
\end{proof}

\subsection{Existence and Uniqueness of $\theta^*$}

Recall that in Lemma \ref{l: existence and uniqueness}, we define $\theta^*$ as the solution to the linear system $\Phi^\top D (I-\Pi P) \Phi \theta = \Phi^\top D R$ and we claim that $\theta^*$ exists and is unique. In this section, we give the proof of this lemma. 

\begin{proof}[Proof of Lemma \ref{l: existence and uniqueness}]
Recall that in Eq.(\ref{eq: eta}) we already established that 
\begin{align*}
\eta = \min_{x:\, \Vert x \Vert = 1} \Vert \Phi x \Vert_{\rm Dir}^2 + (\mu^\top \Phi x)^2 > 0.
\end{align*}
Observe that 
\begin{align*}
\Vert \Phi x \Vert_{\rm Dir}^2 + (\mu^\top \Phi x)^2 = x^\top \Phi^\top D(I-P)\Phi x + x^\top \Phi^\top \mu \mu^\top \Phi x = x^\top \Phi^\top D (I-\Pi P) \Phi x.
\end{align*}
Suppose $\Phi^\top D (I-\Pi P) \Phi$ is not invertible. Then there exists a nonzero vector $x_0$ such that $x_0^\top \Phi^\top D (I-\Pi P) \Phi x_0 = 0$, which contradicts the fact that $\eta > 0$. Therefore, the matrix $\Phi^\top D (I-\Pi P) \Phi$ must be invertible. It follows that the linear system admits a unique solution $\theta^*$.
\end{proof}

\subsection{Choice of Projection Radius}
\label{appendix: projection radius}

Recall that in our single chain algorithm, we need to project both $w$ and $\theta$ onto a ball of radius $R_w$ and $R_\theta$, respectively. In this section, we will discuss on the choice of radius such that $w^*$ and $\theta^*$ are in the feasible set. 

\textbf{Bound on $\|w^*\|$: } This is easier because we can explicitly write $w^* = \Phi^\top \mu$. Using the fact that $\Vert \phi(s) \Vert \le 1$ we can conclude $\Vert w^* \Vert \le 1$. Therefore, we need $R_w \ge 1$. 

\textbf{Bound on $\|\theta^*\|$: } Recall that $W^*$ is defined as
\begin{align*}
W^* = \lim_{t \to \infty} \sum_{k=0}^{t-1} (P^k-e\mu^\top) R.
\end{align*}
According to Eq.(\ref{eq: mixing}), 
\begin{align*}
([P^k]_i - \mu^\top)R \le \Vert[P^k]_i - \mu^\top \Vert_1 \cdot \Vert R \Vert_{\infty} \le r_{\rm max} C \beta^k,
\end{align*}
which indicates $\Vert W^* \Vert_D \le r_{\rm max} C / (1-\beta)$. 

By the Pythagorean theorem,
\begin{align*}
\Vert \Phi \theta^* - W^* \Vert^2_D 
= & \Vert W^* - \Pi_D \Pi T_\pi W^* \Vert^2_D + \Vert \Pi_D \Pi T_\pi W^* - \Phi \theta^* \Vert^2_D \\
\le & \Vert W^* - \Pi_D \Pi T_\pi W^* \Vert^2_D + \omega^2 \Vert \Phi \theta^* - W^* \Vert^2_D, 
\end{align*}
where we use the fact $\Pi T_\pi \Phi \theta^* = \Phi \theta^*$ and Lemma \ref{l: contraction}. Therefore, 
\begin{align*}
\Vert \Phi \theta^* - W^* \Vert_D \le \frac{1}{\sqrt{1-\omega^2}} \Vert W^* - \Pi_D \Pi T_\pi W^* \Vert_D.
\end{align*}
Further, according to the Pythagorean theorem, $\Vert W^* - \Pi_D \Pi T_\pi W^* \Vert_D \le \Vert W^* \Vert_D$. Therefore, 
\begin{align*}
\Vert \Phi \theta^* \Vert_D \le \Vert W^* \Vert_D + \Vert \Phi \theta^* - W^* \Vert_D \le \frac{2}{\sqrt{1-\omega^2}} \Vert W^* \Vert_D \le \frac{2 r_{\rm max} C}{(1-\beta)\sqrt{1-\omega^2}}.
\end{align*}
This suggests that
\begin{align*}
\Vert \theta^* \Vert \le \frac{1}{\sqrt{\lambda_{\rm min}(\Phi^T D \Phi)}} \Vert \Phi \theta^* \Vert_D \le \frac{2 r_{\rm max} C}{(1-\beta)\sqrt{(1-\omega^2)\lambda_{\rm min}(\Phi^\top D \Phi)}}.
\end{align*}
Therefore, to ensure $\theta^*$ is in the feasible set, we can set $R_\theta$ such that
\begin{align*}
R_\theta \ge \frac{2 r_{\rm max} C}{(1-\beta)\sqrt{(1-\omega^2)\lambda_{\rm min}(\Phi^\top D \Phi)}}.
\end{align*}

%% file: sections/comparison.tex
\section{Comparison with Previous Works}
\label{appendix: condition number}

In this section, we will expand our result on reward estimation and compare the difference between condition numbers with previous works. 

\subsection{Reward Estimation}

Many previous works \cite{zhang2021finite, haque2024stochastic} also include convergence of the averaged reward function. Although not stated in our theorems, we can also achieve such convergence as mentioned in Eq.(\ref{eq: reward estimation}). In this section, we give a proof of Eq.(\ref{eq: reward estimation}) based on a Central Limit Theorem for Markov chains. 

\begin{proof}
For all \(t = \tau_{\rm mix}, \ldots, T\), we can decompose \(\mathbb{E} \left[ (r_t - g)^2 \right] \) as
\[
\mathbb{E} \left[ (r_t - g)^2 \right] \le 2 \mathbb{E} \left[ (r_t - \mathbb{E}[r_t])^2 \right] + 2 \mathbb{E} \left[ (\mathbb{E}[r_t] - g)^2 \right] = 2 {\rm Var}(r_t) + 2 (\mathbb{E}[r_t] - g)^2.
\]
First, we bound the variance of \(r_t\), \({\rm Var}(r_t) = \mathbb{E} \left[ (r_t - g)^2 \right] \). Let \(\gamma_k = {\rm Cov}(r_{S_0}, r_{S_k})\) be the covariance function. It is well known that
\[
{\rm Var}(r_t) = \frac{1}{t^2} \sum_{i=0}^{t-1} \sum_{j=0}^{t-1} \gamma_{|i-j|}.
\]
By changing the order of summation, 
\[
\sum_{i=0}^{t-1} \sum_{j=0}^{t-1} \gamma_{|i-j|} = \sum_{i=-(t-1)}^{t-1} (t-|i|) \gamma_{|i|} \le n \sum_{-\infty}^{\infty} \gamma_i = t \left( \gamma_0 + 2 \sum_{i=1}^{\infty} \gamma_i \right).
\]
Since all Markov chains considered in this paper is \(V\)-uniformly ergodic with \(V = 1\) and the reward function is also bounded since the Markov chain is finite, the conditions of Theorem 17.0.1 in \cite{meyn2012markov} hold (up to a constant factor which will not change the final result) and therefore, the term \(\gamma^2 := \left( \gamma_0 + 2 \sum_{i=1}^{\infty} \gamma_i \right)\) must be finite. We further conclude
\[
{\rm Var}(r_t) \le \frac{1}{t} \gamma^2 = O\left( \frac{1}{t} \right).
\]
Next, we bound the other term on the right hand side, \(2 (\mathbb{E}[r_t] - g)^2\), using Eq.(\ref{eq: mixing}). Since each reward is at most \(r_{\rm max}\), the averaged reward \(g\) can also be at most \(r_{\rm max}\). Therefore,
\[
| \mathbb{E}[r_{\tau_{\rm mix}}] - g | = \frac{1}{\tau_{\rm mix}}\sum_{i=0}^{\tau_{\rm mix}-1} (r_{S_i} - g) \le 2 r_{\rm max}.
\]
Since we also have \(r_t = \tau_{\rm mix}r_{\tau_{\rm mix}}/t + \sum_{i=\tau_{\rm mix}}^{t-1} r_{t}/t\), 
\begin{align*}
|\mathbb{E}[r_t] - g| 
\le & \frac{\tau_{\rm mix}}{t} |\mathbb{E}[r_{\tau_{\rm mix}}] - g| + \frac{1}{t} \sum_{i=\tau_{\rm mix}}^{t-1} |\mathbb{E}[r_i] - g|    \\
\le & \frac{2 \tau_{\rm mix}}{t} r_{\rm max} + \frac{r_{\rm max}}{t} \sum_{i=\tau_{\rm mix}}^{t-1} \| P_i(\cdot|s_0) - \mu\|_1 \\
\le & \tilde{O} \left( \frac{1}{t} \right).
\end{align*}
Therefore, combining both bounds and we conclude that 
\[
\mathbb{E} \left[ (r_T - g)^2 \right] = O(1/T). 
\]
\end{proof}

\subsection{Condition Numbers}

We notice that, in \cite{zhang2021finite} and \cite{li2024stochastic}, the definitions on condition number are both different from ours. Our condition number is defined by
\[
\eta_{1} = \min_{||x||=1} \left\{ ||\Phi x||_{\rm Dir}^2 + (\mu^\top \Phi x)^2 \right\}.
\]
Meanwhile, in \cite{zhang2021finite}, 
\[
\eta_{2} = \min_{||x||=1, x^T e = 0} ||\Phi x||_{\rm Dir}^2
\]
and in \cite{li2024stochastic, kim2025implicit}, 
\[
\eta_3 = \left(\min_{\|x\|=1} x^\top \Phi^\top D \Phi x \right) \cdot \left( \min_{\langle y, e \rangle_D = 0, \|y\|_D=1} y^\top D (I-P)y \right).
\]
Therefore, in this section we will describe this difference. 

\textbf{Remark: } In \cite{li2024stochastic}, the condition number is actually defined as \(\min_{y \neq e, \|y\|_D=1} y^\top D (I-P)y\). However, we can always decompose it as \(y = y_\perp + y_\parallel\) where $\langle y_{\perp}, e \rangle_{D} = 0$ and $\langle \Delta_{\parallel}, e \rangle_{D} = \Vert y_{\parallel} \Vert_D \cdot \Vert e \Vert_D$. We can easily check \(y^\top D (I-P)y = y_{\perp}^\top D (I-P)y_{\perp} \). Therefore, \(\eta_3\) is actually greater than the condition number defined in \cite{li2024stochastic}, offering an optimistic approximation of their sample complexity. 

\subsubsection{Difference between \(\eta_1\) and \(\eta_2\)}

\begin{lemma}
\label{l: condition number}
We have
\[
\eta_1 \ge \frac{\mu_{\rm min}}{n \mu_{\rm max} \Vert \mu \Vert^2}\eta_2.
\]
\end{lemma}

Notice that \(\eta_1\) defined in the above lemma corresponds to the condition number in this paper and \(\eta_2\) corresponds to the condition number in \cite{zhang2021finite}. We can conclude that when \(\mu\) is a multiple of the all-one vector, them \(\eta_1\) is no less than \(\eta_2\), which suggests our result is better. In other cases it is difficult to compare them. 

Now we provide the proof to Lemma \ref{l: condition number}. 

\begin{proof}[Proof of Lemma \ref{l: condition number}]
It is useful to define the $\mu$--weighted inner product and norm by
\[
\langle x,y\rangle_\mu=\sum_{s\in\mathcal{S}}\mu(s)x(s)y(s),\quad
\|x\|_\mu=\sqrt{\langle x,x\rangle_\mu}.
\]
We first define \(\eta_1', \eta_2'\) to be
\[
\eta_1'=\min_{\|x\|_\mu=1}\left\{\,\|x\|_{\mathrm{Dir}}^2+(\mu^\top x)^2\right\},
\]
and
\[
\eta_2'=\min_{\|x\|_\mu=1, \mu^\top x=0} \|x\|_{\mathrm{Dir}}^2.
\]
We can show a simple fact that \(\eta_1' \ge \eta_2'\) when \(\eta_2' \le 1 \). We notice that the Dirichlet semi-norm is invariant under addition of a constant. In other words, for any $x\in\mathbb{R}^n$ and \(c \in \mathbb{R}\), 
\[
\Vert x + c e \Vert_{\rm Dir} = \Vert x \Vert_{\rm Dir}.
\]
Therefore, we can decompose
\[
x=ce+z,\quad\text{where } c\in\mathbb{R}\text{ and } \mu^\top z=0.
\]
On one hand, we obtain \(\Vert x \Vert_{\rm Dir} = \Vert z \Vert_{\rm Dir}\). On the other hand, 
\[
\|x\|_\mu^2=\|ce\|_\mu^2+\|z\|_\mu^2=c^2\|e\|_\mu^2+\|z\|_\mu^2=c^2+\|z\|_\mu^2=1.
\]
Therefore,
\[
\|x\|_{\mathrm{Dir}}^2+(\mu^\top x)^2 = \|z\|_{\mathrm{Dir}}^2+(\mu^\top x)^2 \ge \eta_2'\,(1-c^2)+c^2 = \eta_2 + c^2 (1-\eta_2).
\]
where the second equation uses the fact that \(\|z\|_{\mathrm{Dir}}^2\ge \eta_2' \|z\|_\mu^2\). We conclude that \(\eta_1' \ge \eta_2'\) as long as \(\eta_2' \le 1\). 

Now we return to the original problem. To distinguish, we denote \(\eta_1\) to be what is defined in our paper. Namely, 
\[ 
\eta_{1} = \min_{||x||=1} \left\{ ||x||_{\rm Dir}^2 + (\mu^\top x)^2 \right\}. 
\]
According to \cite{zhang2021finite}, we define \(\eta_2\) to be
\[
\eta_{2} = \min_{||x||=1, x^\top e = 0} ||x||_{\rm Dir}^2.
\]
We first notice the simple fact
\[ \mu_{\max}^{-1} \min_{||x||=1} x^\top A x = \frac{x^\top A x}{\mu_{\rm max} ||x||_2^2}  \leq \min_{x: ||x||_{\mu}=1} x^\top A x \leq \frac{x^\top A x}{\mu_{\rm min} ||x||_2^2} = \mu_{\min}^{-1} \min_{||x||=1} x^\top A x. \]

We introduce the following lemma:
\begin{lemma}
Let $A\in\mathbb{R}^{n\times n}$ be a symmetric positive semi-definite matrix satisfying \(Ae = 0\). Let $\mu\in\mathbb{R}^n$ be a vector with strictly positive entries satisfying
\[
0 < \mu_{\min} \le \mu_i \le \mu_{\max}\quad\text{for }i=1,\dots,n,\qquad\text{and}\qquad \sum_{i=1}^n \mu_i = 1.
\]
Define
\[
\lambda := \min_{\|x\|_2=1,  x^\top e=0} x^\top A x,\quad \lambda^{(\mu)} := \min_{\|x\|_2=1, x^\top\mu=0} x^\top A x
\]
and 
\[
\lambda' := \min_{\|x\|_2=1} x^\top A x, \quad {\lambda^{(\mu)}}' := \min_{\|x\|_\mu=1} x^\top A x.
\]
Then we conclude
\[
\frac{\lambda}{n ||\mu||_2^2} \le \lambda^{(\mu)}, \quad  \mu_{\rm max}^{-1} \lambda' \le  {\lambda^{(\mu)}}' \le \mu_{\rm min}^{-1} \lambda'.
\]     
\end{lemma}
With the above lemma, one can show
\[
\eta_2' \ge \frac{1}{n \mu_{\rm max} \Vert \mu \Vert^2}\eta_2
\]
and
\[
\eta_1' \le \mu_{\rm min}^{-1}\eta_1.
\]
Therefore, 
\[
\eta_1 \ge \frac{\mu_{\rm min}}{n \mu_{\rm max} \Vert \mu \Vert^2}\eta_2.
\]
Now, we only need to prove the above lemma. We split the proof into two parts.

First, we show that \(\mu_{\rm max}^{-1} \lambda' \le  {\lambda^{(\mu)}}' \le \mu_{\rm min}^{-1} \lambda'\). We can always rewrite \(\lambda'\) and \({\lambda^{(\mu)}}'\) as
\[
\lambda' = \min_{x^\top \mu = 0} \frac{x^\top A x}{\Vert x\Vert^2}, \quad {\lambda^{(\mu)}}' = \min_{x^\top \mu = 0} \frac{x^\top A x}{\Vert x\Vert^2_\mu}.
\]
The result followed by applying the fact \(\mu_{\rm min} \Vert x \Vert^2 \le \Vert x \Vert^2_\mu \le \mu_{\rm max} \Vert x \Vert^2\). 

Next, we show \(\frac{\lambda}{n ||\mu||_2^2} \le \lambda^{(\mu)}\). Let $x\in\mathbb{R}^n$ be a vector satisfying
\[
x^\top\mu=0, \quad \Vert x \Vert = 1.
\]
Define
\[
y = x - de,\quad \text{where } d=\frac{x^\top e}{n},
\]
we can easily check \(y^\top e = 0\) and \(y^\top A y = x^\top A x\). Consider the norm of \(y\),
\[
\|y\|_2^2 = \|x-de\|_2^2 
= \|x\|_2^2 - 2d\,x^\top e + d^2\|e\|_2^2
= 1 - \frac{(x^\top e)^2}{n}\,.
\]
To bound $|x^\top e|$, decompose \(e\) by
\[
e = \frac{1}{\|\mu\|^2}\,\mu + w,\qquad\text{with }w^\top \mu=0.
\]
Then \(x^\top e = x^\top w\). By the Cauchy–Schwarz inequality, \(|x^\top e|\le \|w\|_2\). Notice that
\[
\|w\|_2^2=\|e\|_2^2 - \left\Vert\frac{\mu}{\|\mu\|_2^2}\right\Vert^2
=n-\frac{1}{\|\mu\|_2^2}\,,
\]
Since $e^\top \mu=1$,
\[
\|y\|_2^2\ge 1-\frac{n-\frac{1}{\|\mu\|_2^2}}{n}
=\frac{1}{n\|\mu\|_2^2}\,.
\]
Now, letting $\hat{y}=y/\|y\|_2$, we have $\hat{y}^\top e=0$ and thus
\[
\hat{y}^\top A \hat{y}\ge \lambda.
\]
It follows that
\[
x^\top A x = y^\top A y = \|y\|_2^2\,\hat{y}^\top A \hat{y}\ge \frac{\lambda_2}{ n ||\mu||_2^2}\,.
\]
Take the minimum over all unit vectors $x$ with $x^\top\mu=0$ yields
\[
\lambda^{(\mu)}\ge \frac{\lambda_2}{n ||\mu||_2^2}\,.
\]
\end{proof}

\subsubsection{Difference between \(\eta_1\) and \(\eta_3\)}

\begin{lemma}
\label{l:condition_number_2_eta1_eta3}
We have
\[
\eta_1 \;\ge\; \frac{1}{2}\,\eta_3 .
\]
\end{lemma}

\begin{proof}
Recall
\[
\eta_{1}=\min_{\|x\|=1}\Bigl\{\|\Phi x\|_{\rm Dir}^2+(\mu^\top \Phi x)^2\Bigr\},
\qquad
\eta_3=\Bigl(\min_{\|x\|=1}x^\top \Phi^\top D\Phi x\Bigr)\cdot
\Bigl(\min_{\langle y,e\rangle_D=0,\ \|y\|_D=1} y^\top D(I-P)y\Bigr),
\]
where \(D=\mathrm{diag}(\mu)\), \(\langle a,b\rangle_D=a^\top Db\), and \(e\) is the all-ones vector.

Let
\[
\lambda \;:=\;\min_{\langle y,e\rangle_D=0,\ \|y\|_D=1} y^\top D(I-P)y,
\qquad
\sigma \;:=\;\min_{\|x\|=1} x^\top \Phi^\top D\Phi x.
\]
Then \(\eta_3=\sigma\lambda\).

First, for any \(v\in\mathbb{R}^{|\mathcal{S}|}\),
\begin{align*}
\|v\|_{\rm Dir}^2
&=\frac12\sum_{s,s'}\mu(s)P(s'|s)\bigl(v(s)-v(s')\bigr)^2 \\
&=\sum_s \mu(s)v(s)^2-\sum_{s,s'}\mu(s)P(s'|s)v(s)v(s') \\
&= v^\top D(I-P)v,
\end{align*}
where we used \(\sum_{s'}P(s'|s)=1\) and stationarity \(\sum_s \mu(s)P(s'|s)=\mu(s')\).

We also remark on the standard observation that multiplication by \(P\) is a contraction in the \(D\)-norm:
\[
\|Pv\|_D^2
=\sum_s \mu(s)\Bigl(\sum_{s'}P(s'|s)v(s')\Bigr)^2
\le \sum_s \mu(s)\sum_{s'}P(s'|s)v(s')^2
=\sum_{s'}\mu(s')v(s')^2
=\|v\|_D^2,
\]
where the inequality is Jensen and the equality again uses stationarity.
Hence \(\|Pv\|_D\le \|v\|_D\). Therefore for any \(y\) with \(\|y\|_D=1\),
\[
y^\top D(I-P)y
=\|y\|_D^2-\langle y,Py\rangle_D
\le 1+|\langle y,Py\rangle_D|
\le 1+\|y\|_D\,\|Py\|_D
\le 2.
\]
Taking the minimum over the constraint set gives \(\lambda\le 2\).

Having established that, we next fix arbitrary \(v\in\mathbb{R}^{|\mathcal{S}|}\). Let
\[
c := \langle v,e\rangle_D = v^\top De = \mu^\top v,
\qquad
u := v - c e,
\]
so that \(\langle u,e\rangle_D=0\). Since \(\|e\|_D^2=e^\top De=\sum_s \mu(s)=1\) and \(u\perp e\) in \(\langle\cdot,\cdot\rangle_D\),
\[
\|v\|_D^2=\|u\|_D^2+c^2.
\]
Also, adding a constant does not change the Dirichlet seminorm, so \(\|v\|_{\rm Dir}=\|u\|_{\rm Dir}\). Using Step 1 and the definition of \(\lambda\),
\[
\|v\|_{\rm Dir}^2 = \|u\|_{\rm Dir}^2 = u^\top D(I-P)u \;\ge\; \lambda \|u\|_D^2.
\]
Therefore,
\[
\|v\|_{\rm Dir}^2+(\mu^\top v)^2
= \|u\|_{\rm Dir}^2 + c^2
\ge \lambda\|u\|_D^2 + c^2
\ge \frac{\lambda}{2}\bigl(\|u\|_D^2+c^2\bigr)
= \frac{\lambda}{2}\|v\|_D^2,
\]
where we used \(\lambda\le 2\) so that $(\lambda/2)c^2 \leq c^2$.

\paragraph{Step 3:}
Applying Step 2 to \(v=\Phi x\) gives, for every \(x\in\mathbb{R}^d\),
\[
\|\Phi x\|_{\rm Dir}^2+(\mu^\top \Phi x)^2
\;\ge\; \frac{\lambda}{2}\|\Phi x\|_D^2
\;=\; \frac{\lambda}{2}\, x^\top \Phi^\top D\Phi x.
\]
Taking \(\min_{\|x\|=1}\) on both sides yields
\[
\eta_1
=\min_{\|x\|=1}\Bigl\{\|\Phi x\|_{\rm Dir}^2+(\mu^\top \Phi x)^2\Bigr\}
\;\ge\; \frac{\lambda}{2}\min_{\|x\|=1} x^\top \Phi^\top D\Phi x
=\frac{\lambda}{2}\sigma
=\frac12\,\eta_3.
\]
\end{proof}

\subsubsection{Relation between Projection and Condition Number}

In \cite{kim2025implicit}, their sample complexity depends on the projection radius \(R_{\rm proj}\). Although they did not have a discussion on the projection radius, we will assume \(R_{\rm proj} = R_\theta\) and compare their sample complexity for completeness. 

\begin{lemma}
\label{l: radius and condition}
Suppose the projection radius is chosen to be 
\[
R_{\rm proj} = R_\theta = \frac{2 r_{\rm max} C}{(1-\beta)\sqrt{(1-\omega^2)\lambda_{\rm min}(\Phi^\top D \Phi)}}.
\]
We have
\[
R_\theta^2 \ge \frac{2 r_{\rm max}^2 C^2}{(1-\beta)^2 \eta_3}.
\]
\end{lemma}

\begin{proof}[Proof of Lemma \ref{l: radius and condition}]
For any vector \(y\) such that \(\|y\|_D = 1\) and \(\langle y,e \rangle_D = 0\), 
\[
y^\top D(I-P) y = 1 - \langle y,Py \rangle_D.
\]
By Cauchy–Schwarz, 
\[
\langle y,Py \rangle_D \le \|y\|_D \|Py\|_D = \|Py\|_D \le \omega.
\]
Therefore, we have 
\[
\min_{\langle y,e \rangle_D = 0, \|y\|_D=1} y^\top D (I-P)y \ge 1-\omega.
\]
Since \(\omega < 1\) which is already established in Lemma \ref{l: contraction}, \(1-\omega \ge (1-\omega^2)/2\). Therefore,
\[
\frac{1}{1-\omega^2} \ge \frac{1}{2\min_{\langle y,e \rangle_D = 0, \|y\|_D=1} y^\top D (I-P)y}
\]
Now we obtain
\[
R_\theta^2 \ge \frac{(2 r_{\rm max} C)^2}{2 (1-\beta)^2 \min_{\langle y,e \rangle_D = 0, \|y\|_D=1} y^\top D (I-P)y \cdot \lambda_{\rm min}(\Phi^\top D \Phi)} = \frac{2 r_{\rm max}^2 C^2}{(1-\beta)^2 \eta_3}.
\]
\end{proof}

%% file: sections/double_constant.tex
\section{Analysis of Double Chain Algorithm with Constant Stepsize}
\label{appendix: constant proof}

In this section, we will give a detailed proof of Theorem \ref{t: linear iid} and Theorem \ref{t: linear markov}. Notice that these proofs can also be applied to the tabular case if one sets \(\Phi = I\).

Throughout the analysis, we denote \(\mathcal{F}_t\) as the history up to iteration \(t-1\), i.e.,
\begin{equation}
\label{eq: ft}
\mathcal{F}_t := \sigma\!\left(\{s_0,\ldots,s_{t-1}\}\cup\{\hat{s}_0,\ldots,\hat{s}_{t-1}\}\right). \tag*{(filtration)}
\end{equation}
In particular, \(\theta_t\) and \(\delta_t\) is \(\mathcal{F}_t\)-measurable. We also denote \(B = 2 \Vert \delta_0 \Vert + (r_{\rm max}+\|\theta^*\|)\).

Before going to the proof, we first record some useful properties about \(f(s,\hat{s},\theta)\) and \(g(s,s',\theta)\).

Define the mean field
\[
\bar{h}(\theta) := \mathbb{E}_{\mu}\!\left[f(s,\hat{s},\theta)+g(s,s',\theta)\right],
\]
where under \(\mathbb{E}_\mu\) we sample \(s\sim\mu\), \(\hat{s}\sim\mu\) independently, and then sample \(s'\sim P(\cdot\mid s)\). By Eq.\,\eqref{eq: fix point double chain}, it is straightforward to check that
\begin{equation}
\label{eq: fg expectation}
\begin{aligned}
\bar{h}(\theta)
= & -\Phi^\top \mu \mu^\top (R+\Phi \theta) + \Phi^\top D (R+P\Phi \theta - \Phi \theta) \\
= & -\Phi^\top \mu \mu^\top \Phi \delta + \Phi^\top D (P-I) \Phi \delta,
\end{aligned}
\end{equation}
where \(\delta:=\theta-\theta^*\). 

Beyond the above fact, we have two additional lemmas.

\begin{lemma} \label{l: bound on f and g linear}
For any $t\geq 0$, we have
\begin{align*}
    \|f(s_t, \hat{s}_t, \theta_t)\| \le\,& \|\delta_t\|+r_{\max}+\|\theta^*\|,\\
    \|g(s_t,s_t', \theta_t)\| \le \,&2(\|\delta_t\|+r_{\max}+\|\theta^*\|).
\end{align*}
\end{lemma}
\begin{proof}[Proof of Lemma \ref{l: bound on f and g linear}]
According to Eq.(\ref{eq: definition of f and g}),
\begin{align*}
    \|f(s_t, \hat{s}_t, \theta_t)\|
    \le\,& \left(|r_{s_t}|+|\phi(s_t)^\top \theta_t|\right)\|\phi(\hat{s}_t)\|\\
    \le\,& \left(r_{\max}+\|\theta_t\|\,\|\phi(s_t)\|\right)\|\phi(\hat{s}_t)\| \tag*{(Cauchy--Schwarz)}\\
    \le\,& r_{\max}+\|\theta_t\| \tag*{(\(\|\phi(\cdot)\|\le 1\))}\\
    \le\,& r_{\max}+\|\theta_t-\theta^*\|+\|\theta^*\|\\
    =\,& \|\delta_t\|+r_{\max}+\|\theta^*\|.
\end{align*}
Similarly,
\begin{align*}
    \|g(s_t,s_t', \theta_t)\|
    \le\,&\left( |r_{s_t}|+|\phi(s_t')^\top \theta_t|+|\phi(s_t)^\top \theta_t| \right)\|\phi(s_t)\|\\
    \le\,& r_{\max}+2\|\theta_t\| \tag*{(\(\|\phi(\cdot)\|\le 1\))}\\
    \le\,& r_{\max}+2\|\theta_t-\theta^*\|+2\|\theta^*\|\\
    \le\,&2(\|\delta_t\|+r_{\max}+\|\theta^*\|).
\end{align*}
\end{proof}

\begin{lemma} \label{l: lipschitz f and g}
The function \(f\) is \(1\)-Lipschitz and \(g\) is \(2\)-Lipschitz with respect to \(\theta\), i.e.,
\[
\|g(s,s', \theta_1) - g(s,s', \theta_2)\| \le 2 \|\theta_1 - \theta_2\|, \quad
\|f(s, \hat{s}, \theta_1) - f(s, \hat{s}, \theta_2)\| \le \|\theta_1 - \theta_2\|.
\]
\end{lemma}

\begin{proof}[Proof of Lemma \ref{l: lipschitz f and g}]
Using Eq.(\ref{eq: definition of f and g}) and \(\|\phi(\cdot)\|\le 1\),
\[
g(s,s',\theta_1)-g(s,s',\theta_2)
=\big((\phi(s')-\phi(s))^\top(\theta_1-\theta_2)\big)\,\phi(s),
\]
hence
\[
\|g(s,s', \theta_1) - g(s,s', \theta_2)\|
\le \|\phi(s')-\phi(s)\|\,\|\phi(s)\|\,\|\theta_1-\theta_2\|
\le 2\|\theta_1-\theta_2\|.
\]
Likewise,
\[
f(s,\hat{s},\theta_1)-f(s,\hat{s},\theta_2)
=-(\phi(s)^\top(\theta_1-\theta_2))\,\phi(\hat{s}),
\]
so \(\|f(s,\hat{s},\theta_1)-f(s,\hat{s},\theta_2)\|\le \|\theta_1-\theta_2\|\).
\end{proof}

Now we are ready to provide our proof of Theorem \ref{t: linear iid}.

\subsection{Proof of Theorem \ref{t: linear iid}}

\begin{proof}[Proof of Theorem \ref{t: linear iid}]
Let \(u_t := f(s_t,\hat{s}_t,\theta_t)+g(s_t,s_t',\theta_t)\). From Eq.(\ref{eq: linear stochastic averaged reward update}) with \(\alpha_t=\alpha\),
\begin{equation*}
\|\delta_{t+1}\|^2
= \|\delta_t\|^2 + 2\alpha\,\delta_t^\top u_t + \alpha^2\|u_t\|^2. 
\end{equation*}
Taking conditional expectation given \(\mathcal{F}_t\), and using that \(\delta_t\) is \(\mathcal{F}_t\)-measurable,
\begin{equation*}
\mathbb{E}\!\left[\|\delta_{t+1}\|^2\mid \mathcal{F}_t\right]
= \|\delta_t\|^2 + 2\alpha\,\delta_t^\top \mathbb{E}[u_t\mid \mathcal{F}_t]
+ \alpha^2\,\mathbb{E}[\|u_t\|^2\mid \mathcal{F}_t]. 
\end{equation*}

\textbf{Step 1: Drift term.}
We first introduce the following lemma:

\begin{lemma}
\label{l: gradient splitting}
The linear function $h(\theta):=\Phi^\top D (I-P) \Phi \theta$ is a gradient splitting of the quadratic function $\|\Phi\theta\|_{\rm Dir}^2$, i.e.,
\[
\langle\theta,h(\theta)\rangle=\frac{1}{2}\langle\theta,\nabla_\theta\|\Phi\theta\|_{\rm Dir}^2\rangle,\quad \forall\,\theta\in\mathbb{R}^d.
\]
As a result, we have $\|\Phi \theta\|_{\rm Dir}^2=\theta^\top h(\theta)$ for all $\theta\in\mathbb{R}^d$.
\end{lemma}

The proof of this lemma can be found in Section \ref{sec: pof gradient splitting}.

Under i.i.d.\ sampling, \((s_t,\hat{s}_t)\) are independent of \(\mathcal{F}_t\) with marginals \(\mu\), and \(s_t'\sim P(\cdot\mid s_t)\). Therefore,
\[
\mathbb{E}[u_t\mid \mathcal{F}_t] = \bar{h}(\theta_t),
\]
where \(\bar h(\cdot)\) is defined in Eq.(\ref{eq: fg expectation}).

Using Eq.(\ref{eq: fg expectation}) with \(\delta=\delta_t\),
\begin{align*}
\delta_t^\top \bar{h}(\theta_t)
=&\ \delta_t^\top \Phi^\top D(P-I)\Phi\,\delta_t
-\delta_t^\top \Phi^\top \mu\mu^\top \Phi\,\delta_t \\
=&\ -\delta_t^\top \Phi^\top D(I-P)\Phi\,\delta_t - \|\mu^\top\Phi\delta_t\|^2 \\
=&\ -\|\Phi\delta_t\|_{\rm Dir}^2 - \|\mu^\top\Phi\delta_t\|^2 \tag*{(Lemma \ref{l: gradient splitting})}\\
\le&\ -\eta\,\|\delta_t\|^2. \tag{Eq.\,\eqref{eq: eta}}
\end{align*}

\textbf{Step 2: Second-moment term.}
By Lemma \ref{l: bound on f and g linear},
\[
\|u_t\|
\le \|f(s_t,\hat{s}_t,\theta_t)\|+\|g(s_t,s_t',\theta_t)\|
\le 3(\|\delta_t\|+r_{\max}+\|\theta^*\|). \tag*{(Lemma \ref{l: bound on f and g linear})}
\]
Hence,
\begin{align*}
\mathbb{E}[\|u_t\|^2\mid \mathcal{F}_t]
\le&\ 9(\|\delta_t\|+r_{\max}+\|\theta^*\|)^2 \\
\le&\ 18\|\delta_t\|^2 + 18(r_{\max}+\|\theta^*\|)^2. \tag{\((a+b)^2\le 2a^2+2b^2\)}
\end{align*}

\textbf{Step 3: Combine.}
Plugging the two bounds back,
\[
\mathbb{E}\!\left[\|\delta_{t+1}\|^2\mid \mathcal{F}_t\right]
\le (1-2\alpha\eta+18\alpha^2)\|\delta_t\|^2 + 18\alpha^2(r_{\max}+\|\theta^*\|)^2.
\]
Since \(\alpha\le \eta/18\), we have \(18\alpha^2\le \alpha\eta\), hence
\[
1-2\alpha\eta+18\alpha^2 \le 1-\alpha\eta,
\]
and therefore
\[
\mathbb{E}\!\left[\|\delta_{t+1}\|^2\mid \mathcal{F}_t\right]
\le (1-\alpha\eta)\|\delta_t\|^2 + 18\alpha^2(r_{\max}+\|\theta^*\|)^2.
\]
Taking expectation and iterating the recursion yields
\[
\mathbb{E}\|\delta_T\|^2
\le (1-\alpha\eta)^T\mathbb{E}\|\delta_0\|^2
+ 18\alpha^2(r_{\max}+\|\theta^*\|)^2\sum_{k=0}^{T-1}(1-\alpha\eta)^k.
\]
Using \(\sum_{k=0}^{T-1}(1-\alpha\eta)^k \le \frac{1}{\alpha\eta}\) and \(1-x\le e^{-x}\),
\[
\mathbb{E}\|\delta_T\|^2
\le e^{-\alpha\eta T}\mathbb{E}\|\delta_0\|^2
+ \frac{18\alpha(r_{\max}+\|\theta^*\|)^2}{\eta}.
\]
\end{proof}

\subsection{Proof of Lemma \ref{l: gradient splitting}}
\label{sec: pof gradient splitting}

\begin{proof}[Proof of Lemma \ref{l: gradient splitting}]
According to Eq.\,\eqref{eq: dir norm} and the stationarity \(\mu^\top=\mu^\top P\),
\begin{align*}
\|f\|_{\rm Dir}^2
=&\ \frac12\sum_{s,s'}\mu(s)P(s'|s)\big(f(s)^2+f(s')^2-2f(s)f(s')\big)\\
=&\ \sum_s\mu(s)f(s)^2-\sum_{s,s'}\mu(s)P(s'|s)f(s)f(s') \\
=&\ f^\top D(I-P)f.
\end{align*}
Taking \(f=\Phi\theta\) gives
\[
\|\Phi\theta\|_{\rm Dir}^2 = \theta^\top \Phi^\top D(I-P)\Phi\,\theta = \theta^\top h(\theta).
\]
Since \(\|\Phi\theta\|_{\rm Dir}^2\) is a quadratic form, its gradient is
\[
\nabla_\theta\|\Phi\theta\|_{\rm Dir}^2 = \big(\Phi^\top D(I-P)\Phi + \Phi^\top (I-P)^\top D\Phi\big)\theta,
\]
hence
\[
\frac12\langle\theta,\nabla_\theta\|\Phi\theta\|_{\rm Dir}^2\rangle
= \theta^\top \Phi^\top D(I-P)\Phi\,\theta
= \langle\theta,h(\theta)\rangle.
\]
\end{proof}

\subsection{Proof of Theorem \ref{t: linear markov}}

\begin{proof}[Proof of Theorem \ref{t: linear markov}]
For convenience, define
\begin{equation}
\label{eq: fg bar}
\begin{aligned}
\bar{g}(\theta_t)
:= &\ \sum_{s,s'}\mu(s)P(s'|s)\,g(s,s',\theta_t)
= \Phi^\top D (R+P \Phi \theta_t - \Phi \theta_t), \\
\bar{f}(\theta_t)
:= &\ \sum_{s,\hat{s}} \mu(s)\mu(\hat{s})\, f(s,\hat{s},\theta_t)
= -\Phi^\top \mu \mu^\top (R+\Phi \theta_t).
\end{aligned}
\end{equation}

Let \(u_t := f(s_t,\hat{s}_t,\theta_t)+g(s_t,s_t',\theta_t)\). Using that \(\delta_t\) and \(\bar f(\theta_t)+\bar g(\theta_t)\) are \(\mathcal{F}_t\)-measurable, we have
\begin{align*}
&\mathbb{E}\!\left[\delta_t^\top\!\left(\Phi^\top D (P-I) \Phi \delta_t - \Phi^\top \mu \mu^\top \Phi \delta_t - \mathbb{E}\!\left[u_t \mid \mathcal{F}_t\right]\right)\right] \\
=&\ \mathbb{E}\!\left[\delta_t^\top\!\left(\bar{g}(\theta_t)+\bar{f}(\theta_t) - \mathbb{E}\!\left[u_t \mid \mathcal{F}_t\right]\right)\right] \tag*{(Eq.\,\eqref{eq: fg expectation})}\\
=&\ \mathbb{E}\!\left[\delta_t^\top\!\left(\bar{g}(\theta_t)+\bar{f}(\theta_t) - u_t\right)\right] \\
=&\ -\mathbb{E}\!\left[\delta_t^\top\!\left((g(s_t,s_t',\theta_t)-\bar g(\theta_t))+(f(s_t,\hat s_t,\theta_t)-\bar f(\theta_t))\right)\right].
\end{align*}

Define \(B\) and \(G\) as in the theorem statement. We introduce the following lemma:

\begin{lemma}
\label{l: all parts bound}
Suppose
\[
\alpha \le \frac{\eta B^2}{(3\tau_{\rm mix}+1)G},
\]
then \(\mathbb{E} [\Vert \delta_t \Vert^2] \le B^2\) for all \(t\ge 0\).
\end{lemma}

The proof of this lemma can be found in Section \ref{sec: pof all parts bound}.

Since \(\mathbb{E} [\Vert \delta_t \Vert^2] \le B^2\) for all \(t\ge 0\), we know that the following lemma holds. 

\begin{lemma}
\label{l: markov noise}
Suppose \(t \ge \tau_{\rm mix}\). With \(\bar{f}, \bar{g}\) defined in Eq.(\ref{eq: fg bar}), we have
\begin{align*}
\mathbb{E} \left[ \delta_t^\top \left(g(s_t, s_t', \theta_t) - \bar{g}(\theta_t)\right) \right]
\le &\ \alpha \left(3B^2+(r_{\max}+\|\theta^*\|)^2\right) + (42B^2 + 30(r_{\rm max}+\|\theta^*\|)^2)\,\tau_{\rm mix}\,\alpha, \\
\mathbb{E} \left[ \delta_t^\top \left(f(s_t, \hat{s}_t, \theta_t) - \bar{f}(\theta_t)\right) \right]
\le &\ \alpha \left(3B^2+(r_{\max}+\|\theta^*\|)^2\right) + (21B^2 + 15(r_{\rm max}+\|\theta^*\|)^2)\,\tau_{\rm mix}\,\alpha.
\end{align*}
\end{lemma}

The proof of this lemma can be found in Section \ref{sec: pof markov noise}. 

Summing the two inequalities and multiplying by the factor \(2\alpha\), we obtain for all \(t\ge\tau_{\rm mix}\),
\begin{align*}
&-2\alpha\,\mathbb{E}\!\left[\delta_t^\top\!\left(\Phi^\top D (P-I) \Phi \delta_t - \Phi^\top \mu \mu^\top \Phi \delta_t
- \mathbb{E}[u_t\mid\mathcal{F}_t]\right)\right] \\
\le\;&
4\alpha^2\Big(3B^2+(r_{\max}+\|\theta^*\|)^2\Big)
+\alpha^2\tau_{\rm mix}\Big(126B^2+90(r_{\max}+\|\theta^*\|)^2\Big).
\end{align*}

Using the same \(I_1\) and \(I_2\) bounds as in the i.i.d.\ case (so that the drift contributes \(-(2\alpha\eta)\|\delta_t\|^2\) and the squared-norm term contributes \(18\alpha^2(\|\delta_t\|^2+(r_{\max}+\|\theta^*\|)^2)\)),
we obtain for any \(t\ge\tau_{\rm mix}\),
\begin{align*}
\mathbb{E}\|\delta_{t+1}\|^2
\le\;& (1-2\alpha\eta)\mathbb{E}\|\delta_t\|^2
+18\alpha^2\,\mathbb{E}\|\delta_t\|^2
+18\alpha^2(r_{\max}+\|\theta^*\|)^2 \\
&\quad +4\alpha^2\Big(3B^2+(r_{\max}+\|\theta^*\|)^2\Big)
+\alpha^2\tau_{\rm mix}\Big(126B^2+90(r_{\max}+\|\theta^*\|)^2\Big) \\
\le\;& (1-2\alpha\eta)\mathbb{E}\|\delta_t\|^2
+18\alpha^2 B^2
+22\alpha^2(r_{\max}+\|\theta^*\|)^2
+12\alpha^2 B^2 \\
&\quad +\alpha^2\tau_{\rm mix}\Big(126B^2+90(r_{\max}+\|\theta^*\|)^2\Big) \tag*{(Lemma \ref{l: all parts bound})}\\
\le\;& (1-2\alpha\eta)\mathbb{E}\|\delta_t\|^2
+\alpha^2\Big(42B^2+30(r_{\max}+\|\theta^*\|)^2\Big)\,(3\tau_{\rm mix}+1) \\
=\;& (1-2\alpha\eta)\mathbb{E}\|\delta_t\|^2+\alpha^2 G(3\tau_{\rm mix}+1).
\end{align*}

Iterating the above inequality from \(t=\tau_{\rm mix}\) to \(T-1\) gives
\[
\mathbb{E}\|\delta_T\|^2
\le (1-2\alpha\eta)^{T-\tau_{\rm mix}}\mathbb{E}\|\delta_{\tau_{\rm mix}}\|^2
+\alpha^2 G(3\tau_{\rm mix}+1)\sum_{k=0}^{T-\tau_{\rm mix}-1}(1-2\alpha\eta)^k.
\]
Using Lemma \ref{l: all parts bound}, \(\sum_{k\ge0}(1-2\alpha\eta)^k \le \frac{1}{2\alpha\eta}\), and \(1-x\le e^{-x}\), we obtain
\[
\mathbb{E}\|\delta_T\|^2
\le e^{-2\alpha\eta(T-\tau_{\rm mix})}B^2+\frac{\alpha G(3\tau_{\rm mix}+1)}{2\eta}.
\]
\end{proof}

\subsection{Proof of Lemma \ref{l: all parts bound}}
\label{sec: pof all parts bound}

\begin{proof}[Proof of Lemma \ref{l: all parts bound}]
We first handle \(t\le \tau_{\rm mix}\) via a pathwise bound.
From Eq.(\ref{eq: linear stochastic averaged reward update}) and Lemma \ref{l: bound on f and g linear},
\begin{align*}
\|\delta_{t+1}\|
\le&\ \|\delta_t\|+\alpha\big(\|f(s_t,\hat s_t,\theta_t)\|+\|g(s_t,s_t',\theta_t)\|\big)\\
\le&\ \|\delta_t\|+3\alpha(\|\delta_t\|+r_{\max}+\|\theta^*\|)
=(1+3\alpha)\|\delta_t\|+3\alpha(r_{\max}+\|\theta^*\|).
\end{align*}
Iterating the inequality yields, for \(t\le \tau_{\rm mix}\),
\[
\|\delta_t\|
\le (1+3\alpha)^t\|\delta_0\|
+3\alpha(r_{\max}+\|\theta^*\|)\sum_{k=0}^{t-1}(1+3\alpha)^k.
\]
Using \((1+3\alpha)^t\le e^{3\alpha t}\) and \(\sum_{k=0}^{t-1}(1+3\alpha)^k\le t(1+3\alpha)^{t}\),
we have
\[
\|\delta_t\|\le e^{3\alpha\tau_{\rm mix}}\|\delta_0\|
+3\alpha\tau_{\rm mix}e^{3\alpha\tau_{\rm mix}}(r_{\max}+\|\theta^*\|).
\]
The stepsize condition implies \(\alpha\tau_{\rm mix}\le 1/6\) (indeed \(G\ge 42B^2\) and \(\eta\le 3\) imply
\(\alpha \le \eta/(42(3\tau_{\rm mix}+1))\le 1/(6\tau_{\rm mix})\)),
so \(e^{3\alpha\tau_{\rm mix}}\le e^{1/2}<2\) and \(6\alpha\tau_{\rm mix}\le 1\). Therefore,
\[
\|\delta_t\|\le 2\|\delta_0\|+6\alpha\tau_{\rm mix}(r_{\max}+\|\theta^*\|)\le 2\|\delta_0\|+(r_{\max}+\|\theta^*\|)=B,
\quad \forall t\le \tau_{\rm mix}.
\]
Hence \(\mathbb{E}\|\delta_t\|^2\le B^2\) for all \(t\le \tau_{\rm mix}\).

Now consider \(t>\tau_{\rm mix}\) and use induction on \(t\).
Assume \(\mathbb{E}\|\delta_k\|^2\le B^2\) for all \(k\le t\).
Then we have the following lemma:

\begin{lemma}
\label{l: markov noise induction}
Suppose \(t \ge \tau_{\rm mix}\) and \(\mathbb{E} [\Vert \delta_k \Vert^2] \le B^2\) for all \(k \le t\).
Then Lemma \ref{l: markov noise} holds at time \(t\).
\end{lemma}

The proof of this lemma can be found in Section \ref{sec: pof markov noise induction}. 

Plugging those bounds into the recursion in the proof of Theorem \ref{t: linear markov} gives
\[
\mathbb{E}\|\delta_{t+1}\|^2 \le (1-2\alpha\eta)B^2 + \alpha^2 G(3\tau_{\rm mix}+1).
\]
Using the stepsize condition \(\alpha \le \frac{\eta B^2}{(3\tau_{\rm mix}+1)G}\), we have
\(\alpha^2 G(3\tau_{\rm mix}+1)\le \alpha\eta B^2\), hence
\[
\mathbb{E}\|\delta_{t+1}\|^2 \le (1-2\alpha\eta)B^2+\alpha\eta B^2 \le B^2.
\]
This completes the induction and proves the claim for all \(t\ge 0\).
\end{proof}

\subsection{Proof of Lemma \ref{l: markov noise}}
\label{sec: pof markov noise}

\begin{proof}[Proof of Lemma~\ref{l: markov noise}]
Fix any \(t\ge \tau_{\rm mix}\).
Lemma~\ref{l: all parts bound} gives \(\mathbb{E}\|\delta_k\|^2\le B^2\) for all \(k\le t\),
so the assumptions of Lemma~\ref{l: markov noise induction} are satisfied.
Lemma~\ref{l: markov noise induction} therefore implies the two inequalities stated in Lemma~\ref{l: markov noise}.
\end{proof}

\subsection{Proof of Lemma \ref{l: markov noise induction}}
\label{sec: pof markov noise induction}

\begin{proof}[Proof of Lemma \ref{l: markov noise induction}]
For simplicity, denote
\[
g_t(\theta) := g(s_t,s_t',\theta),\qquad f_t(\theta):=f(s_t,\hat s_t,\theta).
\]

\paragraph{Step 1: the \(g\)-term.}
Decompose
\begin{align*}
&\mathbb{E}\left[(\theta_t-\theta^*)^\top\left(g_t(\theta_t)-\bar g(\theta_t)\right)\right]\\
=&\underbrace{\mathbb{E}\left[(\theta_t-\theta_{t-\tau_{\rm mix}})^\top\left(g_t(\theta_t)-\bar g(\theta_t)\right)\right]}_{I_1}
+\underbrace{\mathbb{E}\left[(\theta_{t-\tau_{\rm mix}}-\theta^*)^\top\left(g_t(\theta_{t-\tau_{\rm mix}})-\bar g(\theta_{t-\tau_{\rm mix}})\right)\right]}_{I_2}\\
&\quad+\underbrace{\mathbb{E}\left[(\theta_{t-\tau_{\rm mix}}-\theta^*)^\top\left(g_t(\theta_t)-g_t(\theta_{t-\tau_{\rm mix}})\right)\right]}_{I_3}
+\underbrace{\mathbb{E}\left[(\theta_{t-\tau_{\rm mix}}-\theta^*)^\top\left(\bar g(\theta_{t-\tau_{\rm mix}})-\bar g(\theta_t)\right)\right]}_{I_4}.
\end{align*}

\textbf{Term \(I_1\).}
We have
\[
\theta_t-\theta_{t-\tau_{\rm mix}}=\sum_{i=t-\tau_{\rm mix}+1}^{t}\alpha\big(g_i(\theta_i)+f_i(\theta_i)\big),
\]
so by Lemma \ref{l: bound on f and g linear},
\[
\|\theta_t-\theta_{t-\tau_{\rm mix}}\|\le 3\sum_{i=t-\tau_{\rm mix}+1}^{t}\alpha\big(\|\delta_i\|+r_{\max}+\|\theta^*\|\big).
\]
Also \(\|g_t(\theta_t)-\bar g(\theta_t)\|\le \|g_t(\theta_t)\|+\|\bar g(\theta_t)\|
\le 4(\|\delta_t\|+r_{\max}+\|\theta^*\|)\).
Therefore, since \(2ab\le a^2+b^2\), we have
\begin{align*}
I_1
&\le \mathbb{E}\!\left[\|\theta_t-\theta_{t-\tau_{\rm mix}}\|\,\|g_t(\theta_t)-\bar g(\theta_t)\|\right]\\
&\le \mathbb{E}\!\left[12(\|\delta_t\|+c)\sum_{i=t-\tau_{\rm mix}+1}^{t}\alpha(\|\delta_i\|+c)\right]\\
&\le 12\sum_{i=t-\tau_{\rm mix}+1}^{t}\alpha\,
\mathbb{E}\!\left[\|\delta_t\|^2+\|\delta_i\|^2+2c^2\right] \\
&\le 24\,(B^2+c^2)\sum_{i=t-\tau_{\rm mix}+1}^{t}\alpha.
\end{align*}

\textbf{Term \(I_2\).}
Condition on \(\mathcal{F}_{t-\tau_{\rm mix}}\).
Given \(\mathcal{F}_{t-\tau_{\rm mix}}\), the law of \((s_t,s_t')\) differs from \(\mu(s)P(s'|s)\) by at most
\(\|P_{\tau_{\rm mix}}(\cdot|s_{t-\tau_{\rm mix}})-\mu\|_1\),
and by definition of \(\tau_{\rm mix}=\tau_{\rm mix}(\alpha)\) we have \(\|P_{\tau_{\rm mix}}(\cdot|s)-\mu\|_1\le \alpha\) for all \(s\).
Using Lemma \ref{l: bound on f and g linear}, \(\sup_{s,s'}\|g(s,s',\theta_{t-\tau_{\rm mix}})\|\le 2(\|\delta_{t-\tau_{\rm mix}}\|+c)\), hence
\[
\left\|\mathbb{E}\left[g_t(\theta_{t-\tau_{\rm mix}})-\bar g(\theta_{t-\tau_{\rm mix}})\mid \mathcal{F}_{t-\tau_{\rm mix}}\right]\right\|
\le 2\alpha(\|\delta_{t-\tau_{\rm mix}}\|+c).
\]
Therefore,
\begin{align*}
I_2
&= \mathbb{E}\!\left[\mathbb{E}\left[(\theta_{t-\tau_{\rm mix}}-\theta^*)^\top\left(g_t(\theta_{t-\tau_{\rm mix}})-\bar g(\theta_{t-\tau_{\rm mix}})\right)\mid \mathcal{F}_{t-\tau_{\rm mix}}\right]\right]\\
&\le \mathbb{E}\!\left[\|\delta_{t-\tau_{\rm mix}}\|\cdot 2\alpha(\|\delta_{t-\tau_{\rm mix}}\|+c)\right]\\
&\le \alpha\,\mathbb{E}\!\left[3\|\delta_{t-\tau_{\rm mix}}\|^2+c^2\right]
\le \alpha\left(3B^2+c^2\right).
\end{align*}

\textbf{Terms \(I_3\) and \(I_4\).}
By Lemma \ref{l: lipschitz f and g},
\[
\|g_t(\theta_t)-g_t(\theta_{t-\tau_{\rm mix}})\|\le 2\|\theta_t-\theta_{t-\tau_{\rm mix}}\|
\le 6\sum_{i=t-\tau_{\rm mix}+1}^{t}\alpha(\|\delta_i\|+c).
\]
Hence
\begin{align*}
I_3
&\le 6\sum_{i=t-\tau_{\rm mix}+1}^{t}\alpha\,
\mathbb{E}\!\left[\|\delta_{t-\tau_{\rm mix}}\|(\|\delta_i\|+c)\right]\\
&\le 6\sum_{i=t-\tau_{\rm mix}+1}^{t}\alpha\,
\mathbb{E}\!\left[\|\delta_{t-\tau_{\rm mix}}\|\|\delta_i\|+\|\delta_{t-\tau_{\rm mix}}\|c\right]\\
&\le \sum_{i=t-\tau_{\rm mix}+1}^{t}\alpha\,\mathbb{E}\!\left[6\cdot\frac{\|\delta_{t-\tau_{\rm mix}}\|^2+\|\delta_i\|^2}{2}
+6\cdot\frac{\|\delta_{t-\tau_{\rm mix}}\|^2+c^2}{2}\right]\\
&\le \left(9B^2+3c^2\right)\sum_{i=t-\tau_{\rm mix}+1}^{t}\alpha.
\end{align*}
The same bound applies to \(I_4\) (since \(\bar g(\cdot)\) is also \(2\)-Lipschitz in \(\theta\) by Lemma \ref{l: lipschitz f and g} and linearity of expectation).

Combining \(I_1\)–\(I_4\),
\[
\mathbb{E}\left[(\theta_t-\theta^*)^\top\left(g_t(\theta_t)-\bar g(\theta_t)\right)\right]
\le \alpha(3B^2+c^2) + (42B^2+30c^2)\sum_{i=t-\tau_{\rm mix}+1}^{t}\alpha.
\]
With constant stepsize, \(\sum_{i=t-\tau_{\rm mix}+1}^{t}\alpha = \tau_{\rm mix}\alpha\), proving the first inequality in Lemma \ref{l: markov noise}.

\paragraph{Step 2: the \(f\)-term.}
Same as before, we can do the same decomposition so that
\[
\mathbb{E}\left[(\theta_t-\theta^*)^\top\left(f_t(\theta_t)-\bar f(\theta_t)\right)\right] = I_1'+I_2'+I_3'+I_4'.
\]
The only change is in \(I_2'\): given \(\mathcal{F}_{t-\tau_{\rm mix}}\), the joint law of \((s_t,\hat s_t)\) equals
\(P_{\tau_{\rm mix}}(\cdot|s_{t-\tau_{\rm mix}})\otimes P_{\tau_{\rm mix}}(\cdot|\hat s_{t-\tau_{\rm mix}})\),
whose \(\ell_1\) distance to \(\mu\otimes\mu\) is at most
\(\|P_{\tau_{\rm mix}}(\cdot|s_{t-\tau_{\rm mix}})-\mu\|_1 + \|P_{\tau_{\rm mix}}(\cdot|\hat s_{t-\tau_{\rm mix}})-\mu\|_1 \le 2\alpha\).

Using the same analysis as function \(g\), we conclude
\[
\mathbb{E}\left[(\theta_t-\theta^*)^\top\left(f_t(\theta_t)-\bar f(\theta_t)\right)\right]
\le \alpha(3B^2+c^2) + (21B^2+15c^2)\sum_{i=t-\tau_{\rm mix}+1}^{t}\alpha,
\]
which proves the second inequality in Lemma \ref{l: markov noise}.
\end{proof}

%% file: sections/double_decaying.tex
\section{Analysis of Double Chain Algorithm with Decaying Stepsize}
\label{appendix: double chain with decaying}

In this section, we give a detailed proof of Theorem
\ref{t: linear markov with decaying stepsize}. We first state the full version of the theorem with all the constants:

\begin{theorem}[Restatement of Theorem \ref{t: linear markov with decaying stepsize}]
Let \(a>0\), \(c_0>0\), and \(\alpha_t = \frac{a}{(t+c_0)^{\xi}}\) with \(\xi\in(0,1]\). Denote
\begin{align*}
\beta_1 := & 2\mathbb{E}\Big[\big(\|\theta_0\|+r_{\max}+2\|\theta^*\|\big)^2\Big]
          +2\mathbb{E}\big[\|\theta_0-\theta^*\|^2\big], \\
\beta_2 := & 946\big(r_{\max}+3\|\theta^*\|\big)^2,    
\end{align*}
and
\begin{align*}
L_1 := & \max\left\{1,\ \frac{\log C + 1}{\log(1/\beta)}\right\} \\
\beta(T) := & 4aL_1\beta_2\big(\log(T+c_0)-\log a+1\big)
\end{align*}
Assume \(c_0\ge \max\{c_{0,1}(a,\xi),\,c_{0,2}(a,\xi)\}\), where \(c_{0,1},c_{0,2}\) are chosen so that
Lemma~\ref{l: sum alpha bound} holds.

\begin{enumerate}
\item If \(\xi\in(0,1)\) and \(c_0 \ge \big(\frac{2\xi}{a\eta}\big)^{\frac{1}{1-\xi}}\), then
\begin{small}
\begin{align*}
& \mathbb{E}\big[\|\delta_T\|^2\big] \\
\le\;
&\beta_1 \exp\!\left(-\frac{\eta a}{1-\xi}
\Big( (T+c_0)^{1-\xi}-(\tau_{\rm mix}+c_0)^{1-\xi}\Big)\right) \\
& \quad \;+\;\frac{\beta(T)}{\eta (T+c_0)^{\xi}}.
\end{align*}
\end{small}
\item If \(\xi=1\) and \(c_0\ge \eta a\), then
\[
\mathbb{E}\big[\|\delta_T\|^2\big]
\le
\beta_1\left(\frac{\tau_{\rm mix}+c_0}{T+c_0}\right)^{\eta a}
+\frac{\Gamma(T)}{(T+c_0)^q},
\]
where \(q=\min\{1,\eta a\}\), \(\tilde\beta(T):=2L_1\beta_2\big(\log(T+c_0)-\log a+1\big)\) and
\begin{small}
\[
\Gamma(T) :=
\begin{cases}
\displaystyle \frac{4a^2}{1-\eta a}\,\tilde\beta(T), & 0<a<1/\eta,\\[6pt]
\displaystyle 4a^2\log(T+c_0)\,\tilde\beta(T), & a=1/\eta,\\[6pt]
\displaystyle \frac{4ea^2}{\eta a-1}\,\tilde\beta(T), & a>1/\eta.
\end{cases}
\]
\end{small}
\end{enumerate}
\end{theorem}

Throughout, we use the same notations as in Appendix \ref{appendix: constant proof}.

\subsection{Proof of Theorem \ref{t: linear markov with decaying stepsize}}

\begin{proof}[Proof of Theorem \ref{t: linear markov with decaying stepsize}]
To control Markov noise, we need an upper bound on \(\sum_{i=t-\tau_{\rm mix}}^{t-1}\alpha_i\).
This is captured by the following lemma.

\begin{lemma}
\label{l: sum alpha bound}
There exist constants \(c_{0,1}=c_{0,1}(a,\xi)>0\) and \(c_{0,2}=c_{0,2}(a,\xi)>0\)
such that if \(c_0\ge \max\{c_{0,1},c_{0,2}\}\), then for all \(t\in[\tau_{\rm mix},T]\),
\begin{enumerate}
\item \(\displaystyle
\sum_{i=t-\tau_{\rm mix}}^{t-1} \alpha_i
\le 2 L_1 \alpha_t\big(\log(1/\alpha_T)+1\big);\)
\item \(\displaystyle
\sum_{i=t-\tau_{\rm mix}}^{t-1} \alpha_i
\le \min\{1/12,\ \eta / 2694\}.
\)
\end{enumerate}
\end{lemma}

The proof of Lemma \ref{l: sum alpha bound} can be found in
Section \ref{sec: pof sum alpha bound}.

Denote
\begin{align*}
\bar{g}(\theta_t)
:=\;& \sum_{s_t,s_t'}\mu(s_t)P(s_t'|s_t)g(s_t,s_t',\theta_t)
= \Phi^\top D(R+P\Phi\theta_t-\Phi\theta_t),\\
\bar{f}(\theta_t)
:=\;& \sum_{s_t,\hat{s}_t}\mu(s_t)\mu(\hat{s}_t)f(s_t,\hat{s}_t,\theta_t)
= -\Phi^\top \mu\mu^\top (R+\Phi\theta_t).
\end{align*}

We will also need the following two lemmas.

\begin{lemma}
\label{l: theta difference decaying}
Suppose \(t_1\le t_2\) and \(\sum_{i=t_1}^{t_2-1}\alpha_i\le 1/12\). Then
\begin{enumerate}
\item \(\displaystyle
\|\theta_{t_2}-\theta_{t_1}\|
\le 6\big(\|\theta_{t_1}\|+r_{\rm max}+2\|\theta^*\|\big)\sum_{i=t_1}^{t_2-1}\alpha_i;
\)
\item \(\displaystyle
\|\theta_{t_2}-\theta_{t_1}\|
\le 12\big(\|\theta_{t_2}\|+r_{\rm max}+2\|\theta^*\|\big)\sum_{i=t_1}^{t_2-1}\alpha_i.
\)
\end{enumerate}
\end{lemma}

\begin{lemma}
\label{l: markov noise decaying}
Suppose \(t\in[\tau_{\rm mix},T]\) and \(\sum_{i=t-\tau_{\rm mix}}^{t-1}\alpha_i\le 1/12\).
Then
\begin{align*}
\mathbb{E} \Big[ (\theta_t-\theta^*)^\top\big(g(s_t,s_t',\theta_t)-\bar{g}(\theta_t)\big) \Big]
\le\;& \Big(880 \mathbb{E}\|\delta_t\|^2 + 304(r_{\rm max}+3\|\theta^*\|)^2\Big)
\sum_{i=t-\tau_{\rm mix}}^{t-1}\alpha_i,\\
\mathbb{E} \Big[ (\theta_t-\theta^*)^\top\big(f(s_t,\hat{s}_t,\theta_t)-\bar{f}(\theta_t)\big) \Big]
\le\;& \Big(448 \mathbb{E}\|\delta_t\|^2 + 160(r_{\rm max}+3\|\theta^*\|)^2\Big)
\sum_{i=t-\tau_{\rm mix}}^{t-1}\alpha_i.
\end{align*}
\end{lemma}

The proof of Lemma \ref{l: theta difference decaying} can be found in
Section \ref{sec: pof theta difference decaying}, and the proof of
Lemma \ref{l: markov noise decaying} can be found in
Section \ref{sec: pof markov noise decaying}.

\medskip
Using Lemma \ref{l: markov noise decaying} and summing the two bounds, we have
\begin{align*}
& 2\alpha_t \mathbb{E} \left[
\delta_t^\top\left(\bar{g}(\theta_t)+\bar{f}(\theta_t)
-\big(g(s_t,s_t',\theta_t)+f(s_t,\hat{s}_t,\theta_t)\big)\right)\right]\\
\le\;&
2\alpha_t\Big(1328 \mathbb{E}\|\delta_t\|^2 + 464(r_{\rm max}+3\|\theta^*\|)^2\Big)
\sum_{i=t-\tau_{\rm mix}}^{t-1}\alpha_i.
\end{align*}

Combining the basic expansion of \(\|\delta_{t+1}\|^2\) in Appendix \ref{appendix: constant proof} with the bounds derived above, for any \(t\in[\tau_{\rm mix},T]\),
\begin{align*}
\mathbb{E}\|\delta_{t+1}\|^2
\le\;& (1-2\eta\alpha_t)\mathbb{E}\|\delta_t\|^2
+18\alpha_t^2\Big(\mathbb{E}\|\delta_t\|^2+(r_{\rm max}+\|\theta^*\|)^2\Big) \\
&\quad +2\alpha_t\Big(1328 \mathbb{E}\|\delta_t\|^2 + 464(r_{\rm max}+3\|\theta^*\|)^2\Big)
\sum_{i=t-\tau_{\rm mix}}^{t-1}\alpha_i.
\end{align*}
Since \(\alpha_t\) is nonincreasing and \(\tau_{\rm mix}\ge 1\),
\[
\sum_{i=t-\tau_{\rm mix}}^{t-1}\alpha_i \ge \alpha_{t-1}\ge \alpha_t
\quad\Rightarrow\quad
\alpha_t^2\le \alpha_t\sum_{i=t-\tau_{\rm mix}}^{t-1}\alpha_i.
\]
Moreover, \((r_{\rm max}+\|\theta^*\|)^2\le (r_{\rm max}+3\|\theta^*\|)^2\).
Therefore,
\begin{align*}
\mathbb{E}\|\delta_{t+1}\|^2
\le\;& (1-2\eta\alpha_t)\mathbb{E}\|\delta_t\|^2
+\alpha_t\Big(2674\,\mathbb{E}\|\delta_t\|^2 + 946(r_{\rm max}+3\|\theta^*\|)^2\Big)
\sum_{i=t-\tau_{\rm mix}}^{t-1}\alpha_i.
\end{align*}

By Lemma \ref{l: sum alpha bound}(2), we have
\(\sum_{i=t-\tau_{\rm mix}}^{t-1}\alpha_i\le \eta/2694\), which implies
\(2674\sum_{i=t-\tau_{\rm mix}}^{t-1}\alpha_i \le \eta\). Hence,
\[
\mathbb{E}\|\delta_{t+1}\|^2
\le (1-\eta\alpha_t)\mathbb{E}\|\delta_t\|^2
+\beta_2\,\alpha_t\sum_{i=t-\tau_{\rm mix}}^{t-1}\alpha_i.
\]
Define \(\hat{\alpha}_t=\alpha_t\sum_{i=t-\tau_{\rm mix}}^{t-1}\alpha_i\). Then
\[
\mathbb{E}\|\delta_{t+1}\|^2
\le (1-\eta\alpha_t)\mathbb{E}\|\delta_t\|^2 + \beta_2\hat{\alpha}_t,
\qquad \forall\, t\in[\tau_{\rm mix},T].
\]
Recursively applying the above inequality from \(\tau_{\rm mix}\) to \(T\), we obtain
\begin{align*}
\mathbb{E}\|\delta_T\|^2
\le\;& \mathbb{E}\|\delta_{\tau_{\rm mix}}\|^2 \prod_{j=\tau_{\rm mix}}^{T-1}(1-\eta\alpha_j)
+\beta_2 \sum_{k=\tau_{\rm mix}}^{T-1}\hat{\alpha}_k\prod_{j=k+1}^{T-1}(1-\eta\alpha_j).
\end{align*}

\paragraph{Step 1: bound \(\mathbb{E}\|\delta_{\tau_{\rm mix}}\|^2\).}
By Lemma \ref{l: sum alpha bound}(2) applied at \(t=\tau_{\rm mix}\),
\(\sum_{i=0}^{\tau_{\rm mix}-1}\alpha_i\le 1/12\). Then Lemma \ref{l: theta difference decaying}
(with \(t_1=0,t_2=\tau_{\rm mix}\)) yields
\[
\|\theta_{\tau_{\rm mix}}-\theta_0\|
\le 12(\|\theta_0\|+r_{\rm max}+2\|\theta^*\|)\sum_{i=0}^{\tau_{\rm mix}-1}\alpha_i
\le \|\theta_0\|+r_{\rm max}+2\|\theta^*\|.
\]
Therefore,
\begin{align*}
\mathbb{E}\|\delta_{\tau_{\rm mix}}\|^2
&\le 2\mathbb{E}\|\theta_{\tau_{\rm mix}}-\theta_0\|^2 + 2\mathbb{E}\|\theta_0-\theta^*\|^2\\
&\le 2 \mathbb{E}\Big[ \big( \|\theta_0\| + r_{\rm max} + 2 \|\theta^*\| \big)^2 \Big]
  + 2 \mathbb{E}\big[\|\theta_0 - \theta^*\|^2\big]
= \beta_1.
\end{align*}

\paragraph{Step 2: bound the product term.}
Let
\[
I_1:=\prod_{j=\tau_{\rm mix}}^{T-1}(1-\eta\alpha_j).
\]
Using \(1-x\le e^{-x}\), we have
\begin{align*}
I_1
\le \exp\left(-\eta\sum_{j=\tau_{\rm mix}}^{T-1}\alpha_j\right)
= \exp\left(-\eta a \sum_{j=\tau_{\rm mix}}^{T-1}\frac{1}{(j+c_0)^{\xi}}\right)
\le \exp\left(-\eta a \int_{\tau_{\rm mix}}^{T}\frac{dx}{(x+c_0)^{\xi}}\right).
\end{align*}
Thus,
\[
I_1 \le
\begin{cases}
\left(\dfrac{\tau_{\rm mix}+c_0}{T+c_0}\right)^{\eta a}, & \xi=1,\\[8pt]
\exp\!\left(
-\dfrac{\eta a}{1-\xi}\Big((T+c_0)^{1-\xi}-(\tau_{\rm mix}+c_0)^{1-\xi}\Big)
\right), & \xi\in(0,1).
\end{cases}
\]

\paragraph{Step 3: bound the sum term.}
Let
\[
I_2:=\sum_{k=\tau_{\rm mix}}^{T-1}\hat{\alpha}_k\prod_{j=k+1}^{T-1}(1-\eta\alpha_j).
\]
By Lemma \ref{l: sum alpha bound}(1),
\[
\hat{\alpha}_k
=\alpha_k\sum_{i=k-\tau_{\rm mix}}^{k-1}\alpha_i
\le 2L_1\alpha_k^2\big(\log(1/\alpha_T)+1\big),
\qquad \forall\,k\in[\tau_{\rm mix},T].
\]
Hence,
\begin{align*}
\beta_2 I_2
&\le 2L_1\beta_2\big(\log(1/\alpha_T)+1\big)
\sum_{k=\tau_{\rm mix}}^{T-1}\alpha_k^2\prod_{j=k+1}^{T-1}(1-\eta\alpha_j).
\end{align*}
Since \(\log(1/\alpha_T)=\xi\log(T+c_0)-\log a \le \log(T+c_0)-\log a\), we can define
\[
\tilde{\beta}(T):=2L_1\beta_2\big(\log(T+c_0)-\log a+1\big)
\]
so that
\[
\beta_2 I_2
\le \tilde{\beta}(T)\cdot
\underbrace{\sum_{k=\tau_{\rm mix}}^{T-1}\alpha_k^2\prod_{j=k+1}^{T-1}(1-\eta\alpha_j)}_{=:I_3}.
\]
We next bound \(I_3\) in two cases.

\medskip
\noindent\textbf{Case 1: \(\xi=1\).}
In this case \(\alpha_k=\frac{a}{k+c_0}\), and we assume \(c_0\ge a\eta\).
Using \(1-x\le e^{-x}\),
\[
\prod_{j=k+1}^{T-1}\left(1-\frac{\eta a}{j+c_0}\right)
\le \exp\left(-\eta a\sum_{j=k+1}^{T-1}\frac{1}{j+c_0}\right)
\le \left(\frac{k+1+c_0}{T+c_0}\right)^{\eta a}.
\]
Therefore,
\begin{align*}
I_3
&=\sum_{k=\tau_{\rm mix}}^{T-1}\frac{a^2}{(k+c_0)^2}
\prod_{j=k+1}^{T-1}\left(1-\frac{\eta a}{j+c_0}\right)\\
&\le \sum_{k=\tau_{\rm mix}}^{T-1}\frac{a^2}{(k+c_0)^2}
\left(\frac{k+1+c_0}{T+c_0}\right)^{\eta a}
= \frac{a^2}{(T+c_0)^{\eta a}}
\sum_{k=\tau_{\rm mix}}^{T-1}\left(\frac{k+1+c_0}{k+c_0}\right)^2 (k+1+c_0)^{\eta a-2}\\
&\le \frac{4a^2}{(T+c_0)^{\eta a}}
\sum_{k=\tau_{\rm mix}}^{T-1}(k+1+c_0)^{\eta a-2}.
\end{align*}
Standard summation bounds give
\[
I_3 \le
\begin{cases}
\displaystyle \frac{4a^2}{1-\eta a}\cdot \frac{1}{(T+c_0)^{\eta a}},
& \eta a<1,\\[10pt]
\displaystyle 4a^2\frac{\log(T+c_0)}{T+c_0},
& \eta a=1,\\[10pt]
\displaystyle \frac{4ea^2}{\eta a-1}\cdot \frac{1}{T+c_0},
& \eta a>1.
\end{cases}
\]
Define \(q=\min\{1,\eta a\}\) and
\[
\Gamma(T):=
\begin{cases}
\displaystyle \frac{4a^2}{1-\eta a}\,\tilde{\beta}(T), & \eta a<1,\\[10pt]
\displaystyle 4a^2\log(T+c_0)\,\tilde{\beta}(T), & \eta a=1,\\[10pt]
\displaystyle \frac{4ea^2}{\eta a-1}\,\tilde{\beta}(T), & \eta a>1.
\end{cases}
\]
Then \(\tilde{\beta}(T)I_3 \le \Gamma(T)/(T+c_0)^q\). Combining with the bound on \(I_1\),
we obtain the \(\xi=1\) claim.

\medskip
\noindent\textbf{Case 2: \(\xi\in(0,1)\).}
Consider the sequence \(\{u_t\}_{t\ge \tau_{\rm mix}}\) defined by
\[
u_{\tau_{\rm mix}}=0,\qquad
u_{t+1}=\left(1-\frac{\eta a}{(t+c_0)^\xi}\right)u_t + \frac{a^2}{(t+c_0)^{2\xi}}.
\]
One can check that \(I_3=u_T\). To bound \(u_T\), we use the following lemma.

\begin{lemma}
\label{l: dynamic system with decaying stepsize}
Given a sequence \(\{x_t\}_{t\ge \tau}\) and positive constants \(c_0,c_1,c_2,\xi\),
consider the recursion
\[
x_{t+1}
=\left(1-\frac{c_1 c_2}{(t+c_0)^\xi}\right)x_t + \frac{c_2^2}{(t+c_0)^{2\xi}},
\]
with initial condition
\(
x_{\tau}\le \frac{2c_2}{c_1}\frac{1}{(\tau+c_0)^\xi}.
\)
Then \(x_t\le \frac{2c_2}{c_1}\frac{1}{(t+c_0)^\xi}\) for all \(t\ge \tau\) if either:
\begin{enumerate}
\item \(\xi=1\) and \(c_1c_2\ge 2\);
\item \(\xi\in(0,1)\) and \(\tau \ge (2\xi/(c_1c_2))^{1/(1-\xi)}-c_0\).
\end{enumerate}
\end{lemma}

The proof of Lemma \ref{l: dynamic system with decaying stepsize} can be found in
Section \ref{sec: pof dynamic system with decaying stepsize}.

Applying Lemma \ref{l: dynamic system with decaying stepsize} with
\(c_1=\eta\), \(c_2=a\), \(\tau=\tau_{\rm mix}\), and noting that
\(u_{\tau_{\rm mix}}=0\le \frac{2a}{\eta}\frac{1}{(\tau_{\rm mix}+c_0)^\xi}\),
we get (under the condition \(c_0\ge (\frac{2\xi}{\eta a})^{1/(1-\xi)}\))
\[
I_3=u_T\le \frac{2a}{\eta (T+c_0)^\xi}.
\]
Therefore,
\[
\tilde{\beta}(T)I_3
\le \frac{2a\,\tilde{\beta}(T)}{\eta (T+c_0)^\xi}
= \frac{\beta(T)}{\eta (T+c_0)^\xi},
\]
where \(\beta(T)=4aL_1\beta_2(\log(T+c_0)-\log a+1)\).
Combining with the bound on \(I_1\) proves the \(\xi\in(0,1)\) claim.
\end{proof}


\subsection{Proof of Lemma \ref{l: sum alpha bound}}
\label{sec: pof sum alpha bound}

\begin{proof}[Proof of Lemma \ref{l: sum alpha bound}]
By definition of \(\tau_{\rm mix}=\tau_{\rm mix}(\alpha_T)\), we have \(C\beta^{\tau_{\rm mix}}\le \alpha_T\),
which implies
\[
\tau_{\rm mix}
\le \frac{\log C + \log(1/\alpha_T)}{\log(1/\beta)}
\le L_1\big(\log(1/\alpha_T)+1\big),
\qquad
L_1=\max\left\{1,\frac{\log C+1}{\log(1/\beta)}\right\}.
\]
Since \(\alpha_t\) is nonincreasing,
\[
\sum_{i=t-\tau_{\rm mix}}^{t-1}\alpha_i
\le \tau_{\rm mix}\alpha_{t-\tau_{\rm mix}}
= \tau_{\rm mix}\alpha_t\left(\frac{t+c_0}{t-\tau_{\rm mix}+c_0}\right)^\xi.
\]
Note that \(\left(\frac{t+c_0}{t-\tau_{\rm mix}+c_0}\right)^\xi\to 1\) as \(c_0\to\infty\)
uniformly over \(t\in[\tau_{\rm mix},T]\). Hence, there exists \(c_{0,1}(a,\xi)\) large enough such that
for all \(t\in[\tau_{\rm mix},T]\),
\[
\left(\frac{t+c_0}{t-\tau_{\rm mix}+c_0}\right)^\xi \le 2,
\qquad \forall\, c_0\ge c_{0,1}(a,\xi).
\]
Using \(\tau_{\rm mix}\le L_1(\log(1/\alpha_T)+1)\), we obtain
\[
\sum_{i=t-\tau_{\rm mix}}^{t-1}\alpha_i
\le 2L_1\alpha_t\big(\log(1/\alpha_T)+1\big),
\qquad \forall\, t\in[\tau_{\rm mix},T],
\]
which proves part (1).

For part (2), since \(\alpha_t\to 0\) as \(c_0\to\infty\) (for fixed \(T\)) and the right-hand side in part (1)
is \(O(\alpha_t\log(1/\alpha_T))\), there exists \(c_{0,2}(a,\xi)\) large enough such that for all
\(c_0\ge c_{0,2}(a,\xi)\),
\[
2L_1\alpha_t\big(\log(1/\alpha_T)+1\big)\le \min\{1/12,\eta/2694\},
\qquad \forall\, t\in[\tau_{\rm mix},T].
\]
Combining with part (1) proves part (2).
\end{proof}

\subsection{Proof of Lemma \ref{l: theta difference decaying}}
\label{sec: pof theta difference decaying}

\begin{proof}[Proof of Lemma \ref{l: theta difference decaying}]
By Lemma \ref{l: bound on f and g linear},
\[
\|\theta_{t+1}\|-\|\theta_t\|
\le \|\theta_{t+1}-\theta_t\|
\le \alpha_t\|f(s_t,\hat{s}_t,\theta_t)+g(s_t,s_t',\theta_t)\|
\le 3\alpha_t(\|\theta_t\|+r_{\rm max}+2\|\theta^*\|).
\]
Thus,
\[
\|\theta_{t+1}\|+r_{\rm max}+2\|\theta^*\|
\le (1+3\alpha_t)(\|\theta_t\|+r_{\rm max}+2\|\theta^*\|).
\]
For any \(t\in[t_1,t_2]\),
\begin{align*}
\|\theta_t\|+r_{\rm max}+2\|\theta^*\|
&\le \prod_{i=t_1}^{t-1}(1+3\alpha_i)(\|\theta_{t_1}\|+r_{\rm max}+2\|\theta^*\|)\\
&\le \exp\left(3\sum_{i=t_1}^{t-1}\alpha_i\right)(\|\theta_{t_1}\|+r_{\rm max}+2\|\theta^*\|)\\
&\le \left(1+6\sum_{i=t_1}^{t-1}\alpha_i\right)(\|\theta_{t_1}\|+r_{\rm max}+2\|\theta^*\|),
\end{align*}
where the last step uses \(e^x\le 1+2x\) for \(x\le 1/2\), and here
\(x=3\sum_{i=t_1}^{t-1}\alpha_i\le 3\cdot(1/12)=1/4\).
In particular, since \(\sum_{i=t_1}^{t_2-1}\alpha_i\le 1/12\),
\[
\|\theta_t\|+r_{\rm max}+2\|\theta^*\|
\le 2(\|\theta_{t_1}\|+r_{\rm max}+2\|\theta^*\|).
\]
Therefore,
\begin{align*}
\|\theta_{t_2}-\theta_{t_1}\|
&\le \sum_{i=t_1}^{t_2-1}\|\theta_{i+1}-\theta_i\|
\le \sum_{i=t_1}^{t_2-1}3\alpha_i(\|\theta_i\|+r_{\rm max}+2\|\theta^*\|)\\
&\le 6(\|\theta_{t_1}\|+r_{\rm max}+2\|\theta^*\|)\sum_{i=t_1}^{t_2-1}\alpha_i,
\end{align*}
which proves part (1). For part (2), using \(\|\theta_{t_1}\|\le \|\theta_{t_2}\|+\|\theta_{t_2}-\theta_{t_1}\|\),
\begin{align*}
\|\theta_{t_2}-\theta_{t_1}\|
&\le 6(\|\theta_{t_1}\|+r_{\rm max}+2\|\theta^*\|)\sum_{i=t_1}^{t_2-1}\alpha_i\\
&\le 6(\|\theta_{t_2}\|+\|\theta_{t_2}-\theta_{t_1}\|+r_{\rm max}+2\|\theta^*\|)\sum_{i=t_1}^{t_2-1}\alpha_i\\
&\le \frac{1}{2}\|\theta_{t_2}-\theta_{t_1}\|
+6(\|\theta_{t_2}\|+r_{\rm max}+2\|\theta^*\|)\sum_{i=t_1}^{t_2-1}\alpha_i,
\end{align*}
where we used \(\sum_{i=t_1}^{t_2-1}\alpha_i\le 1/12\). Rearranging gives
\[
\|\theta_{t_2}-\theta_{t_1}\|
\le 12(\|\theta_{t_2}\|+r_{\rm max}+2\|\theta^*\|)\sum_{i=t_1}^{t_2-1}\alpha_i,
\]
which proves part (2).
\end{proof}

\subsection{Proof of Lemma \ref{l: markov noise decaying}}
\label{sec: pof markov noise decaying}

\begin{proof}[Proof of Lemma \ref{l: markov noise decaying}]
For simplicity, denote
\[
g_t(\theta)=g(s_t,s_t',\theta),\qquad f_t(\theta)=f(s_t,\hat{s}_t,\theta).
\]
Throughout this proof we fix \(t\in[\tau_{\rm mix},T]\) and assume
\(\sum_{i=t-\tau_{\rm mix}}^{t-1}\alpha_i\le 1/12\).

By Lemma \ref{l: theta difference decaying} (part (2)),
\begin{align*}
\|\theta_t-\theta_{t-\tau_{\rm mix}}\|
&\le 12\big(\|\theta_t\|+r_{\rm max}+2\|\theta^*\|\big)\sum_{i=t-\tau_{\rm mix}}^{t-1}\alpha_i\\
&\le 12\big(\|\delta_t\|+r_{\rm max}+3\|\theta^*\|\big)\sum_{i=t-\tau_{\rm mix}}^{t-1}\alpha_i.
\end{align*}

\paragraph{Step 1: bound the noise term for \(g\).}
Decompose
\begin{align*}
&\mathbb{E}\Big[(\theta_t-\theta^*)^\top\big(g_t(\theta_t)-\bar g(\theta_t)\big)\Big]\\
=\;&
\underbrace{\mathbb{E}\Big[(\theta_t-\theta_{t-\tau_{\rm mix}})^\top\big(g_t(\theta_t)-\bar g(\theta_t)\big)\Big]}_{I_1}
+\underbrace{\mathbb{E}\Big[(\theta_{t-\tau_{\rm mix}}-\theta^*)^\top\big(g_t(\theta_{t-\tau_{\rm mix}})-\bar g(\theta_{t-\tau_{\rm mix}})\big)\Big]}_{I_2}\\
&\quad
+\underbrace{\mathbb{E}\Big[(\theta_{t-\tau_{\rm mix}}-\theta^*)^\top\big(g_t(\theta_t)-g_t(\theta_{t-\tau_{\rm mix}})\big)\Big]}_{I_3}
+\underbrace{\mathbb{E}\Big[(\theta_{t-\tau_{\rm mix}}-\theta^*)^\top\big(\bar g(\theta_{t-\tau_{\rm mix}})-\bar g(\theta_t)\big)\Big]}_{I_4}.
\end{align*}

\textbf{Term \(I_1\).}
By Lemma \ref{l: bound on f and g linear},
\(\|g_t(\theta_t)\|\le 2(\|\delta_t\|+r_{\rm max}+\|\theta^*\|)\), hence
\(\|g_t(\theta_t)-\bar g(\theta_t)\|\le 4(\|\delta_t\|+r_{\rm max}+\|\theta^*\|)\).
Thus,
\begin{align*}
I_1
&\le \mathbb{E}\Big[\|\theta_t-\theta_{t-\tau_{\rm mix}}\|\cdot \|g_t(\theta_t)-\bar g(\theta_t)\|\Big]\\
&\le \mathbb{E}\Big[
12\sum_{i=t-\tau_{\rm mix}}^{t-1}\alpha_i(\|\delta_t\|+r_{\rm max}+3\|\theta^*\|)
\cdot 4(\|\delta_t\|+r_{\rm max}+\|\theta^*\|)
\Big]\\
&\le 48\sum_{i=t-\tau_{\rm mix}}^{t-1}\alpha_i\,
\mathbb{E}\Big[(\|\delta_t\|+r_{\rm max}+3\|\theta^*\|)^2\Big]\\
&\le 96\sum_{i=t-\tau_{\rm mix}}^{t-1}\alpha_i\,
\Big(\mathbb{E}\|\delta_t\|^2+(r_{\rm max}+3\|\theta^*\|)^2\Big).
\end{align*}

\textbf{Term \(I_2\).}
First note that
\begin{align*}
\|\delta_{t-\tau_{\rm mix}}\|
&\le \|\delta_t\|+\|\theta_t-\theta_{t-\tau_{\rm mix}}\|\\
&\le \|\delta_t\| + 12(\|\delta_t\|+r_{\rm max}+3\|\theta^*\|)\sum_{i=t-\tau_{\rm mix}}^{t-1}\alpha_i\\
&\le 2\|\delta_t\|+r_{\rm max}+3\|\theta^*\|,
\end{align*}
where we used \(\sum_{i=t-\tau_{\rm mix}}^{t-1}\alpha_i\le 1/12\).

Next, by Eq.(\ref{eq: mixing})
\[
\sup_{s}\|P_{\tau_{\rm mix}}(\cdot\mid s)-\mu\|_1
\le C\beta^{\tau_{\rm mix}}\le \alpha_T\le \alpha_t.
\]
Therefore,
\begin{align*}
\Big\|\mathbb{E}\big[g_t(\theta_{t-\tau_{\rm mix}})-\bar g(\theta_{t-\tau_{\rm mix}})\mid \mathcal{F}_{t-\tau_{\rm mix}}\big]\Big\|
&\le \sup_{s}\|P_{\tau_{\rm mix}}(\cdot\mid s)-\mu\|_1 \cdot \sup_{s,s'}\|g(s,s',\theta_{t-\tau_{\rm mix}})\|\\
&\le \alpha_t \cdot 2(\|\delta_{t-\tau_{\rm mix}}\|+r_{\rm max}+\|\theta^*\|)\\
&\le 4\alpha_t(\|\delta_t\|+r_{\rm max}+2\|\theta^*\|).
\end{align*}
Hence,
\begin{align*}
I_2
&\le \mathbb{E}\Big[\|\delta_{t-\tau_{\rm mix}}\|\cdot
\Big\|\mathbb{E}\big[g_t(\theta_{t-\tau_{\rm mix}})-\bar g(\theta_{t-\tau_{\rm mix}})\mid \mathcal{F}_{t-\tau_{\rm mix}}\big]\Big\|\Big]\\
&\le \mathbb{E}\Big[(2\|\delta_t\|+r_{\rm max}+3\|\theta^*\|)\cdot
4\alpha_t(\|\delta_t\|+r_{\rm max}+2\|\theta^*\|)\Big]\\
&\le 8\alpha_t\,\mathbb{E}\Big[(\|\delta_t\|+r_{\rm max}+2\|\theta^*\|)^2\Big]\\
&\le 16\alpha_t\Big(\mathbb{E}\|\delta_t\|^2+(r_{\rm max}+2\|\theta^*\|)^2\Big).
\end{align*}

\textbf{Terms \(I_3\) and \(I_4\).}
By Lemma \ref{l: lipschitz f and g}, \(g(\cdot)\) is \(2\)-Lipschitz in \(\theta\), and
\(\bar g(\cdot)\) is also \(2\)-Lipschitz. Thus,
\begin{align*}
I_3
&\le \mathbb{E}\Big[\|\delta_{t-\tau_{\rm mix}}\|\cdot \|g_t(\theta_t)-g_t(\theta_{t-\tau_{\rm mix}})\|\Big]
\le 2\mathbb{E}\Big[\|\delta_{t-\tau_{\rm mix}}\|\cdot \|\theta_t-\theta_{t-\tau_{\rm mix}}\|\Big],\\
I_4
&\le \mathbb{E}\Big[\|\delta_{t-\tau_{\rm mix}}\|\cdot \|\bar g(\theta_t)-\bar g(\theta_{t-\tau_{\rm mix}})\|\Big]
\le 2\mathbb{E}\Big[\|\delta_{t-\tau_{\rm mix}}\|\cdot \|\theta_t-\theta_{t-\tau_{\rm mix}}\|\Big].
\end{align*}
Using the bounds on \(\|\delta_{t-\tau_{\rm mix}}\|\) and \(\|\theta_t-\theta_{t-\tau_{\rm mix}}\|\),
\begin{align*}
I_3
&\le \mathbb{E}\Big[(2\|\delta_t\|+r_{\rm max}+3\|\theta^*\|)
\cdot 24(\|\delta_t\|+r_{\rm max}+3\|\theta^*\|)\sum_{i=t-\tau_{\rm mix}}^{t-1}\alpha_i\Big]\\
&\le 48\sum_{i=t-\tau_{\rm mix}}^{t-1}\alpha_i\,
\mathbb{E}\Big[(2\|\delta_t\|+r_{\rm max}+3\|\theta^*\|)^2\Big]\\
&\le 96\sum_{i=t-\tau_{\rm mix}}^{t-1}\alpha_i\,
\Big(4\mathbb{E}\|\delta_t\|^2+(r_{\rm max}+3\|\theta^*\|)^2\Big),
\end{align*}
and the same bound holds for \(I_4\).

Combining \(I_1\)--\(I_4\), we get
\begin{align*}
\mathbb{E}\Big[(\theta_t-\theta^*)^\top(g_t(\theta_t)-\bar g(\theta_t))\Big]
\le\;& 16\alpha_t\Big(\mathbb{E}\|\delta_t\|^2+(r_{\rm max}+2\|\theta^*\|)^2\Big)\\
&\quad +\Big(864\mathbb{E}\|\delta_t\|^2+288(r_{\rm max}+3\|\theta^*\|)^2\Big)
\sum_{i=t-\tau_{\rm mix}}^{t-1}\alpha_i.
\end{align*}
Since \(\alpha_t\le \sum_{i=t-\tau_{\rm mix}}^{t-1}\alpha_i\), the above implies
\[
\mathbb{E}\Big[(\theta_t-\theta^*)^\top(g_t(\theta_t)-\bar g(\theta_t))\Big]
\le \Big(880 \mathbb{E}\|\delta_t\|^2 + 304(r_{\rm max}+3\|\theta^*\|)^2\Big)
\sum_{i=t-\tau_{\rm mix}}^{t-1}\alpha_i.
\]

\paragraph{Step 2: bound the noise term for \(f\).}
Repeat the same decomposition for \(f\):
\begin{align*}
&\mathbb{E}\Big[(\theta_t-\theta^*)^\top\big(f_t(\theta_t)-\bar f(\theta_t)\big)\Big]\\
=\;&
\underbrace{\mathbb{E}\Big[(\theta_t-\theta_{t-\tau_{\rm mix}})^\top\big(f_t(\theta_t)-\bar f(\theta_t)\big)\Big]}_{I_1'}
+\underbrace{\mathbb{E}\Big[(\theta_{t-\tau_{\rm mix}}-\theta^*)^\top\big(f_t(\theta_{t-\tau_{\rm mix}})-\bar f(\theta_{t-\tau_{\rm mix}})\big)\Big]}_{I_2'}\\
&\quad
+\underbrace{\mathbb{E}\Big[(\theta_{t-\tau_{\rm mix}}-\theta^*)^\top\big(f_t(\theta_t)-f_t(\theta_{t-\tau_{\rm mix}})\big)\Big]}_{I_3'}
+\underbrace{\mathbb{E}\Big[(\theta_{t-\tau_{\rm mix}}-\theta^*)^\top\big(\bar f(\theta_{t-\tau_{\rm mix}})-\bar f(\theta_t)\big)\Big]}_{I_4'}.
\end{align*}

\textbf{Term \(I_1'\).}
By Lemma \ref{l: bound on f and g linear},
\(\|f_t(\theta_t)\|\le \|\delta_t\|+r_{\rm max}+\|\theta^*\|\), hence
\(\|f_t(\theta_t)-\bar f(\theta_t)\|\le 2(\|\delta_t\|+r_{\rm max}+\|\theta^*\|)\).
Thus,
\begin{align*}
I_1'
&\le \mathbb{E}\Big[\|\theta_t-\theta_{t-\tau_{\rm mix}}\|\cdot \|f_t(\theta_t)-\bar f(\theta_t)\|\Big]\\
&\le \mathbb{E}\Big[
12\sum_{i=t-\tau_{\rm mix}}^{t-1}\alpha_i(\|\delta_t\|+r_{\rm max}+3\|\theta^*\|)
\cdot 2(\|\delta_t\|+r_{\rm max}+\|\theta^*\|)\Big]\\
&\le 48\sum_{i=t-\tau_{\rm mix}}^{t-1}\alpha_i\,
\Big(\mathbb{E}\|\delta_t\|^2+(r_{\rm max}+\|\theta^*\|)^2\Big).
\end{align*}

\textbf{Term \(I_2'\).}
By Eq.(\ref{eq: mixing}) (for both chains \(s_t\) and \(\hat s_t\)),
\[
\sup_{s}\|P_{\tau_{\rm mix}}(\cdot\mid s)-\mu\|_1 \le \alpha_t,
\qquad
\sup_{\hat s}\|P_{\tau_{\rm mix}}(\cdot\mid \hat s)-\mu\|_1 \le \alpha_t.
\]
Hence,
\begin{align*}
\Big\|\mathbb{E}\big[f_t(\theta_{t-\tau_{\rm mix}})-\bar f(\theta_{t-\tau_{\rm mix}})\mid \mathcal{F}_{t-\tau_{\rm mix}}\big]\Big\|
&\le \Big(\sup_s\|P_{\tau_{\rm mix}}(\cdot\mid s)-\mu\|_1
+\sup_{\hat s}\|P_{\tau_{\rm mix}}(\cdot\mid \hat s)-\mu\|_1\Big)\cdot \sup_{s,\hat s}\|f(s,\hat s,\theta_{t-\tau_{\rm mix}})\|\\
&\le 2\alpha_t\cdot (\|\delta_{t-\tau_{\rm mix}}\|+r_{\rm max}+\|\theta^*\|)\\
&\le 4\alpha_t(\|\delta_t\|+r_{\rm max}+2\|\theta^*\|).
\end{align*}
Therefore,
\begin{align*}
I_2'
&\le \mathbb{E}\Big[\|\delta_{t-\tau_{\rm mix}}\|\cdot
\Big\|\mathbb{E}\big[f_t(\theta_{t-\tau_{\rm mix}})-\bar f(\theta_{t-\tau_{\rm mix}})\mid \mathcal{F}_{t-\tau_{\rm mix}}\big]\Big\|\Big]\\
&\le 16\alpha_t\Big(\mathbb{E}\|\delta_t\|^2+(r_{\rm max}+2\|\theta^*\|)^2\Big),
\end{align*}
where we used the same bound \(\|\delta_{t-\tau_{\rm mix}}\|\le 2\|\delta_t\|+r_{\rm max}+3\|\theta^*\|\).

\textbf{Terms \(I_3'\) and \(I_4'\).}
By Lemma \ref{l: lipschitz f and g}, \(f(\cdot)\) is \(1\)-Lipschitz in \(\theta\), and \(\bar f(\cdot)\)
is also \(1\)-Lipschitz, hence
\[
I_3'\le \mathbb{E}\big[\|\delta_{t-\tau_{\rm mix}}\|\cdot \|\theta_t-\theta_{t-\tau_{\rm mix}}\|\big],
\qquad
I_4'\le \mathbb{E}\big[\|\delta_{t-\tau_{\rm mix}}\|\cdot \|\theta_t-\theta_{t-\tau_{\rm mix}}\|\big].
\]
Thus,
\begin{align*}
I_3'
&\le \mathbb{E}\Big[(2\|\delta_t\|+r_{\rm max}+3\|\theta^*\|)
\cdot 12(\|\delta_t\|+r_{\rm max}+3\|\theta^*\|)\sum_{i=t-\tau_{\rm mix}}^{t-1}\alpha_i\Big]\\
&\le 24\sum_{i=t-\tau_{\rm mix}}^{t-1}\alpha_i\,
\mathbb{E}\Big[(2\|\delta_t\|+r_{\rm max}+3\|\theta^*\|)^2\Big]\\
&\le 48\sum_{i=t-\tau_{\rm mix}}^{t-1}\alpha_i\,
\Big(4\mathbb{E}\|\delta_t\|^2+(r_{\rm max}+3\|\theta^*\|)^2\Big),
\end{align*}
and the same bound holds for \(I_4'\).

Combining \(I_1'\)--\(I_4'\), we obtain
\begin{align*}
\mathbb{E}\Big[(\theta_t-\theta^*)^\top(f_t(\theta_t)-\bar f(\theta_t))\Big]
\le\;& 16\alpha_t\Big(\mathbb{E}\|\delta_t\|^2+(r_{\rm max}+2\|\theta^*\|)^2\Big)\\
&\quad +\Big(432\mathbb{E}\|\delta_t\|^2+144(r_{\rm max}+3\|\theta^*\|)^2\Big)
\sum_{i=t-\tau_{\rm mix}}^{t-1}\alpha_i.
\end{align*}
Using \(\alpha_t\le \sum_{i=t-\tau_{\rm mix}}^{t-1}\alpha_i\) gives the desired bound for \(f\), and
this completes the proof of Lemma \ref{l: markov noise decaying}.
\end{proof}

\subsection{Proof of Lemma \ref{l: dynamic system with decaying stepsize}}
\label{sec: pof dynamic system with decaying stepsize}

\begin{proof}[Proof of Lemma \ref{l: dynamic system with decaying stepsize}]
We use induction. The base case is assumed in the statement. Now suppose
\(x_t\le \frac{2c_2}{c_1}\frac{1}{(t+c_0)^\xi}\). Then
\begin{align*}
\frac{2c_2}{c_1}\frac{1}{(t+1+c_0)^\xi} - x_{t+1}
&= \frac{2c_2}{c_1}\frac{1}{(t+1+c_0)^\xi}
-\left(1-\frac{c_1c_2}{(t+c_0)^\xi}\right)x_t - \frac{c_2^2}{(t+c_0)^{2\xi}}\\
&\ge \frac{2c_2}{c_1}\frac{1}{(t+1+c_0)^\xi}
-\left(1-\frac{c_1c_2}{(t+c_0)^\xi}\right)\frac{2c_2}{c_1}\frac{1}{(t+c_0)^\xi}
-\frac{c_2^2}{(t+c_0)^{2\xi}}\\
&= \frac{2c_2}{c_1}\left[
\frac{1}{(t+1+c_0)^\xi}-\frac{1}{(t+c_0)^\xi}+\frac{c_1c_2}{2}\frac{1}{(t+c_0)^{2\xi}}
\right]\\
&=\frac{2c_2}{c_1}\frac{1}{(t+c_0)^{2\xi}}
\left[
\frac{c_1c_2}{2}-(t+c_0)^\xi\left(1-\left(\frac{t+c_0}{t+1+c_0}\right)^\xi\right)
\right].
\end{align*}
Observe that
\[
\left(\frac{t+c_0}{t+1+c_0}\right)^\xi
=\left(1+\frac{1}{t+c_0}\right)^{-\xi}
\ge \exp\left(-\frac{\xi}{t+c_0}\right)\ge 1-\frac{\xi}{t+c_0},
\]
where we used \(e^x\ge 1+x\). Hence,
\[
(t+c_0)^\xi\left(1-\left(\frac{t+c_0}{t+1+c_0}\right)^\xi\right)
\le (t+c_0)^\xi\cdot \frac{\xi}{t+c_0}
= \frac{\xi}{(t+c_0)^{1-\xi}}.
\]
Therefore,
\[
\frac{2c_2}{c_1}\frac{1}{(t+1+c_0)^\xi}-x_{t+1}
\ge \frac{2c_2}{c_1}\frac{1}{(t+c_0)^{2\xi}}
\left(\frac{c_1c_2}{2}-\frac{\xi}{(t+c_0)^{1-\xi}}\right).
\]
The last term is nonnegative under either condition in the lemma statement, which yields
\(x_{t+1}\le \frac{2c_2}{c_1}\frac{1}{(t+1+c_0)^\xi}\). This completes the induction.
\end{proof}

%% file: sections/one_constant.tex
\section{Analysis of Single Markov Chain Case with Constant Step-size}
\label{appendix: one chain constant}

Recall the single-chain update components in Eq.~(\ref{eq: onechain stochastic averaged reward update}):
\[
f(s,w):=\phi(s)-w,
\qquad
g(s,s',w,\theta):=\big(r_s+\phi(s')^\top\theta-\phi(s)^\top\theta\big)\phi(s) - \big(r_s+\phi(s)^\top\theta\big)w .
\]
Under stationarity \((s,s')\sim(\mu,P)\), define the mean fields
\begin{align}
\bar f(w) &:= \mathbb{E}_{s\sim\mu}[f(s,w)] = \Phi^\top\mu - w = w^*-w, \label{eq: barf_onechain}\\
\bar g(w,\theta) &:= \mathbb{E}_{s\sim\mu,\ s'\sim P(\cdot|s)}[g(s,s',w,\theta)]
= \Phi^\top D\big(R+P\Phi\theta-\Phi\theta\big) - \mu^\top(R+\Phi\theta)\,w. \label{eq: barg_onechain}
\end{align}
The fixed points are \(w^*=\Phi^\top\mu\), and \(\theta^*\) satisfies
\begin{equation}
\bar g(w^*,\theta^*)=0.
\label{eq: theta_star_onechain}
\end{equation}
First, we introduce some useful basic properties. 

\begin{lemma}
\label{l: delta bound one chain constant}
For all \(t\ge 0\), \(\|w_t\|\le 1\) and \(\|\delta_t^w\|\le 2\).
\end{lemma}

\begin{proof}[Proof of Lemma~\ref{l: delta bound one chain constant}]
The update \(w_{t+1}=(1-\beta)w_t+\beta\phi(s_t)\) implies
\[
\|w_{t+1}\|\le (1-\beta)\|w_t\|+\beta\|\phi(s_t)\|\le (1-\beta)\|w_t\|+\beta.
\]
Since \(\|w_0\|\le 1\), induction gives \(\|w_t\|\le 1\) for all \(t\).
Also \(\|w^*\|=\|\mathbb{E}_{s\sim\mu}[\phi(s)]\|\le \mathbb{E}\|\phi(s)\|\le 1\), hence \(\|\delta_t^w\|\le 2\).
\end{proof}

\begin{lemma}
\label{l: bound on f and g onechain}
For all \(t\ge 0\),
\[
\|f(s_t,w_t)\|\le 2,
\qquad
\|g(s_t,s_t',w_t,\theta_t)\|\le 2r_{\max}+6R_\theta = 2(r_{\max}+3R_\theta).
\]
\end{lemma}

\begin{proof}[Proof of Lemma~\ref{l: bound on f and g onechain}]
For \(f(s,w)=\phi(s)-w\), \(\|f(s_t,w_t)\|\le \|\phi(s_t)\|+\|w_t\|\le 2\).
For \(g\), using \(\|\phi(\cdot)\|\le 1\) and \(\|w_t\|\le 1\),
\begin{align*}
\|g(s_t,s_t',w_t,\theta_t)\|
&\le \big|r_{s_t}+(\phi(s_t')-\phi(s_t))^\top\theta_t\big|\cdot\|\phi(s_t)\|
+\big|r_{s_t}+\phi(s_t)^\top\theta_t\big|\cdot\|w_t\|\\
&\le \big(r_{\max}+2\|\theta_t\|\big)+\big(r_{\max}+\|\theta_t\|\big)
\le 2r_{\max}+3\|\theta_t\|\le 2r_{\max}+6R_\theta.
\end{align*}
\end{proof}

\begin{lemma}
\label{l: lipschitz of f and g one chain}
For all \(s,s'\) and all \(w,w_1,w_2,\theta,\theta_1,\theta_2\),
\[
\|f(s,w_1)-f(s,w_2)\|\le \|w_1-w_2\|,
\]
\[
\|g(s,s',w,\theta_1)-g(s,s',w,\theta_2)\|\le 3\|\theta_1-\theta_2\|,
\]
and
\[
\|g(s,s',w_1,\theta)-g(s,s',w_2,\theta)\|
\le (r_{\max}+2R_\theta)\,\|w_1-w_2\|.
\]
\end{lemma}

\begin{proof}[Proof of Lemma~\ref{l: lipschitz of f and g one chain}]
The \(f\) claim is immediate:
\(\|f(s,w_1)-f(s,w_2)\|=\|w_1-w_2\|\).

For \(\theta\),
\begin{align*}
g(s,s',w,\theta_1)-g(s,s',w,\theta_2)
&=\big((\phi(s')-\phi(s))^\top(\theta_1-\theta_2)\big)\phi(s)
-\big(\phi(s)^\top(\theta_1-\theta_2)\big)w,
\end{align*}
so
\[
\|g(\cdot,\theta_1)-g(\cdot,\theta_2)\|
\le \|\phi(s')-\phi(s)\|\cdot\|\theta_1-\theta_2\|\cdot\|\phi(s)\|
+\|\phi(s)\|\cdot\|\theta_1-\theta_2\|\cdot\|w\|
\le (2+1)\|\theta_1-\theta_2\|.
\]
For \(w\),
\[
g(s,s',w_1,\theta)-g(s,s',w_2,\theta)=-(r_s+\phi(s)^\top\theta)(w_1-w_2),
\]
hence \(\|g(\cdot,w_1)-g(\cdot,w_2)\|\le (r_{\max}+\|\theta\|)\|w_1-w_2\|\le (r_{\max}+2R_\theta)\|w_1-w_2\|\).
\end{proof}

\subsection{Proof of Theorem~\ref{t: one chain constant}}

We first restate the theorem with all the constants.

\begin{theorem}[Full statement of Theorem \ref{t: one chain constant}]
Consider the single-chain algorithm in Eq.~(\ref{eq: onechain stochastic averaged reward update}) with Markov sampling.
Assume constant stepsizes \(\alpha_t\equiv \alpha>0\), \(\beta_t\equiv \beta>0\) and \(\rho_0:=\beta/\alpha\le 1\). Let \(\lambda>0\) satisfy
\(
0<\lambda^2< \frac{2\eta}{r_{\max}+2R_\theta}
\)
and define
\begin{align*}
\zeta := & \eta-\frac{\lambda^2(r_{\max}+2R_\theta)}{2}>0, \\
\kappa := & 1-2\alpha\zeta \in (0,1), \\
G_1:= & e^{-2\rho_0\alpha}.    
\end{align*}
Then for any \(t\ge \tau_{\rm mix}\),
\begin{align*}
\mathbb{E}\|\delta_t^\theta\|^2
\le\;&
\frac{4\alpha (r_{\max}+2R_\theta)\,(t-\tau_{\rm mix})\,\max\{\kappa,G_1\}^{\,t-\tau_{\rm mix}}}{\lambda^2} \\
& \quad \;+\; 4R_\theta^2\,\kappa^{\,t-\tau_{\rm mix}}
\;+\;\frac{\alpha\,G_{\rm const}}{2\zeta},
\end{align*}
where
\begin{align*}
G_{\rm const}
:= & \frac{(16\tau_{\rm mix}+6)(r_{\max}+2R_\theta)}{\lambda^2} \\
& \quad \;+\; (88\tau_{\rm mix}+4)\,(r_{\max}+3R_\theta)^2. 
\end{align*}
\end{theorem}

\begin{proof}[Proof of Theorem~\ref{t: one chain constant}]
By non-expansiveness of projection,
\begin{equation}
\|\delta_{t+1}^\theta\|^2 \le \|\delta_t^\theta+\alpha\,g(s_t,s_t',w_t,\theta_t)\|^2.
\label{eq: theta_basic_expand}
\end{equation}
Taking expectation and expanding,
\begin{align}
\mathbb{E}\|\delta_{t+1}^\theta\|^2
\le\;&
\mathbb{E}\|\delta_t^\theta\|^2
+2\alpha\,\mathbb{E}\Big[{\delta_t^\theta}^\top g(s_t,s_t',w_t,\theta_t)\Big]
+\alpha^2\mathbb{E}\|g(s_t,s_t',w_t,\theta_t)\|^2 \nonumber\\
=\;&
\mathbb{E}\|\delta_t^\theta\|^2
+2\alpha\,\mathbb{E}\Big[{\delta_t^\theta}^\top\big(g(s_t,s_t',w_t,\theta_t)-\bar g(w_t,\theta_t)\big)\Big]
+2\alpha\,\mathbb{E}\Big[{\delta_t^\theta}^\top \bar g(w_t,\theta_t)\Big]
+\alpha^2\mathbb{E}\|g(s_t,s_t',w_t,\theta_t)\|^2. \label{eq: theta_decomp}
\end{align}

\textbf{Step 1: Markov-noise term.}

\begin{lemma}
\label{l: markov noise one chain}
For any \(t\ge \tau_{\rm mix}\),
\[
\mathbb{E}\Big[{\delta_t^\theta}^\top\big(g(s_t,s_t',w_t,\theta_t)-\bar g(w_t,\theta_t)\big)\Big]
\;\le\; 44\,\alpha\,\tau_{\rm mix}\,(r_{\max}+3R_\theta)^2.
\]
\end{lemma}

The proof of this lemma can be found in Section \ref{sec: pof markov noise one chain}. 

By Lemma~\ref{l: markov noise one chain}, for \(t\ge \tau_{\rm mix}\),
\begin{equation}
2\alpha\,\mathbb{E}\Big[{\delta_t^\theta}^\top(g(s_t,s_t',w_t,\theta_t)-\bar g(w_t,\theta_t))\Big]
\le 88\,\alpha^2\,\tau_{\rm mix}\,(r_{\max}+3R_\theta)^2.
\label{eq: theta_noise_term}
\end{equation}

\textbf{Step 2: Mean-drift term.}
Using \(\bar g(w^*,\theta^*)=0\),
\begin{align*}
{\delta_t^\theta}^\top \bar g(w_t,\theta_t)
&=
-\|\Phi\delta_t^\theta\|_{\rm Dir}^2
-(\mu^\top\Phi\delta_t^\theta)^2
-\big(\mu^\top(R+\Phi\theta_t)\big)\,({\delta_t^\theta}^\top \delta_t^w).
\end{align*}
By Eq.\eqref{eq: eta} and Young's inequality, for any \(\lambda>0\),
\[
\big|\mu^\top(R+\Phi\theta_t)\big|\,\big|{\delta_t^\theta}^\top \delta_t^w\big|
\le (r_{\max}+2R_\theta)\cdot \frac{1}{2}\Big(\lambda^2\|\delta_t^\theta\|^2+\|\delta_t^w\|^2/\lambda^2\Big).
\]
Therefore, with \(\zeta=\eta-\lambda^2(r_{\max}+2R_\theta)/2>0\),
\begin{equation}
2\alpha\,\mathbb{E}\big[{\delta_t^\theta}^\top \bar g(w_t,\theta_t)\big]
\le
-2\alpha\zeta\,\mathbb{E}\|\delta_t^\theta\|^2
+\alpha\,\frac{r_{\max}+2R_\theta}{\lambda^2}\,\mathbb{E}\|\delta_t^w\|^2.
\label{eq: theta_drift_term}
\end{equation}

\textbf{Step 3: Second-moment term.}
By Lemma~\ref{l: bound on f and g onechain},
\begin{equation}
\alpha^2\mathbb{E}\|g\|^2 \le 4\alpha^2\,(r_{\max}+3R_\theta)^2.
\label{eq: theta_second_moment}
\end{equation}

Combining \eqref{eq: theta_decomp}--\eqref{eq: theta_second_moment}, for \(t\ge \tau_{\rm mix}\),
\begin{align}
\mathbb{E}\|\delta_{t+1}^\theta\|^2
\le\;&
(1-2\alpha\zeta)\,\mathbb{E}\|\delta_t^\theta\|^2
+\alpha\,\frac{r_{\max}+2R_\theta}{\lambda^2}\,\mathbb{E}\|\delta_t^w\|^2
+\alpha^2(88\tau_{\rm mix}+4)(r_{\max}+3R_\theta)^2. \label{eq: theta_recursion_pre_w}
\end{align}

To deal with \(\mathbb{E}\|\delta_t^w\|^2\), we introduce the following lemma:

\begin{lemma}
\label{l: bound on delta w}
For any \(t\ge \tau_{\rm mix}\),
\[
\mathbb{E}\|\delta_t^w\|^2 \;\le\; 4 e^{-2\beta(t-\tau_{\rm mix})} + (16\tau_{\rm mix}+6)\beta
=4\,G_1^{\,t-\tau_{\rm mix}}+(16\tau_{\rm mix}+6)\beta.
\]
\end{lemma}

The proof of this lemma can be found in Section \ref{sec: pof bound on delta w}. 

Now apply Lemma~\ref{l: bound on delta w} and use \(\beta=\rho_0\alpha\le \alpha\):
\begin{align}
\mathbb{E}\|\delta_{t+1}^\theta\|^2
\le\;&
\kappa\,\mathbb{E}\|\delta_t^\theta\|^2
+\frac{4\alpha(r_{\max}+2R_\theta)}{\lambda^2}\,G_1^{\,t-\tau_{\rm mix}}
+\rho_0\alpha^2\frac{(16\tau_{\rm mix}+6)(r_{\max}+2R_\theta)}{\lambda^2}
+\alpha^2(88\tau_{\rm mix}+4)(r_{\max}+3R_\theta)^2. \label{eq: theta_recursion_final}
\end{align}

Iterate \eqref{eq: theta_recursion_final} from \(\tau_{\rm mix}\) to \(t\) and use \(\|\delta_{\tau_{\rm mix}}^\theta\|\le 2R_\theta\):
\begin{align*}
\mathbb{E}\|\delta_t^\theta\|^2
\le\;&
\kappa^{t-\tau_{\rm mix}}\mathbb{E}\|\delta_{\tau_{\rm mix}}^\theta\|^2
+\frac{4\alpha(r_{\max}+2R_\theta)}{\lambda^2}\sum_{j=\tau_{\rm mix}}^{t-1}\kappa^{t-1-j}G_1^{\,j-\tau_{\rm mix}}
+\Big(\alpha^2 G_{\rm const}\Big)\sum_{j=\tau_{\rm mix}}^{t-1}\kappa^{t-1-j},
\end{align*}
where
\[
G_{\rm const}:=
\frac{\rho_0 (16\tau_{\rm mix}+6)(r_{\max}+2R_\theta)}{\lambda^2}
+(88\tau_{\rm mix}+4)(r_{\max}+3R_\theta)^2.
\]
Use
\[
\sum_{j=\tau_{\rm mix}}^{t-1}\kappa^{t-1-j}G_1^{\,j-\tau_{\rm mix}}
\le (t-\tau_{\rm mix})\,\max\{\kappa,G_1\}^{\,t-\tau_{\rm mix}},
\qquad
\sum_{j=\tau_{\rm mix}}^{t-1}\kappa^{t-1-j}\le \frac{1}{1-\kappa}=\frac{1}{2\alpha\zeta}.
\]
Then we complete the proof. 
\end{proof}

\subsection{Proof of Lemma~\ref{l: markov noise one chain}}
\label{sec: pof markov noise one chain}

\begin{proof}[Proof of Lemma~\ref{l: markov noise one chain}]
For brevity write \(g_t:=g(s_t,s_t',w_t,\theta_t)\) and \(\bar g_t:=\bar g(w_t,\theta_t)\).
Using \(\delta_t^\theta=\delta_{t-\tau_{\rm mix}}^\theta+(\theta_t-\theta_{t-\tau_{\rm mix}})\),
\begin{align*}
\mathbb{E}\big[{\delta_t^\theta}^\top(g_t-\bar g_t)\big]
&=
\mathbb{E}\big[{\delta_{t-\tau_{\rm mix}}^\theta}^\top(g_t-\bar g_t)\big]
+\mathbb{E}\big[(\theta_t-\theta_{t-\tau_{\rm mix}})^\top(g_t-\bar g_t)\big].
\end{align*}

We first note the crude drift bounds over the last \(\tau_{\rm mix}\) steps:
\begin{align}
\|\theta_t-\theta_{t-\tau_{\rm mix}}\|
&\le \sum_{i=t-\tau_{\rm mix}}^{t-1}\|\theta_{i+1}-\theta_i\|
\le \alpha\sum_{i=t-\tau_{\rm mix}}^{t-1}\|g(s_i,s_i',w_i,\theta_i)\|
\le 2\alpha\tau_{\rm mix}(r_{\max}+3R_\theta), \label{eq: theta_last_tau_onechain}\\
\|w_t-w_{t-\tau_{\rm mix}}\|
&\le \sum_{i=t-\tau_{\rm mix}}^{t-1}\|w_{i+1}-w_i\|
= \beta\sum_{i=t-\tau_{\rm mix}}^{t-1}\|f(s_i,w_i)\|
\le 2\beta\tau_{\rm mix}. \label{eq: w_last_tau_onechain}
\end{align}

Now decompose
\begin{align*}
{\delta_{t-\tau_{\rm mix}}^\theta}^\top(g_t-\bar g_t)
&=
{\delta_{t-\tau_{\rm mix}}^\theta}^\top\big(g_t-g(s_t,s_t',w_{t-\tau_{\rm mix}},\theta_{t-\tau_{\rm mix}})\big) \\
&\quad +{\delta_{t-\tau_{\rm mix}}^\theta}^\top\big(\bar g(w_{t-\tau_{\rm mix}},\theta_{t-\tau_{\rm mix}})-\bar g_t\big) \\
&\quad +{\delta_{t-\tau_{\rm mix}}^\theta}^\top\big(g(s_t,s_t',w_{t-\tau_{\rm mix}},\theta_{t-\tau_{\rm mix}})-\bar g(w_{t-\tau_{\rm mix}},\theta_{t-\tau_{\rm mix}})\big).
\end{align*}

Using Lemma~\ref{l: lipschitz of f and g one chain}, Lemma~\ref{l: delta bound one chain constant}, and \eqref{eq: theta_last_tau_onechain}--\eqref{eq: w_last_tau_onechain},
\begin{align*}
\mathbb{E}\Big[{\delta_{t-\tau_{\rm mix}}^\theta}^\top\big(g_t-g(s_t, s_t',w_{t-\tau_{\rm mix}},\theta_{t-\tau_{\rm mix}})\big)\Big]
&\le \mathbb{E}\Big[\|\delta_{t-\tau_{\rm mix}}^\theta\|\big(3\|\theta_t-\theta_{t-\tau_{\rm mix}}\|+(r_{\max}+2R_\theta)\|w_t-w_{t-\tau_{\rm mix}}\|\big)\Big]\\
&\le 2R_\theta\Big(3\cdot 2\alpha\tau_{\rm mix}(r_{\max}+3R_\theta) + (r_{\max}+2R_\theta)\cdot 2\beta\tau_{\rm mix}\Big)\\
&\le 16\alpha\tau_{\rm mix}(r_{\max}+3R_\theta)^2,
\end{align*}
and the same bound holds with \(g\) replaced by \(\bar g\).

For the remaining bias term, by the definition of \(\tau_{\rm mix}(\beta)\) we have
\(\sup_s\|P_{\tau_{\rm mix}}(\cdot|s)-\mu\|_1\le \beta\), hence
\begin{align*}
\Big\|\mathbb{E}\big[g(s_t,s_t',w_{t-\tau_{\rm mix}},\theta_{t-\tau_{\rm mix}})-\bar g(w_{t-\tau_{\rm mix}},\theta_{t-\tau_{\rm mix}})\mid \mathcal F_{t-\tau_{\rm mix}}\big]\Big\|
\le \beta\cdot \sup_{s,s'}\|g(s,s',w_{t-\tau_{\rm mix}},\theta_{t-\tau_{\rm mix}})\|
\le 2\beta(r_{\max}+3R_\theta),
\end{align*}
so
\[
\mathbb{E}\Big[{\delta_{t-\tau_{\rm mix}}^\theta}^\top\big(g(s_t, s_t',w_{t-\tau_{\rm mix}},\theta_{t-\tau_{\rm mix}})-\bar g(w_{t-\tau_{\rm mix}},\theta_{t-\tau_{\rm mix}})\big)\Big]
\le 2R_\theta\cdot 2\beta(r_{\max}+3R_\theta)
\le 4\alpha(r_{\max}+3R_\theta)^2.
\]

Finally, for the second main term,
\[
\mathbb{E}\big[(\theta_t-\theta_{t-\tau_{\rm mix}})^\top(g_t-\bar g_t)\big]
\le \|\theta_t-\theta_{t-\tau_{\rm mix}}\|\cdot \mathbb{E}\big[\|g_t\|+\|\bar g_t\|\big]
\le 2\alpha\tau_{\rm mix}(r_{\max}+3R_\theta)\cdot 4(r_{\max}+3R_\theta)
=8\alpha\tau_{\rm mix}(r_{\max}+3R_\theta)^2.
\]

Summing the pieces,
\[
\mathbb{E}\big[{\delta_t^\theta}^\top(g_t-\bar g_t)\big]
\le (16+16+4+8)\alpha\tau_{\rm mix}(r_{\max}+3R_\theta)^2
=44\alpha\tau_{\rm mix}(r_{\max}+3R_\theta)^2.
\]
\end{proof}

\subsection{Proof of Lemma~\ref{l: bound on delta w}}
\label{sec: pof bound on delta w}

\begin{proof}[Proof of Lemma~\ref{l: bound on delta w}]
Using non-expansiveness of projection, for \(t\ge 0\),
\begin{align}
\mathbb{E}\|\delta_{t+1}^w\|^2
&\le \mathbb{E}\|\delta_t^w+\beta f(s_t,w_t)\|^2 \nonumber\\
&= \mathbb{E}\|\delta_t^w\|^2 + 2\beta\,\mathbb{E}\big[(\delta_t^w)^\top f(s_t,w_t)\big] + \beta^2\mathbb{E}\|f(s_t,w_t)\|^2. \label{eq: w_expand_onechain}
\end{align}
By Lemma~\ref{l: bound on f and g onechain}, \(\beta^2\mathbb{E}\|f(s_t,w_t)\|^2\le 4\beta^2\).

For \(t\ge \tau_{\rm mix}\), write \(\bar f(w_t)=w^*-w_t=-\delta_t^w\) and decompose
\[
\mathbb{E}\big[(\delta_t^w)^\top f(s_t,w_t)\big]
=
\mathbb{E}\big[(\delta_t^w)^\top (f(s_t,w_t)-\bar f(w_t))\big] + \mathbb{E}\big[(\delta_t^w)^\top \bar f(w_t)\big]
=
\mathbb{E}\big[(\delta_t^w)^\top (f(s_t,w_t)-\bar f(w_t))\big] - \mathbb{E}\|\delta_t^w\|^2.
\]
A standard \(\tau_{\rm mix}\)-step decomposition together with
\(\|w_t-w_{t-\tau_{\rm mix}}\|\le 2\beta\tau_{\rm mix}\),
\(\|f(s_t,w_t)\|\le 2\),
\(\|\delta_t^w\|\le 2\),
and the fact \(\sup_s\|P_{\tau_{\rm mix}}(\cdot|s)-\mu\|_1\le \beta\)
yields
\[
\mathbb{E}\big[(\delta_t^w)^\top (f(s_t,w_t)-\bar f(w_t))\big]\le (16\tau_{\rm mix}+4)\beta.
\]
Therefore, for \(t\ge \tau_{\rm mix}\),
\[
\mathbb{E}\big[(\delta_t^w)^\top f(s_t,w_t)\big]
\le -\mathbb{E}\|\delta_t^w\|^2 + (16\tau_{\rm mix}+4)\beta.
\]
Plugging into \eqref{eq: w_expand_onechain} gives
\[
\mathbb{E}\|\delta_{t+1}^w\|^2
\le (1-2\beta)\mathbb{E}\|\delta_t^w\|^2 + (32\tau_{\rm mix}+12)\beta^2.
\]
Iterating from \(t=\tau_{\rm mix}\) and using \(\mathbb{E}\|\delta_{\tau_{\rm mix}}^w\|^2\le 4\),
\[
\mathbb{E}\|\delta_t^w\|^2
\le (1-2\beta)^{t-\tau_{\rm mix}}\cdot 4 + (32\tau_{\rm mix}+12)\beta^2\sum_{i=0}^{t-\tau_{\rm mix}-1}(1-2\beta)^i
\le 4e^{-2\beta(t-\tau_{\rm mix})} + (16\tau_{\rm mix}+6)\beta.
\]
\end{proof}

%% file: sections/one_decaying.tex
\section{Analysis of Single Markov Chain Algorithm with decaying step-size}
\label{appendix: one chain with decaying}
In this section, we provide the proof to Theorem~\ref{t: one chain}.

We now restate and prove the theorem for the single-chain algorithm with decaying step-size.

\begin{theorem} \label{t: one chain}
Suppose Assumption~\ref{a: markov chain} holds. Let $a,c_0>0$, $\rho_0\in(0,1]$, and choose step-sizes
\[
\alpha_t=\frac{a}{(t+c_0)^\xi},\qquad \beta_t=\rho_0\alpha_t=\frac{\rho_0 a}{(t+c_0)^\xi}.
\]
Let
\begin{align*}
c_{1,\lambda}:=& \eta-\frac{\lambda^2(r_{\rm max}+2R_\theta)}{2}, \\
c_{2,\lambda}:=& \frac{r_{\rm max}+2R_\theta}{\lambda^2},\\
\lambda^2\le & \frac{2\eta}{r_{\rm max}+2R_\theta}.    
\end{align*}
Define
\(
\Delta_T:=(T+c_0)^{1-\xi}-(\tau_{\rm mix}+c_0)^{1-\xi},
\)
and \(G_1(T)\) as in \eqref{eq: def G1}, \(G_2(T)\) as in \eqref{eq:G2}, \(\Gamma_0\) as in \eqref{eq:Gamma0}, \(\Gamma_1\) as in \eqref{eq:Gamma1}. Then:
\begin{enumerate}
\item If $\xi\in(0,1)$ and
\begin{small}
\[
c_0 \;\ge\; \max\left\{\left(\frac{\xi}{a c_{1,\lambda}}\right)^{\!\frac{1}{1-\xi}},\;
\left(\frac{\xi}{\rho_0 a}\right)^{\!\frac{1}{1-\xi}}\right\},
\]
\end{small}
then
\begin{small}
\begin{align*}
& \mathbb{E}\big[\|\delta_T^\theta\|^2\big]\\
\le\;&
4R_\theta^2\cdot
\exp\!\left(-\frac{a c_{1,\lambda}}{1-\xi}\Delta_T\right)
\;+\;
\frac{a\big(G_1(T)+\rho_0 c_{2,\lambda}G_2(T)\big)}{c_{1,\lambda}(T+c_0)^\xi}\\
&\quad
+
\frac{4a c_{2,\lambda}}{1-\xi}
\exp\!\left(\frac{a c_{1,\lambda}}{c_0^\xi}\right)
\Delta_T
\exp\!\left(
-\frac{a}{1-\xi}\rho_{\rm min}\Delta_T
\right).
\end{align*}
\end{small}
where \(\rho_{\rm min} = \min\{c_{1,\lambda},\,2\rho_0\}\)

\item If $\xi=1$, $a\eta<1$, and $\lambda^2\ge \frac{2\eta-4\rho_0}{r_{\rm max}+2R_\theta}$, then
\begin{small}
\begin{align*}
\mathbb{E}\big[\|\delta_T^\theta\|^2\big]
\le\;&
4R_\theta^2\cdot \left(\frac{\tau_{\rm mix}+c_0}{T+c_0}\right)^{a c_{1,\lambda}}
\;+\;
\frac{\Gamma_1+\Gamma_0G_1(T)}{(T+c_0)^{a c_{1,\lambda}}},
\end{align*}
\end{small}
where $\Gamma_0$ and $\Gamma_1$ are defined in Appendix~\ref{appendix: one chain with decaying}.
\end{enumerate}
\end{theorem}

\begin{proof}[Proof of Theorem~\ref{t: one chain}]
As before, the dynamic of $\theta_t$ satisfies
\begin{align*}
\mathbb{E}\|\delta_{t+1}^\theta\|^2
\le\;& \mathbb{E}\|\delta_t^\theta\|^2
+2\alpha_t\mathbb{E}\!\left[{\delta_t^\theta}^\top g(s_t,s_t',w_t,\theta_t)\right]
+\alpha_t^2\mathbb{E}\!\left[\|g(s_t,s_t',w_t,\theta_t)\|^2\right]\\
=\;& \mathbb{E}\|\delta_t^\theta\|^2
+\underbrace{2\alpha_t\mathbb{E}\!\left[{\delta_t^\theta}^\top\big(g(s_t,s_t',w_t,\theta_t)-\bar g(w_t,\theta_t)\big)\right]}_{I_1}\\
&\quad +\underbrace{2\alpha_t\mathbb{E}\!\left[{\delta_t^\theta}^\top\bar g(w_t,\theta_t)\right]}_{I_2}
+\underbrace{\alpha_t^2\mathbb{E}\!\left[\|g(s_t,s_t',w_t,\theta_t)\|^2\right]}_{I_3},
\end{align*}
where
\[
\bar g(w_t,\theta_t)
:=\mathbb{E}_{s\sim\mu}\, g(s,s',w_t,\theta_t)
= \Phi^\top D\big(R+P\Phi\theta_t-\Phi\theta_t\big)-\mu^\top(R+\Phi\theta_t)w_t.
\]

\textbf{Term $I_1$:}
We use the following Markov-noise bound.

\begin{lemma}
\label{l: markov noise one chain decaying}
Suppose $t\ge \tau_{\rm mix}$. Then
\[
\mathbb{E}\!\left[{\delta_t^\theta}^\top\big(g(s_t,s_t',w_t,\theta_t)-\bar g(w_t,\theta_t)\big)\right]
\le \rho_1 (r_{\rm max}+3R_\theta)^2 \sum_{i=t-\tau_{\rm mix}}^{t-1}\alpha_i,
\]
where $\rho_1:=28+8\rho_0$.
\end{lemma}

The proof is deferred to Section~\ref{sec: pof markov noise one chain decaying}. Using Lemma~\ref{l: markov noise one chain decaying},
\[
I_1 \le 2\rho_1\,\alpha_t (r_{\rm max}+3R_\theta)^2 \sum_{i=t-\tau_{\rm mix}}^{t-1}\alpha_i.
\]

\textbf{Term $I_2$:}
By the same argument as Eq.~(\ref{eq: theta_drift_term}) in Appendix~\ref{appendix: one chain constant}, we have
\[
I_2 \le -2\alpha_t c_{1,\lambda}\,\mathbb{E}\|\delta_t^\theta\|^2
+\alpha_t c_{2,\lambda}\,\mathbb{E}\|\delta_t^w\|^2,
\]
where $c_{1,\lambda}=\eta-\lambda^2(r_{\rm max}+2R_\theta)/2$ and $c_{2,\lambda}=(r_{\rm max}+2R_\theta)/\lambda^2$.

\textbf{Term $I_3$:}
By Lemma~\ref{l: bound on f and g onechain},
\[
I_3 \le 4\alpha_t^2 (r_{\rm max}+3R_\theta)^2.
\]

Combining the three terms yields, for $t\ge \tau_{\rm mix}$,
\begin{align*}
\mathbb{E}\|\delta_{t+1}^\theta\|^2
\le\;& (1-2\alpha_t c_{1,\lambda})\,\mathbb{E}\|\delta_t^\theta\|^2
+\alpha_t c_{2,\lambda}\,\mathbb{E}\|\delta_t^w\|^2
+4\alpha_t^2(r_{\rm max}+3R_\theta)^2\\
&\quad +2\rho_1\,\alpha_t (r_{\rm max}+3R_\theta)^2 \sum_{i=t-\tau_{\rm mix}}^{t-1}\alpha_i.
\end{align*}
Using Lemma~\ref{l: sum alpha bound}, $\sum_{i=t-\tau_{\rm mix}}^{t-1}\alpha_i \le 2L_1\alpha_t(\log(1/\alpha_t)+1)$,
we further obtain
\begin{align}
\label{eq: theta recursion decaying}
\mathbb{E}\|\delta_{t+1}^\theta\|^2
\le\;& (1-2\alpha_t c_{1,\lambda})\,\mathbb{E}\|\delta_t^\theta\|^2
+\alpha_t c_{2,\lambda}\,\mathbb{E}\|\delta_t^w\|^2
+G_1(T)\alpha_t^2,
\end{align}
where we used the monotonicity of $\alpha_t$ to bound $\log(1/\alpha_t)\le \log(1/\alpha_T)$ for all $t\le T$, and defined
\begin{align}
\label{eq: def G1}
G_1(T)
:=\Big(4\rho_1 L_1\big(\log(T+c_0)-\log a+1\big)+4\Big)\,(r_{\rm max}+3R_\theta)^2.
\end{align}

For convenience in the iteration, we relax the contraction factor as
\[
(1-2\alpha_t c_{1,\lambda}) \le (1-\alpha_t c_{1,\lambda})
\qquad(\text{since }\alpha_t c_{1,\lambda}\ge 0),
\]
so from \eqref{eq: theta recursion decaying},
\begin{align}
\label{eq: theta recursion decaying relaxed}
\mathbb{E}\|\delta_{t+1}^\theta\|^2
\le (1-\alpha_t c_{1,\lambda})\,\mathbb{E}\|\delta_t^\theta\|^2
+\alpha_t c_{2,\lambda}\,\mathbb{E}\|\delta_t^w\|^2
+G_1(T)\alpha_t^2.
\end{align}

Iterating \eqref{eq: theta recursion decaying relaxed} from $\tau_{\rm mix}$ to $T$ gives
\begin{align}
\label{eq: theta unroll}
\mathbb{E}\|\delta_T^\theta\|^2
\le\;&
\underbrace{\prod_{i=\tau_{\rm mix}}^{T-1}(1-c_{1,\lambda}\alpha_i)}_{I_1'}
\,\mathbb{E}\|\delta_{\tau_{\rm mix}}^\theta\|^2
+\sum_{j=\tau_{\rm mix}}^{T-1}\left(\prod_{i=j+1}^{T-1}(1-c_{1,\lambda}\alpha_i)\right)\alpha_j c_{2,\lambda}\,\mathbb{E}\|\delta_j^w\|^2\nonumber\\
&\quad +G_1(T)\sum_{j=\tau_{\rm mix}}^{T-1}\alpha_j^2\left(\prod_{i=j+1}^{T-1}(1-c_{1,\lambda}\alpha_i)\right).
\end{align}

We first bound $I_1'$. Using $1-x\le e^{-x}$,
\begin{align}
\label{eq: I1prime bound}
I_1'
\le \exp\!\left(-c_{1,\lambda}\sum_{i=\tau_{\rm mix}}^{T-1}\alpha_i\right)
=\exp\!\left(-a c_{1,\lambda}\sum_{i=\tau_{\rm mix}}^{T-1}\frac{1}{(i+c_0)^\xi}\right)
\le \exp\!\left(-a c_{1,\lambda}\int_{\tau_{\rm mix}}^{T}\frac{dx}{(x+c_0)^\xi}\right),
\end{align}
hence
\begin{align}
\label{eq: I1prime closed form}
I_1'
\le
\begin{cases}
\left(\frac{\tau_{\rm mix}+c_0}{T+c_0}\right)^{a c_{1,\lambda}}, & \xi=1,\\[2mm]
\exp\!\left(-\dfrac{a c_{1,\lambda}}{1-\xi}\Big((T+c_0)^{1-\xi}-(\tau_{\rm mix}+c_0)^{1-\xi}\Big)\right), & \xi\in(0,1).
\end{cases}
\end{align}

Next we control $\mathbb{E}\|\delta_j^w\|^2$.

\begin{lemma}
\label{l: bound on delta w decaying}
Let
\begin{equation}
\label{eq:G2}
G_2(T):=64L_1\big(\log(T+c_0)-\log(\rho_0 a)+1\big)+12.    
\end{equation}
Then for all $t\in[\tau_{\rm mix},T]$:
\begin{enumerate}
\item If $\xi\in(0,1)$ and $c_0\ge \left(\frac{\xi}{\rho_0 a}\right)^{\!\frac{1}{1-\xi}}$, then
\[
\mathbb{E}\|\delta_t^w\|^2
\le
4\exp\!\left(-\frac{2\rho_0 a}{1-\xi}\Big((t+c_0)^{1-\xi}-(\tau_{\rm mix}+c_0)^{1-\xi}\Big)\right)
\;+\; G_2(T)\cdot \frac{\rho_0 a}{(t+c_0)^\xi}.
\]
\item If $\xi=1$, then:
\begin{enumerate}
\item If $\rho_0 a\in(0,1/2)$, then
\[
\mathbb{E}\|\delta_t^w\|^2
\le
4\left(\frac{\tau_{\rm mix}+c_0}{t+c_0}\right)^{2\rho_0 a}
\;+\;
G_2(T)\cdot \frac{4\rho_0^2a^2}{(1-2\rho_0 a)(t+c_0)^{2\rho_0 a}}.
\]
\item If $\rho_0 a=1/2$, then
\[
\mathbb{E}\|\delta_t^w\|^2
\le
4\left(\frac{\tau_{\rm mix}+c_0}{t+c_0}\right)^{2\rho_0 a}
\;+\;
G_2(T)\cdot \frac{4\rho_0^2a^2\log(t+c_0)}{t+c_0}.
\]
\item If $\rho_0 a\in(1/2,\infty)$, then
\[
\mathbb{E}\|\delta_t^w\|^2
\le
4\left(\frac{\tau_{\rm mix}+c_0}{t+c_0}\right)^{2\rho_0 a}
\;+\;
G_2(T)\cdot \frac{4e\rho_0^2a^2}{(2\rho_0 a-1)(t+c_0)}.
\]
\end{enumerate}
\end{enumerate}
\end{lemma}

The proof is deferred to Section~\ref{sec: pof bound on delta w decaying}.

We now split into two cases.

\paragraph{Case 1: $\xi=1$.}
Using \eqref{eq: theta unroll} and \eqref{eq: I1prime closed form} (with $\xi=1$),
\begin{align*}
\mathbb{E}\|\delta_T^\theta\|^2
\le\;&
\left(\frac{\tau_{\rm mix}+c_0}{T+c_0}\right)^{a c_{1,\lambda}}
\mathbb{E}\|\delta_{\tau_{\rm mix}}^\theta\|^2
+ \underbrace{\sum_{j=\tau_{\rm mix}}^{T-1}\left(\frac{j+1+c_0}{T+c_0}\right)^{a c_{1,\lambda}}
\frac{a c_{2,\lambda}}{j+c_0}\,\mathbb{E}\|\delta_j^w\|^2}_{I_2'}\\
&\quad + G_1(T)\underbrace{\sum_{j=\tau_{\rm mix}}^{T-1}\frac{a^2}{(j+c_0)^2}
\left(\frac{j+1+c_0}{T+c_0}\right)^{a c_{1,\lambda}}}_{I_3'}.
\end{align*}

\emph{Bound on $I_3'$.}
Since $(j+1+c_0)/(j+c_0)\le 1+1/c_0\le 2$ when $c_0\ge 1$,
\begin{align*}
I_3'
&=
\frac{a^2}{(T+c_0)^{a c_{1,\lambda}}}\sum_{j=\tau_{\rm mix}}^{T-1}
\frac{(j+1+c_0)^{a c_{1,\lambda}}}{(j+c_0)^2}
\le
\frac{4a^2}{(T+c_0)^{a c_{1,\lambda}}}\sum_{j=\tau_{\rm mix}}^{T-1}(j+1+c_0)^{a c_{1,\lambda}-2}\\
&\le
\frac{4a^2}{(1-a c_{1,\lambda})(T+c_0)^{a c_{1,\lambda}}}
\qquad(a c_{1,\lambda}\le a\eta<1).
\end{align*}
Define
\begin{equation}
\label{eq:Gamma0}   
\Gamma_0:=\frac{4a^2}{1-a c_{1,\lambda}}.
\end{equation}
Then $I_3'\le \Gamma_0/(T+c_0)^{a c_{1,\lambda}}$.

\emph{Bound on $I_2'$.}
Apply Lemma~\ref{l: bound on delta w decaying} (case $\xi=1$). If $\rho_0 a\in(0,1/2)$, then
\begin{align*}
I_2'
\le\;&
\sum_{j=\tau_{\rm mix}}^{T-1}
\left(\frac{j+1+c_0}{T+c_0}\right)^{a c_{1,\lambda}}
\frac{a c_{2,\lambda}}{j+c_0}
\left[
4\left(\frac{\tau_{\rm mix}+c_0}{j+c_0}\right)^{2\rho_0 a}
+
G_2(T)\cdot\frac{4\rho_0^2a^2}{(1-2\rho_0 a)(j+c_0)^{2\rho_0 a}}
\right]\\
=\;&
\frac{a c_{2,\lambda}}{(T+c_0)^{a c_{1,\lambda}}}
\left(4(\tau_{\rm mix}+c_0)^{2\rho_0 a}+\frac{4\rho_0^2a^2}{1-2\rho_0 a}G_2(T)\right)
\sum_{j=\tau_{\rm mix}}^{T-1}\frac{(j+1+c_0)^{a c_{1,\lambda}}}{(j+c_0)^{1+2\rho_0 a}}\\
\le\;&
\frac{2^{a c_{1,\lambda}} a c_{2,\lambda}}{(T+c_0)^{a c_{1,\lambda}}}
\left(4(\tau_{\rm mix}+c_0)^{2\rho_0 a}+\frac{4\rho_0^2a^2}{1-2\rho_0 a}G_2(T)\right)
\sum_{j=\tau_{\rm mix}}^{T-1}(j+c_0)^{a c_{1,\lambda}-1-2\rho_0 a}\\
\le\;&
\frac{2^{a c_{1,\lambda}} a c_{2,\lambda}}{(T+c_0)^{a c_{1,\lambda}}}
\left(4(\tau_{\rm mix}+c_0)^{2\rho_0 a}+\frac{4\rho_0^2a^2}{1-2\rho_0 a}G_2(T)\right)
\cdot \frac{1}{2\rho_0 a-a c_{1,\lambda}},
\end{align*}
where the last step uses $a c_{1,\lambda}<2\rho_0 a$, which follows from
$\lambda^2\ge \frac{2\eta-4\rho_0}{r_{\rm max}+2R_\theta}$ (equivalently, $c_{1,\lambda}\le 2\rho_0$).
The cases $\rho_0 a=1/2$ and $\rho_0 a>1/2$ are handled similarly, yielding the same $(T+c_0)^{-a c_{1,\lambda}}$ scaling.
Define
\begin{equation}
\label{eq:Gamma1}    
\Gamma_1:=
\begin{cases}
\displaystyle
\frac{2^{a c_{1,\lambda}} a c_{2,\lambda}}{2\rho_0 a-a c_{1,\lambda}}
\left(4(\tau_{\rm mix}+c_0)^{2\rho_0 a}+\frac{4\rho_0^2a^2}{1-2\rho_0 a}G_2(T)\right), & \rho_0 a\in(0,1/2),\\[3mm]
\displaystyle
\frac{4 a c_{2,\lambda}}{1-a c_{1,\lambda}}
\left(4(\tau_{\rm mix}+c_0)^{2\rho_0 a}+4\rho_0^2a^2\log(T+c_0)\,G_2(T)\right), & \rho_0 a=1/2,\\[3mm]
\displaystyle
\frac{4 a c_{2,\lambda}}{1-a c_{1,\lambda}}
\left(4(\tau_{\rm mix}+c_0)^{2\rho_0 a}+\frac{4e\rho_0^2a^2}{2\rho_0 a-1}G_2(T)\right), & \rho_0 a\in(1/2,\infty).
\end{cases}
\end{equation}
Then $I_2'\le \Gamma_1/(T+c_0)^{a c_{1,\lambda}}$.

Putting the bounds together and using $\mathbb{E}\|\delta_{\tau_{\rm mix}}^\theta\|^2\le 4R_\theta^2$, we obtain
\[
\mathbb{E}\|\delta_T^\theta\|^2
\le
4R_\theta^2\left(\frac{\tau_{\rm mix}+c_0}{T+c_0}\right)^{a c_{1,\lambda}}
+\frac{\Gamma_1+\Gamma_0G_1(T)}{(T+c_0)^{a c_{1,\lambda}}}.
\]

\paragraph{Case 2: $\xi\in(0,1)$.}

Recall from Lemma~\ref{l: bound on delta w decaying} (case $\xi\in(0,1)$) that for all $t\in[\tau_{\rm mix},T]$,
\begin{equation}
\label{eq: w bound split}
\mathbb{E}\|\delta_t^w\|^2
\le
4\exp\!\left(-A\Big((t+c_0)^{1-\xi}-(\tau_{\rm mix}+c_0)^{1-\xi}\Big)\right)
\;+\;\rho_0 G_2(T)\,\alpha_t,
\qquad
A:=\frac{2\rho_0 a}{1-\xi}.
\end{equation}

Define
\[
\Delta_T:=(T+c_0)^{1-\xi}-(\tau_{\rm mix}+c_0)^{1-\xi},
\qquad
S_T:=\sum_{j=\tau_{\rm mix}}^{T-1}\alpha_j^2\Big(\prod_{i=j+1}^{T-1}(1-c_{1,\lambda}\alpha_i)\Big).
\]
We also define the exponential-part sum
\begin{equation}
\label{eq: STexp def}
S_T^{\exp}
:=
\sum_{j=\tau_{\rm mix}}^{T-1}
\alpha_j
\Big(\prod_{i=j+1}^{T-1}(1-c_{1,\lambda}\alpha_i)\Big)
\exp\!\left(-A\Big((j+c_0)^{1-\xi}-(\tau_{\rm mix}+c_0)^{1-\xi}\Big)\right).
\end{equation}

\begin{lemma}
\label{l: super linear}
Let $\xi\in(0,1)$. Then
\begin{equation}
\label{eq: STexp bound}
S_T^{\exp}
\le
\exp\!\Big(\tfrac{a c_{1,\lambda}}{c_0^\xi}\Big)\cdot
\frac{a}{1-\xi}\,\Delta_T\,
\exp\!\left(
-\frac{a}{1-\xi}\min\{c_{1,\lambda},\,2\rho_0\}\,\Delta_T
\right).
\end{equation}
\end{lemma}
The proof is deferred to Section~\ref{sec: pof super linear}.

Now plug \eqref{eq: w bound split} into the second term of \eqref{eq: theta unroll}. Using
$c_{2,\lambda}\alpha_j\,\rho_0 G_2(T)\alpha_j = \rho_0 c_{2,\lambda}G_2(T)\alpha_j^2$,
we obtain
\begin{align*}
&\sum_{j=\tau_{\rm mix}}^{T-1}\Big(\prod_{i=j+1}^{T-1}(1-c_{1,\lambda}\alpha_i)\Big)\alpha_j c_{2,\lambda}\,\mathbb{E}\|\delta_j^w\|^2\\
\le\;&
4c_{2,\lambda}\,S_T^{\exp}
\;+\;\rho_0 c_{2,\lambda}G_2(T)\,S_T.
\end{align*}
Therefore, by \eqref{eq: theta unroll} and \eqref{eq: I1prime closed form} (with $\xi\in(0,1)$),
\begin{align*}
\mathbb{E}\|\delta_T^\theta\|^2
\le\;&
\exp\!\left(-\frac{a c_{1,\lambda}}{1-\xi}\Delta_T\right)\mathbb{E}\|\delta_{\tau_{\rm mix}}^\theta\|^2
\;+\;\big(G_1(T)+\rho_0 c_{2,\lambda}G_2(T)\big)\,S_T
\;+\;4c_{2,\lambda}\,S_T^{\exp}.
\end{align*}

By Lemma~\ref{l: dynamic system with decaying stepsize} and the condition
$c_0\ge \left(\frac{\xi}{a c_{1,\lambda}}\right)^{\!\frac{1}{1-\xi}}$, we have
\[
S_T \le \frac{a}{c_{1,\lambda}(T+c_0)^\xi}.
\]
Applying Lemma~\ref{l: super linear} (i.e., \eqref{eq: STexp bound}) and using
$\mathbb{E}\|\delta_{\tau_{\rm mix}}^\theta\|^2\le 4R_\theta^2$ yields the following bound for Case~2:
\begin{align*}
\mathbb{E}\|\delta_T^\theta\|^2
\le\;&
4R_\theta^2\exp\!\left(-\frac{a c_{1,\lambda}}{1-\xi}\Delta_T\right)
+\frac{a\big(G_1(T)+\rho_0 c_{2,\lambda}G_2(T)\big)}{c_{1,\lambda}(T+c_0)^\xi}\\
&\quad
+4c_{2,\lambda}\exp\!\Big(\tfrac{a c_{1,\lambda}}{c_0^\xi}\Big)\cdot
\frac{a}{1-\xi}\,\Delta_T\,
\exp\!\left(
-\frac{a}{1-\xi}\min\{c_{1,\lambda},\,2\rho_0\}\,\Delta_T
\right).
\end{align*}
\end{proof}

\subsection{Proof of Lemma~\ref{l: markov noise one chain decaying}}
\label{sec: pof markov noise one chain decaying}

\begin{proof}[Proof of Lemma~\ref{l: markov noise one chain decaying}]
The iteration of \(w\) suggests that
\begin{equation}
\label{eq: w in last tau one chain decaying}
\begin{aligned}
\|w_t-w_{t-\tau_{\rm mix}}\|
&\le \sum_{i=t-\tau_{\rm mix}}^{t-1}\|w_{i+1}-w_i\|
\le \sum_{i=t-\tau_{\rm mix}}^{t-1}\beta_i\|f(s_i,w_i)\| \\
&\le 2\sum_{i=t-\tau_{\rm mix}}^{t-1}\beta_i.
\end{aligned}
\end{equation}
For simplicity, denote $g_t(w_t,\theta_t):=g(s_t,s_t',w_t,\theta_t)$.
First, by telescoping and Lemma~\ref{l: bound on f and g onechain},
\begin{equation}
\label{eq: theta in last tau one chain decaying}
\begin{aligned}
\|\theta_t-\theta_{t-\tau_{\rm mix}}\|
&\le \sum_{i=t-\tau_{\rm mix}}^{t-1}\|\theta_{i+1}-\theta_i\|
\le \sum_{i=t-\tau_{\rm mix}}^{t-1}\alpha_i\|g(s_i,s_i',w_i,\theta_i)\|\\
&\le (2r_{\rm max}+6R_\theta)\sum_{i=t-\tau_{\rm mix}}^{t-1}\alpha_i
=2(r_{\rm max}+3R_\theta)\sum_{i=t-\tau_{\rm mix}}^{t-1}\alpha_i.
\end{aligned}
\end{equation}

Decompose
\begin{align*}
&\mathbb{E}\!\left[{\delta_t^\theta}^\top\big(g_t(w_t,\theta_t)-\bar g(w_t,\theta_t)\big)\right]\\
=\;&
\mathbb{E}\!\left[{\delta_{t-\tau_{\rm mix}}^\theta}^\top\big(g_t(w_t,\theta_t)-g_t(w_{t-\tau_{\rm mix}},\theta_{t-\tau_{\rm mix}})\big)\right]
+\mathbb{E}\!\left[{\delta_{t-\tau_{\rm mix}}^\theta}^\top\big(\bar g(w_{t-\tau_{\rm mix}},\theta_{t-\tau_{\rm mix}})-\bar g(w_t,\theta_t)\big)\right]\\
&\quad +\mathbb{E}\!\left[{\delta_{t-\tau_{\rm mix}}^\theta}^\top\big(g_t(w_{t-\tau_{\rm mix}},\theta_{t-\tau_{\rm mix}})-\bar g(w_{t-\tau_{\rm mix}},\theta_{t-\tau_{\rm mix}})\big)\right]
+\mathbb{E}\!\left[(\theta_t-\theta_{t-\tau_{\rm mix}})^\top\big(g_t(w_t,\theta_t)-\bar g(w_t,\theta_t)\big)\right]\\
=:\;& I_1+I_2+I_3+I_4.
\end{align*}

\textbf{Terms $I_1$ and $I_2$.}
By Lemma~\ref{l: lipschitz of f and g one chain}, $\|g(\cdot,\cdot,w,\theta_1)-g(\cdot,\cdot,w,\theta_2)\|\le 2\|\theta_1-\theta_2\|$
and $\|g(\cdot,\cdot,w_1,\theta)-g(\cdot,\cdot,w_2,\theta)\|\le (r_{\rm max}+2R_\theta)\|w_1-w_2\|$.
Also, by \eqref{eq: w in last tau one chain decaying} and $\beta_i=\rho_0\alpha_i$,
\[
\|w_t-w_{t-\tau_{\rm mix}}\|
\le 2\sum_{i=t-\tau_{\rm mix}}^{t-1}\beta_i
=2\rho_0\sum_{i=t-\tau_{\rm mix}}^{t-1}\alpha_i.
\]
Thus, using $\|\delta_{t-\tau_{\rm mix}}^\theta\|\le 2R_\theta$,
\begin{align*}
I_1
&\le \mathbb{E}\!\left[\|\delta_{t-\tau_{\rm mix}}^\theta\|\cdot 2\|\theta_t-\theta_{t-\tau_{\rm mix}}\|\right]
+\mathbb{E}\!\left[\|\delta_{t-\tau_{\rm mix}}^\theta\|\cdot (r_{\rm max}+2R_\theta)\|w_t-w_{t-\tau_{\rm mix}}\|\right]\\
&\le 2R_\theta\cdot 4(r_{\rm max}+3R_\theta)\sum_{i=t-\tau_{\rm mix}}^{t-1}\alpha_i
+2R_\theta\cdot (r_{\rm max}+2R_\theta)\cdot 2\rho_0\sum_{i=t-\tau_{\rm mix}}^{t-1}\alpha_i\\
&\le (8+4\rho_0)R_\theta(r_{\rm max}+3R_\theta)\sum_{i=t-\tau_{\rm mix}}^{t-1}\alpha_i.
\end{align*}
The same bound applies to $I_2$ since $\bar g$ is the expectation of $g$.

\textbf{Term $I_3$.}
Conditioning on $\mathcal{F}_{t-\tau_{\rm mix}}$ and using the mixing bound
$\|P_{\tau_{\rm mix}}(\cdot\mid s)-\mu\|_1\le C\beta^{\tau_{\rm mix}}\le \beta_T\le \alpha_t$ (for $t\le T$),
\begin{align*}
I_3
&\le
\mathbb{E}\!\left[\|\delta_{t-\tau_{\rm mix}}^\theta\|\cdot \|P_{\tau_{\rm mix}}(\cdot\mid s_{t-\tau_{\rm mix}})-\mu\|_1
\cdot \sup_{s,s'}\|g(s,s',w_{t-\tau_{\rm mix}},\theta_{t-\tau_{\rm mix}})\|\right]\\
&\le 2R_\theta\cdot \alpha_t \cdot (2r_{\rm max}+6R_\theta)
=4\alpha_t R_\theta(r_{\rm max}+3R_\theta).
\end{align*}

\textbf{Term $I_4$.}
Using \eqref{eq: theta in last tau one chain decaying} and Lemma~\ref{l: bound on f and g onechain},
\begin{align*}
I_4
&\le \mathbb{E}\!\left[\|\theta_t-\theta_{t-\tau_{\rm mix}}\|\cdot \big(\|g_t(w_t,\theta_t)\|+\|\bar g(w_t,\theta_t)\|\big)\right]\\
&\le 2(r_{\rm max}+3R_\theta)\sum_{i=t-\tau_{\rm mix}}^{t-1}\alpha_i \cdot 2\cdot (2r_{\rm max}+6R_\theta)\\
&=8(r_{\rm max}+3R_\theta)^2\sum_{i=t-\tau_{\rm mix}}^{t-1}\alpha_i.
\end{align*}

Summing the four terms and using $R_\theta\le r_{\rm max}+3R_\theta$ and $\alpha_t\le \sum_{i=t-\tau_{\rm mix}}^{t-1}\alpha_i$,
\begin{align*}
\mathbb{E}\!\left[{\delta_t^\theta}^\top\big(g_t(w_t,\theta_t)-\bar g(w_t,\theta_t)\big)\right]
&\le (16+8\rho_0)R_\theta(r_{\rm max}+3R_\theta)\sum_{i=t-\tau_{\rm mix}}^{t-1}\alpha_i
+4\alpha_t R_\theta(r_{\rm max}+3R_\theta)\\
&\quad +8(r_{\rm max}+3R_\theta)^2\sum_{i=t-\tau_{\rm mix}}^{t-1}\alpha_i\\
&\le (28+8\rho_0)(r_{\rm max}+3R_\theta)^2\sum_{i=t-\tau_{\rm mix}}^{t-1}\alpha_i,
\end{align*}
which is the desired result with $\rho_1=28+8\rho_0$.
\end{proof}

\subsection{Proof of Lemma~\ref{l: bound on delta w decaying}}
\label{sec: pof bound on delta w decaying}

\begin{proof}[Proof of Lemma~\ref{l: bound on delta w decaying}]
The iterates of \(w\) suggests that
\begin{align*}
\mathbb{E}\|\delta_{t+1}^w\|^2
\le \mathbb{E}\|\delta_t^w\|^2
+2\beta_t \mathbb{E}\!\left[(\delta_t^w)^\top f(s_t,w_t)\right]
+\beta_t^2\mathbb{E}\|f(s_t,w_t)\|^2.
\end{align*}
By Lemma~\ref{l: bound on f and g onechain}, $\mathbb{E}\|f(s_t,w_t)\|^2\le 4$, hence the third term is $\le 4\beta_t^2$.

For the second term, decompose as in the constant-step proof (Appendix~\ref{appendix: one chain constant}):
\[
\mathbb{E}\!\left[(\delta_t^w)^\top f(s_t,w_t)\right]
= I_1+I_2+I_3+I_4+I_5,
\]
where the five terms are exactly those in Section~\ref{sec: pof bound on delta w} with $\beta$ replaced by $\beta_t$ (and sums over the last $\tau_{\rm mix}$ indices).
Using the same arguments but keeping the time-varying step-size, we obtain:
\begin{align*}
I_1 &\le 4\sum_{i=t-\tau_{\rm mix}}^{t-1}\beta_i,\\
I_2 &\le 4\beta_t,\\
I_3 &\le 4\sum_{i=t-\tau_{\rm mix}}^{t-1}\beta_i,\\
I_4 &\le 8\sum_{i=t-\tau_{\rm mix}}^{t-1}\beta_i,\\
I_5 &= -\mathbb{E}\|\delta_t^w\|^2,
\end{align*}
and therefore
\[
\mathbb{E}\!\left[(\delta_t^w)^\top f(s_t,w_t)\right]
\le -\mathbb{E}\|\delta_t^w\|^2 + 16\sum_{i=t-\tau_{\rm mix}}^{t-1}\beta_i + 4\beta_t.
\]
Plugging back yields
\begin{align}
\label{eq: w recursion decaying core}
\mathbb{E}\|\delta_{t+1}^w\|^2
\le (1-2\beta_t)\mathbb{E}\|\delta_t^w\|^2
+32\beta_t\hat\beta_t + 12\beta_t^2,
\qquad
\hat\beta_t:=\sum_{i=t-\tau_{\rm mix}}^{t-1}\beta_i.
\end{align}

As in Lemma~\ref{l: sum alpha bound}, one can show that
\[
\hat\beta_t \le 2L_1\,\beta_t\big(\log(1/\beta_t)+1\big),
\qquad t\le T,
\]
where $L_1=\max\left\{1,\frac{\log C+1}{\log(1/\beta)}\right\}$.
Substituting this bound into \eqref{eq: w recursion decaying core} and using $\log(1/\beta_t)\le \log(1/\beta_T)=\log(T+c_0)-\log(\rho_0 a)$ for $t\le T$ gives
\[
\mathbb{E}\|\delta_{t+1}^w\|^2
\le (1-2\beta_t)\mathbb{E}\|\delta_t^w\|^2 + G_2(T)\,\beta_t^2,
\]
with $G_2(T)$ as defined in Lemma~\ref{l: bound on delta w decaying}.
Unrolling the recursion yields
\[
\mathbb{E}\|\delta_t^w\|^2
\le \underbrace{\prod_{i=\tau_{\rm mix}}^{t-1}(1-2\beta_i)}_{J_1}\,\mathbb{E}\|\delta_{\tau_{\rm mix}}^w\|^2
+G_2(T)\underbrace{\sum_{i=\tau_{\rm mix}}^{t-1}\beta_i^2\prod_{j=i+1}^{t-1}(1-2\beta_j)}_{J_2}.
\]
The term $J_1$ is bounded by $J_1\le \exp(-2\sum_{i=\tau_{\rm mix}}^{t-1}\beta_i)$, which gives the stated exponential bounds for $\xi\in(0,1)$ and the stated power bounds for $\xi=1$.
The term $J_2$ is bounded by the standard summation estimates, yielding the three sub-cases for $\xi=1$ and, for $\xi\in(0,1)$ under
$c_0\ge \left(\frac{\xi}{\rho_0 a}\right)^{\frac{1}{1-\xi}}$,
the bound $J_2\le \frac{\rho_0 a}{(t+c_0)^\xi}$ via Lemma~\ref{l: dynamic system with decaying stepsize}.
Finally, Lemma~\ref{l: delta bound one chain constant} gives $\mathbb{E}\|\delta_{\tau_{\rm mix}}^w\|^2\le 4$, completing the proof.
\end{proof}

\subsection{Proof of Lemma \ref{l: super linear}}
\label{sec: pof super linear}

\begin{proof}[Proof of Lemma~\ref{l: super linear}]
Let $u(t):=(t+c_0)^{1-\xi}$ and $\Delta_T=u(T)-u(\tau_{\rm mix})$.
For any $j\in[\tau_{\rm mix},T-1]$, using $1-x\le e^{-x}$ and $\alpha_i\ge 0$,
\[
\prod_{i=j+1}^{T-1}(1-c_{1,\lambda}\alpha_i)
\le
\exp\!\left(-c_{1,\lambda}\sum_{i=j+1}^{T-1}\alpha_i\right).
\]
Since $x\mapsto (x+c_0)^{-\xi}$ is decreasing,
\[
\sum_{i=j+1}^{T-1}\alpha_i
=a\sum_{i=j+1}^{T-1}\frac{1}{(i+c_0)^\xi}
\ge
a\int_{j+1}^{T}\frac{dx}{(x+c_0)^\xi}
=
\frac{a}{1-\xi}\big(u(T)-u(j+1)\big).
\]
Moreover,
\[
u(j+1)=u(j)+\big(u(j+1)-u(j)\big)
\le u(j)+(1-\xi)(j+c_0)^{-\xi}
= u(j)+\frac{1-\xi}{a}\alpha_j,
\]
hence
\[
u(T)-u(j+1)\ge u(T)-u(j)-\frac{1-\xi}{a}\alpha_j.
\]
Combining the above gives
\[
\prod_{i=j+1}^{T-1}(1-c_{1,\lambda}\alpha_i)
\le
\exp\!\left(-\frac{a c_{1,\lambda}}{1-\xi}\big(u(T)-u(j)\big)+c_{1,\lambda}\alpha_j\right)
\le
\exp\!\Big(\tfrac{a c_{1,\lambda}}{c_0^\xi}\Big)\cdot
\exp\!\left(-\frac{a c_{1,\lambda}}{1-\xi}\big(u(T)-u(j)\big)\right),
\]
where we used $\alpha_j\le \alpha_0=a/c_0^\xi$.

Therefore, each summand in \eqref{eq: STexp def} satisfies
\begin{align*}
&\alpha_j
\Big(\prod_{i=j+1}^{T-1}(1-c_{1,\lambda}\alpha_i)\Big)
\exp\!\left(-A\big(u(j)-u(\tau_{\rm mix})\big)\right)\\
\le\;&
\exp\!\Big(\tfrac{a c_{1,\lambda}}{c_0^\xi}\Big)\cdot
\alpha_j
\exp\!\left(
-\frac{a c_{1,\lambda}}{1-\xi}\big(u(T)-u(j)\big)
-\frac{2\rho_0 a}{1-\xi}\big(u(j)-u(\tau_{\rm mix})\big)
\right).
\end{align*}
Since $u(j)\in[u(\tau_{\rm mix}),u(T)]$, we have for all such $j$,
\[
c_{1,\lambda}\big(u(T)-u(j)\big)+2\rho_0\big(u(j)-u(\tau_{\rm mix})\big)
\ge
\min\{c_{1,\lambda},2\rho_0\}\cdot\big(u(T)-u(\tau_{\rm mix})\big)
=\min\{c_{1,\lambda},2\rho_0\}\Delta_T,
\]
and hence the exponential factor is at most
$\exp\!\left(-\frac{a}{1-\xi}\min\{c_{1,\lambda},2\rho_0\}\Delta_T\right)$.
Summing and using
\[
\sum_{j=\tau_{\rm mix}}^{T-1}\alpha_j
\le a\int_{\tau_{\rm mix}}^{T}\frac{dx}{(x+c_0)^\xi}
=\frac{a}{1-\xi}\Delta_T
\]
yields \eqref{eq: STexp bound}.
\end{proof}

%% file: sections/numerical_result.tex
\section{Numerical Results\label{sec:numerical}}

In this section, we provide numerical results on our proposed algorithms. We consider tasks from OpenAI Gym \cite{towers2024gymnasium}, MO-Gymnasium \cite{felten_toolkit_2023} and Gridworlds. All of the tasks share the following settings: 

\begin{itemize}
    \item Policy: the policy \(\pi\) is learned by using the $Q$-learning algorithm. 
    \item Continuous task: some tasks are episodic. To make it a continuous task, the agent will proceed to the starting point after reaching any terminal state with a reward of \(0\). 
    \item Ergodic MDP: to ensure the induced Markov chain under this policy is ergodic, we modify the transition matrix as follows: for each state \(s\), we examine the \(s\)-th row of the transition matrix \(P\). If the original transition probabilities under the learned policy contains entries that are \(0\) (e.g., \(P(\cdot|s)=[0,0,1,0,0]\)), we redistribute a small portion of the probability mass to all previously unreachable states. Specifically, we assign a small probability \(\epsilon\) equally among the zero-probability entries, and reduce the original non-zero entries accordingly to ensure the row still sums to 1. For instance, when \(\epsilon=0.1\), the row above becomes \([0.025,0.025,0.9,0.025,0.025]\). This adjustment is applied to all rows of \(P\), ensuring that every state has a non-zero probability of transitioning to every other state, thus enforcing ergodicity. The choice of \(\epsilon\) for different tasks can be found in Table \ref{table: parameters}. 
    \item Reward function: the reward function \(R\) is a vector whose \(s\)-th row \(R(s)\) is defined by the deterministic one-step reward of performing the policy \(\pi\) in state \(s\). 
    \item Stationary distribution: the stationary distribution \(\mu\) is obtained by solving \(\mu^T = \mu^T P\).
    \item Averaged reward: the average reward function \(g\) is defined as \(g = \mu^T R\).
    \item Feature matrix: the feature matrix \(\Phi\) is defined to be a \(|S| \times d\) matrix. We first generate a matrix \(\tilde{\Phi} \in \mathbb{R}^{|S| \times (d-2)}\), where each element is drawn from the Bernoulli distribution with success probability \(p=0.5\). Then, we construct \(\Phi\) by stacking the all-ones vector \(e\) and the true value function \(W^*\) as columns into the matrix \(\tilde{\Phi}\), i.e., \(\Phi = [\tilde{\Phi}, e, W^*]\). The process is repeated until the feature matrix has full column rank. We further normalize the features to ensure \(\Vert \phi(s) \Vert \le 1\) for all \(s \in S\). 
\end{itemize}

We plot the value function error $\|W_t-W^*\|$ for four algorithms. Two of them are proposed in this paper, namely the Double-Chain and Single-Chain algorithms. As baselines, we include representative prior methods from Table~\ref{tb:comparison}. We note that \cite{tsitsiklis1999average}, \cite{blaser2024almost}, \cite{haque2024stochastic}, \cite{haque2024stochastic}, and \cite{chen2025non} use essentially the same update rule, while \cite{li2024stochastic} incorporates variance reduction and therefore converges much more slowly. 

Here, for our algorithms, we compute \(W^\star\) as the unique solution to
\[
W^\star + g e = P W^\star + R,
\qquad
\mu^\top W^\star = 0.
\]
For the other algorithms, we measure the error in value-function space modulo additive constants, since in the average-reward setting the relative value function is only defined up to a constant shift. Specifically, rather than comparing \(W_t\) with a single representative \(W^\star\), we measure the Euclidean distance from \(W_t\) to the affine space
\[
\mathcal W^\star := \{W^\star + c e : c \in \mathbb R\},
\]
where \(e\) denotes the all-one vector. Equivalently, we remove from \(W_t - W^\star\) its component in the constant direction \(e\), and retain only the orthogonal component. The resulting error metric is
\[
\mathrm{dist}(W_t,\mathcal W^\star)
=
\left\|
\left(I-\frac{ee^\top}{e^\top e}\right)(W_t-W^\star)
\right\|_2.
\]

Each curve is averaged over three independent runs. The step-size schedule $\alpha_t$ and the total number of iterations $T$ are reported in Table~\ref{table: parameters}.

\begin{table}
\caption{Parameters in different tasks}
\label{table: parameters}
\centering
\begin{tabular}{lllll}
    \toprule
    Task     & \(d\)     & \(\alpha_t\) & \(T\) & \(\epsilon\)\\
    \midrule
    Random Walk (\(50\))  & \(5\)  & \(150/(t+1000)\) & \(150000\) & N/A   \\
    Random Walk (\(100\))  & \(20\)  & \(150/(t+1000)\) & \(150000\) & N/A   \\
    Random Walk (\(1000\)) & \(100\) & \(150/(t+1000)\) & \(150000\) & N/A    \\
    Frozen Lake & \(10\)& \(150/(t+1000)\) & \(150000\) & \(0.1\)    \\
    Cliff Walking & \(20\) & \(150/(t+1000)\)& \(150000\)& \(0.1\)    \\
    Taxi & \(100\) & \(150/(t+1000)\) & \(150000\) & \(0.5\)    \\
    Grid World (5x5) & \(10\) & \(200/(t+1000)\)& \(200000\)& \(0.2\)    \\
    Grid World (10x10) & \(40\) & \(300/(t+1000)\) & \(300000\) & \(0.2\)    \\
    Grid World (2x11) & \(10\) & \(300/(t+1000)\) & \(300000\) & \(0.2\)    \\
    Deep sea & \(20\) & \(200/(t+1000)\) & \(200000\) & \(0.2\)    \\
    Deep sea (concave) & \(20\) & \(200/(t+1000)\) & \(200000\) & \(0.2\)    \\
    Resource Gathering & \(50\)& \(500/(t+1000)\)& \(500000\)& \(0.5\)\\
    Fruit Tree (depth = 5) & \(50\)& \(150/(t+1000)\)& \(150000\)& \(0.2\)    \\
    Fruit Tree (depth = 6) & \(50\)& \(150/(t+1000)\)& \(150000\)& \(0.2\)    \\
    Fruit Tree (depth = 7) & \(50\)& \(300/(t+1000)\)& \(300000\)& \(0.2\)    \\
    \bottomrule
\end{tabular}
\end{table}

\begin{table}
\caption{Distance in value space in different tasks (mean \(\pm\) std)}
\label{table: quality-vstar-compare-all}
\centering
\begin{scriptsize}
\begin{sc}
\resizebox{\linewidth}{!}{%
\begin{tabular}{lcccc}
    \toprule
    Task & DoubleChain & SingleChain & \makecell[c]{\cite{haque2024stochastic} \\ \& \cite{zhang2021finite} \\ \& \cite{chen2025non}} & \cite{kim2025implicit} \\
    \midrule
    Random Walk (50)   & \(0.12 \pm 0.02\) & \(\boldsymbol{0.04 \pm 0.00}\) & \(\boldsymbol{0.04 \pm 0.00}\) & \(\boldsymbol{0.04 \pm 0.00}\) \\
    Random Walk (100)  & \(0.19 \pm 0.02\) & \(0.07 \pm 0.01\) & \(\boldsymbol{0.06 \pm 0.01}\) & \(\boldsymbol{0.06 \pm 0.01}\) \\
    Random Walk (1000) & \(2.94 \pm 0.13\) & \(\boldsymbol{2.82 \pm 0.03}\) & \(2.84 \pm 0.03\) & \(2.92 \pm 0.03\) \\
    Frozen Lake        & \(\boldsymbol{0.44 \pm 0.00}\) & \(\boldsymbol{0.44 \pm 0.01}\) & \(\boldsymbol{0.44 \pm 0.01}\) & \(\boldsymbol{0.44 \pm 0.01}\) \\
    Cliff Walking      & \(\boldsymbol{0.62 \pm 0.12}\) & \(0.85 \pm 0.02\) & \(0.88 \pm 0.01\) & \(0.88 \pm 0.01\) \\
    Taxi               & \(15.88 \pm 2.96\) & \(\boldsymbol{5.90 \pm 0.38}\) & \(9.05 \pm 0.45\) & \(8.96 \pm 0.42\) \\
    Grid World (5x5)   & \(0.04 \pm 0.01\) & \(\boldsymbol{0.01 \pm 0.00}\) & \(0.12 \pm 0.00\) & \(0.12 \pm 0.00\) \\
    Grid World (10x10) & \(\boldsymbol{0.44 \pm 0.03}\) & \(0.49 \pm 0.01\) & \(0.45 \pm 0.01\) & \(0.48 \pm 0.01\) \\
    Grid World (2x11)  & \(0.04 \pm 0.00\) & \(\boldsymbol{0.01 \pm 0.00}\) & \(0.07 \pm 0.00\) & \(0.08 \pm 0.00\) \\
    Deep Sea           & \(0.22 \pm 0.06\) & \(\boldsymbol{0.09 \pm 0.02}\) & \(0.28 \pm 0.01\) & \(0.29 \pm 0.01\) \\
    Deep Sea (concave) & \(0.15 \pm 0.03\) & \(\boldsymbol{0.05 \pm 0.01}\) & \(0.15 \pm 0.01\) & \(0.15 \pm 0.01\) \\
    Resource Gathering & \(\boldsymbol{0.00 \pm 0.00}\) & \(\boldsymbol{0.00 \pm 0.00}\) & \(\boldsymbol{0.00 \pm 0.00}\) & \(\boldsymbol{0.00 \pm 0.00}\) \\
    Fruit Tree (depth = 5) & \(0.87 \pm 0.14\) & \(\boldsymbol{0.49 \pm 0.04}\) & \(0.51 \pm 0.02\) & \(0.52 \pm 0.02\) \\
    Fruit Tree (depth = 6) & \(1.22 \pm 0.10\) & \(\boldsymbol{0.51 \pm 0.03}\) & \(0.85 \pm 0.01\) & \(0.85 \pm 0.01\) \\
    Fruit Tree (depth = 7) & \(0.86 \pm 0.17\) & \(\boldsymbol{0.38 \pm 0.05}\) & \(0.46 \pm 0.06\) & \(0.45 \pm 0.05\) \\
    \bottomrule
\end{tabular}%
}
\end{sc}
\end{scriptsize}
\end{table}

\begin{figure}[h]
     \centering
     \begin{subfigure}{0.32\textwidth}
         \centering
         \includegraphics[scale=0.2]{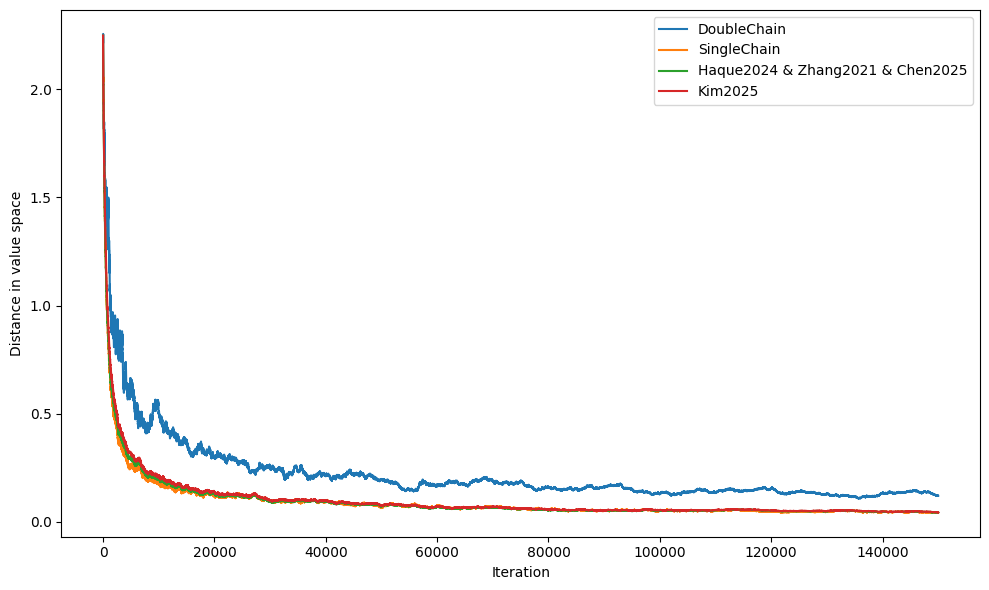}
         \caption{Random Walk (\(50\) states)}
     \end{subfigure}
     \begin{subfigure}{0.32\textwidth}
         \centering
         \includegraphics[scale=0.2]{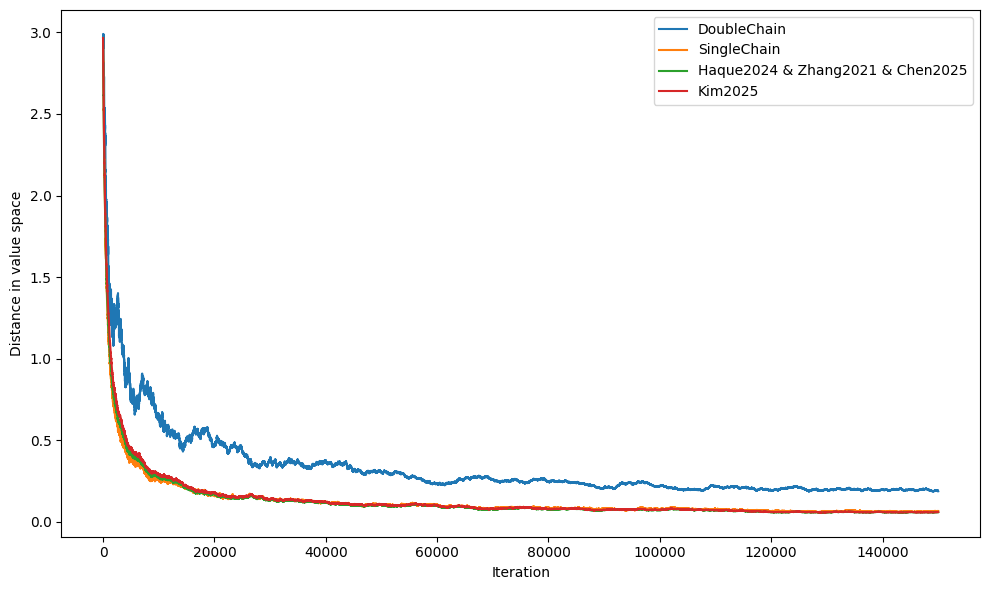}
         \caption{Random Walk (\(100\) states)}
     \end{subfigure}
     \begin{subfigure}{0.32\textwidth}
         \centering
         \includegraphics[scale=0.2]{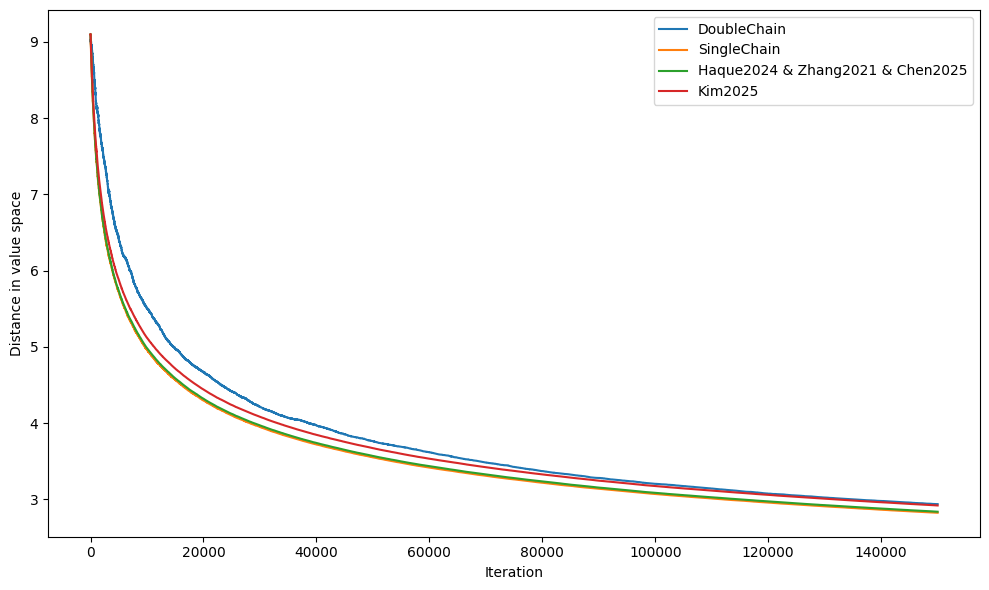}
         \caption{Random Walk (\(1000\) states)}
     \end{subfigure}
     \begin{subfigure}{0.32\textwidth}
         \centering
         \includegraphics[scale=0.2]{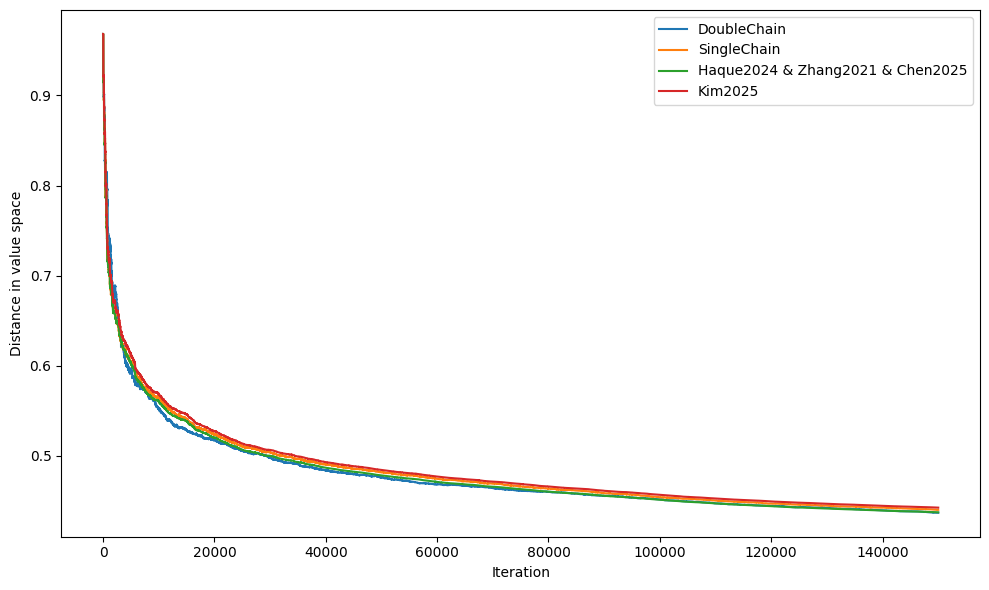}
         \caption{Frozen Lake}
     \end{subfigure}
     \begin{subfigure}{0.32\textwidth}
         \centering
         \includegraphics[scale=0.2]{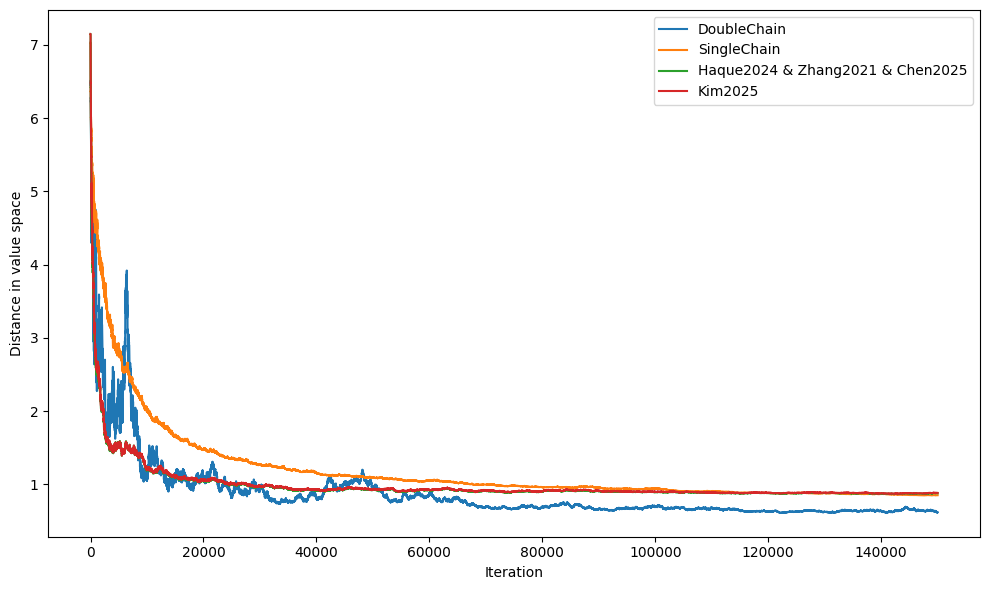}
         \caption{Cliff Walking}
     \end{subfigure}
     \begin{subfigure}{0.32\textwidth}
         \centering
         \includegraphics[scale=0.2]{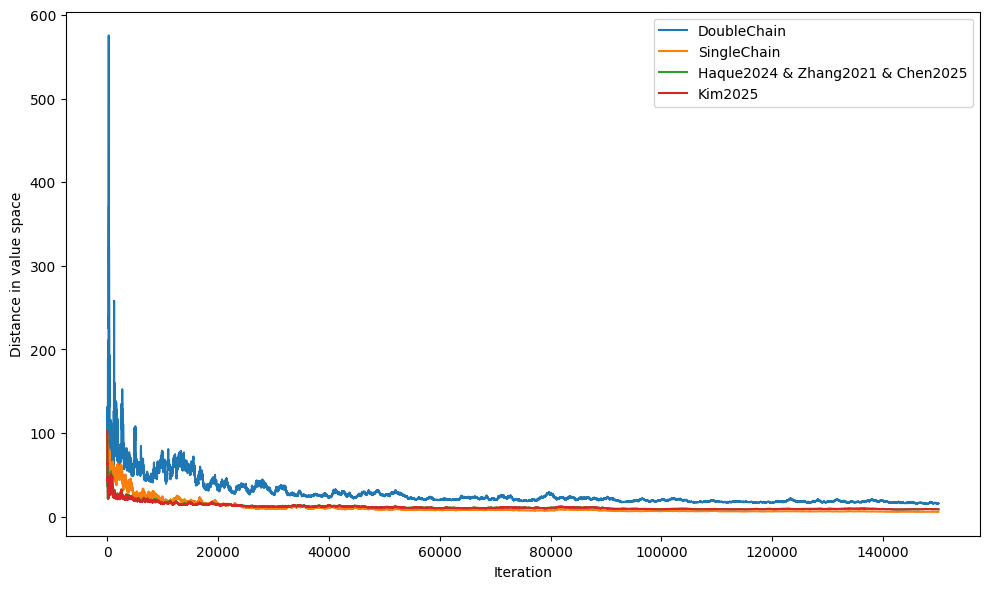}
         \caption{Taxi}
     \end{subfigure}
     \begin{subfigure}{0.32\textwidth}
         \centering
         \includegraphics[scale=0.2]{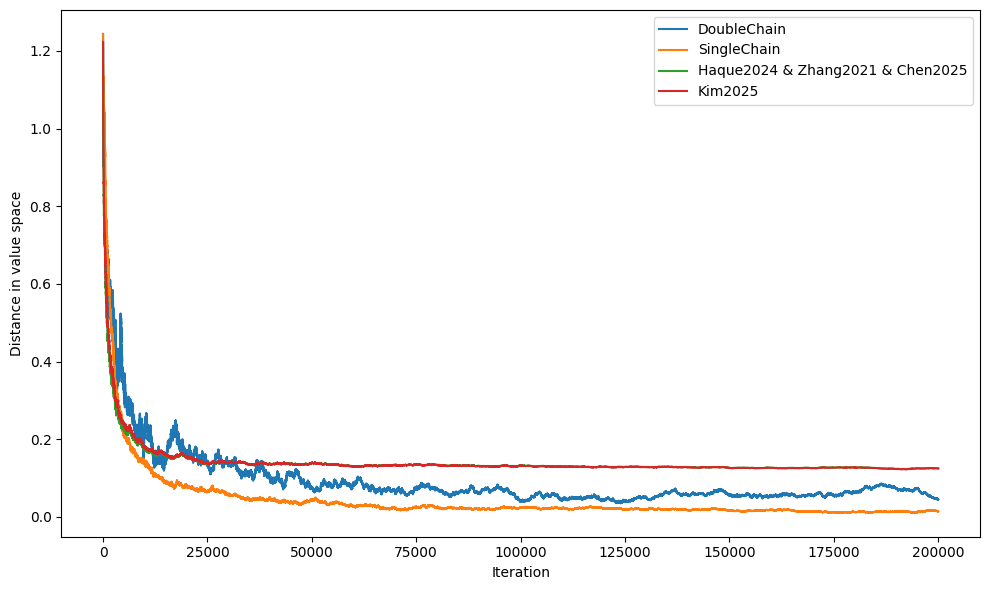}
         \caption{Grid World (5x5)}
     \end{subfigure}
     \begin{subfigure}{0.32\textwidth}
         \centering
         \includegraphics[scale=0.2]{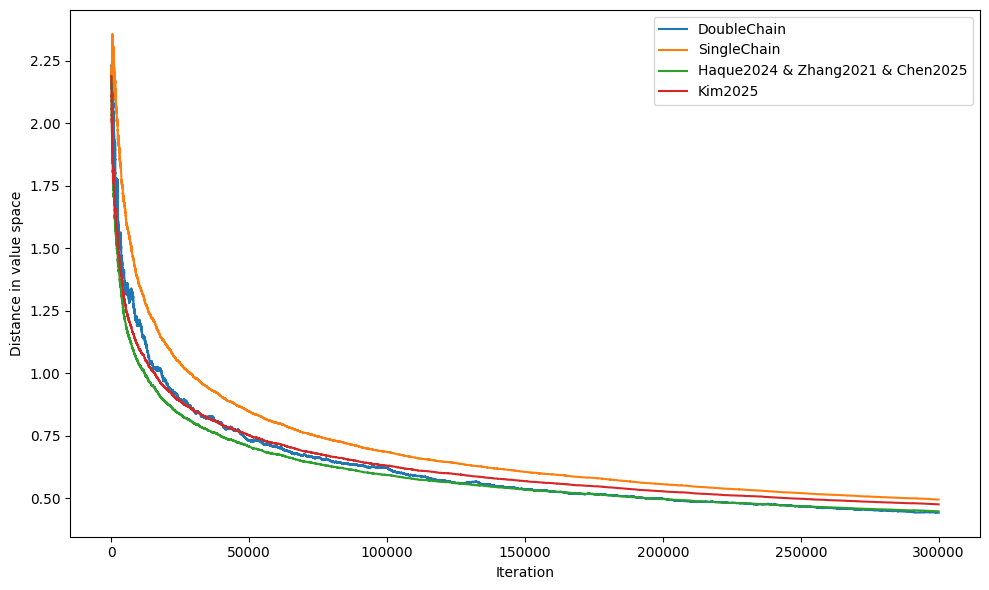}
         \caption{Grid World (10x10)}
     \end{subfigure}
     \begin{subfigure}{0.32\textwidth}
         \centering
         \includegraphics[scale=0.2]{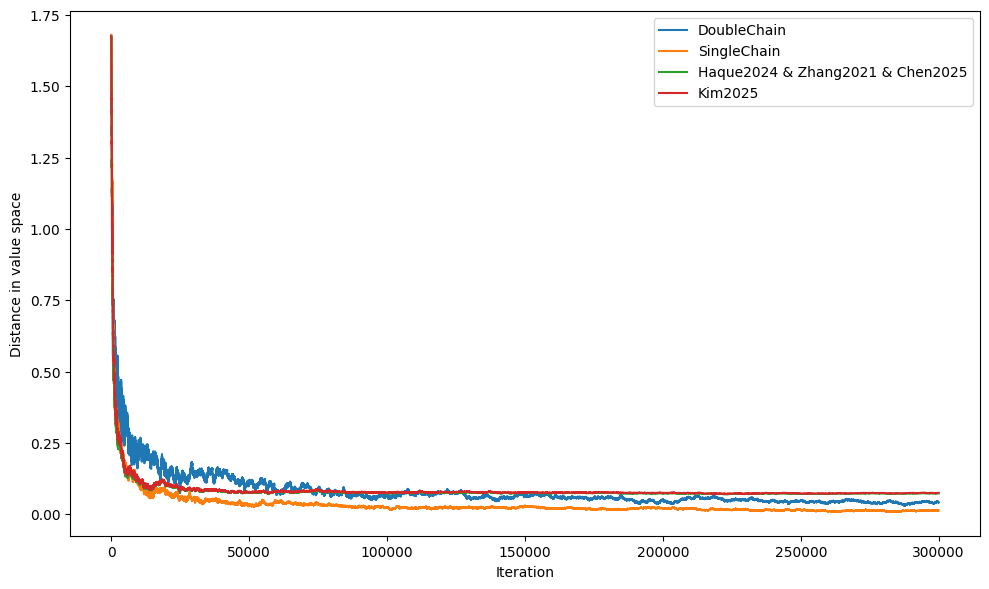}
         \caption{Grid World (2x11)}
     \end{subfigure}
     \begin{subfigure}{0.32\textwidth}
         \centering
         \includegraphics[scale=0.2]{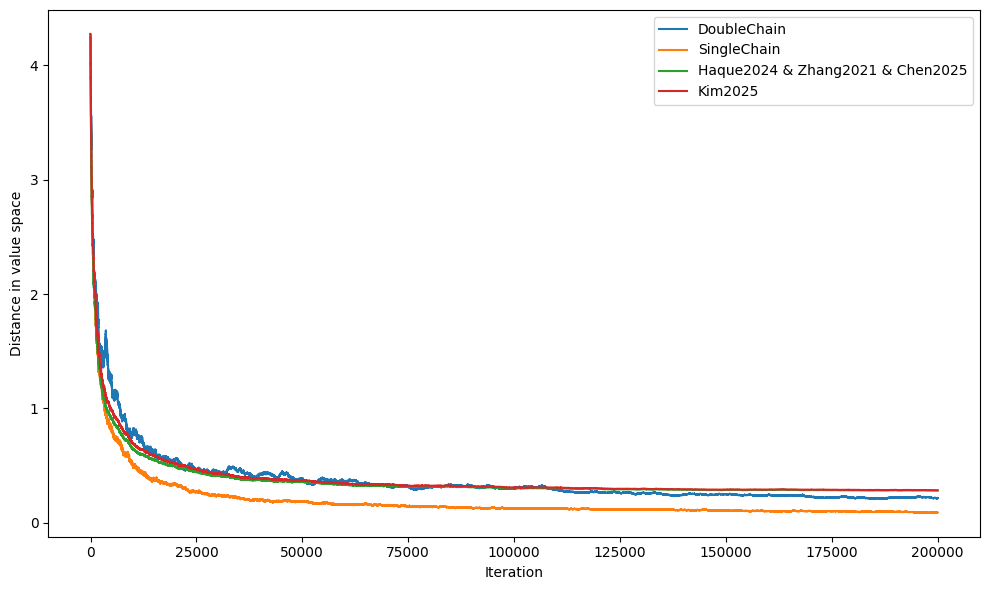}
         \caption{Deep Sea}
     \end{subfigure}
     \begin{subfigure}{0.32\textwidth}
         \centering
         \includegraphics[scale=0.2]{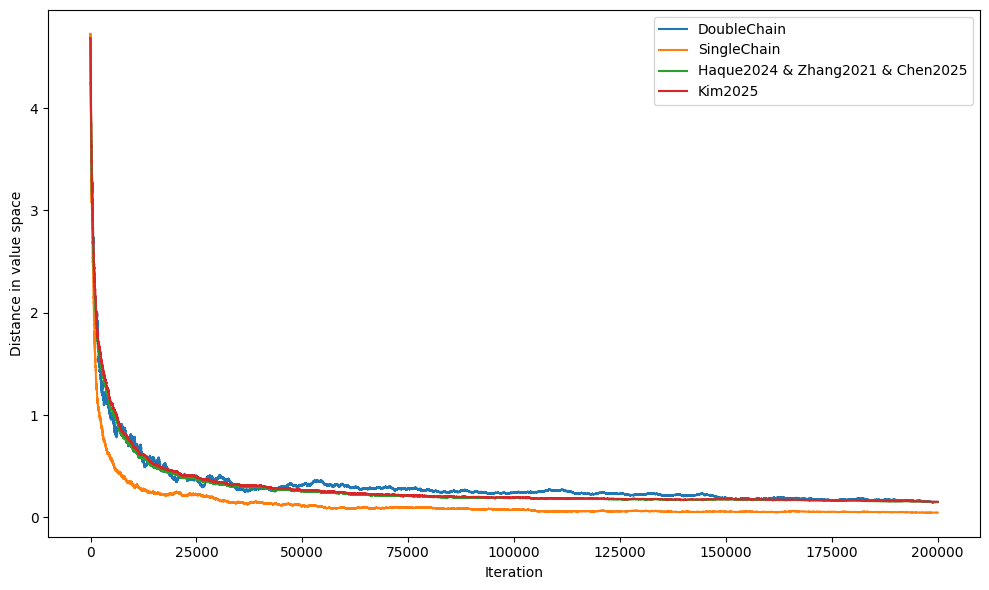}
         \caption{Deep Sea (concave)}
     \end{subfigure}
     \begin{subfigure}{0.32\textwidth}
         \centering
         \includegraphics[scale=0.2]{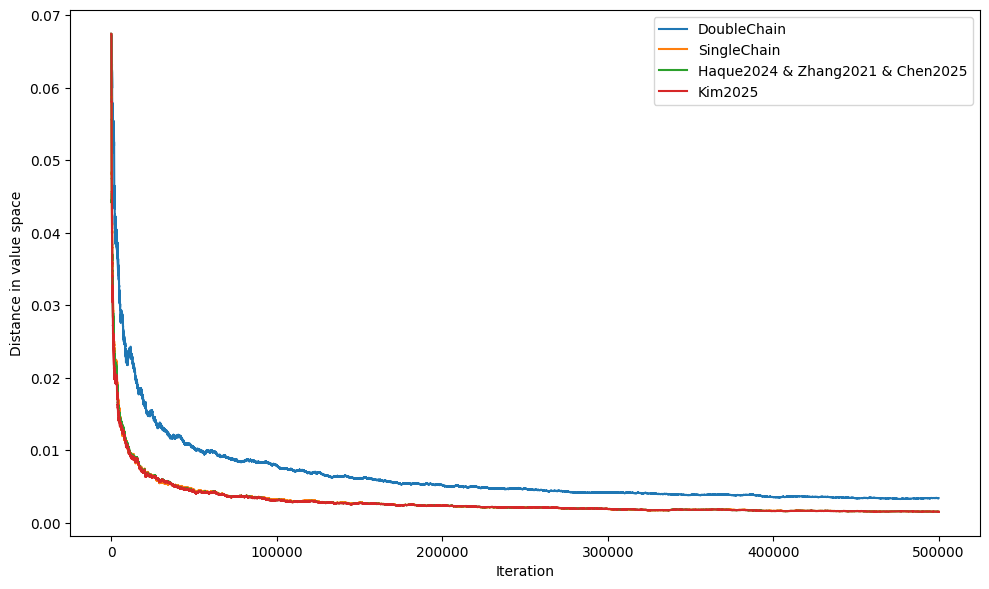}
         \caption{Resource Gathering}
     \end{subfigure}
     \begin{subfigure}{0.32\textwidth}
         \centering
         \includegraphics[scale=0.2]{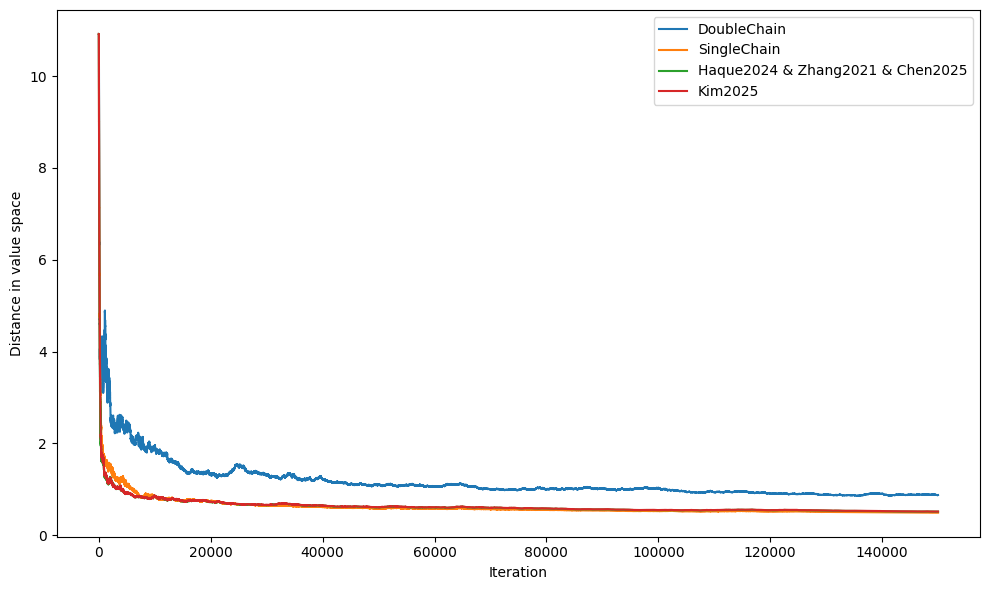}
         \caption{Fruit Tree (depth = 5)}
     \end{subfigure}
     \begin{subfigure}{0.32\textwidth}
         \centering
         \includegraphics[scale=0.2]{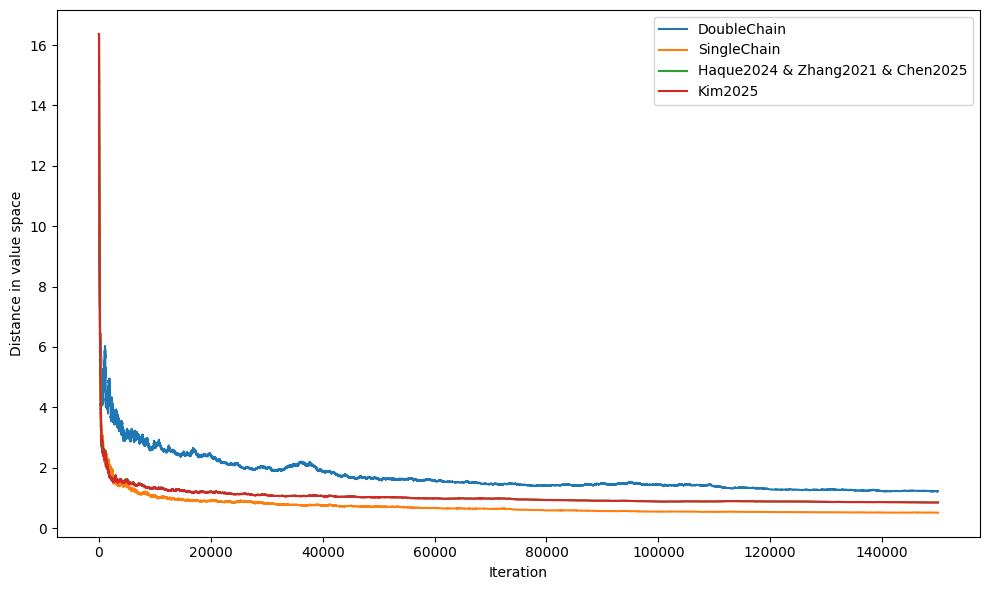}
         \caption{Fruit Tree (depth = 6)}
     \end{subfigure}
     \begin{subfigure}{0.32\textwidth}
         \centering
         \includegraphics[scale=0.2]{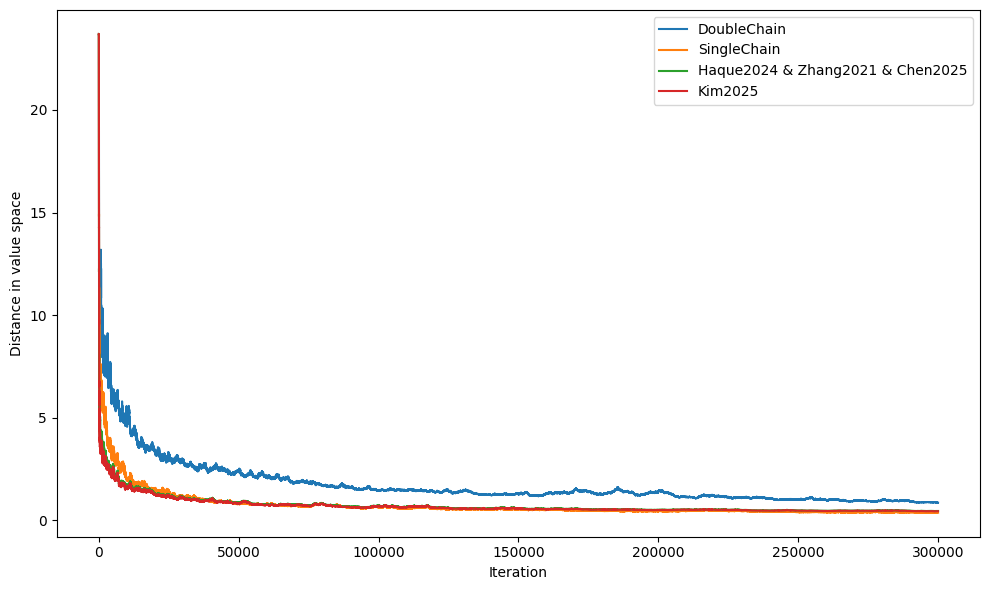}
         \caption{Fruit Tree (depth = 7)}
     \end{subfigure}
\caption{Simulation results}
\end{figure}

\begin{table}
\caption{Condition numbers in different tasks}
\label{table: condition numbers}
\centering
\begin{tabular}{llll}
    \toprule
    Task     & \(\eta_1\)     & \(\eta_2\) & \(\eta_3\) \\
    \midrule
    Random Walk (\(50\)) & \(\mathbf{3.64\times10^{-3}}\)  & \(2.77\times10^{-3}\)  & \(3.04\times10^{-3}\)    \\
    Random Walk (\(100\))  & \(\mathbf{3.68\times10^{-3}}\)  & \(2.63\times10^{-3}\) & \(3.19\times10^{-3}\)    \\
    Random Walk (\(1000\)) & \(\mathbf{1.91\times10^{-4}}\) & \(1.09\times10^{-4}\) & \(1.82\times10^{-4}\)     \\
    Frozen Lake & \(\mathbf{4.81\times10^{-4}}\)& \(4.26\times10^{-4}\) & \(5.20\times10^{-5}\)     \\
    Cliff Walking & \(\mathbf{5.04\times10^{-4}}\) & \(3.27\times10^{-4}\)& \(7.27\times10^{-5}\)    \\
    Taxi & \(\mathbf{2.29\times10^{-4}}\) & \(1.75\times10^{-4}\) & \(1.69\times10^{-5}\)     \\
    Grid World (5x5) & \(\mathbf{4.21\times10^{-3}}\) & \(3.94\times10^{-3}\)& \(1.60\times10^{-3}\)    \\
    Grid World (10x10) & \(\mathbf{4.09\times10^{-4}}\) & \(2.44\times10^{-4}\) & \(7.83\times10^{-5}\)     \\
    Grid World (2x11) & \(3.77\times10^{-3}\) & \(\mathbf{3.88\times10^{-3}}\) & \(1.23\times10^{-3}\)     \\
    Deep sea & \(\mathbf{1.73\times10^{-3}}\) & \(1.19\times10^{-3}\) & \(4.16\times10^{-4}\) \\
    Deep sea (concave) & \(\mathbf{1.92\times10^{-3}}\) & \(1.28\times10^{-3}\) & \(3.72\times10^{-4}\) \\
    Resource Gathering & \(4.35\times10^{-3}\)& \(\mathbf{4.36\times10^{-3}}\)& \(2.07\times10^{-3}\)  \\
    Fruit Tree (depth = 5) & \(\mathbf{1.28\times10^{-5}}\)& \(7.94\times10^{-6}\)& \(5.71\times10^{-6}\) \\
    Fruit Tree (depth = 6) & \(\mathbf{1.07\times10^{-4}}\)& \(6.36\times10^{-5}\)& \(3.75\times10^{-5}\)   \\
    Fruit Tree (depth = 7) & \(\mathbf{1.92\times10^{-4}}\)& \(1.11\times10^{-4}\)& \(5.75\times10^{-5}\)    \\
    \bottomrule
\end{tabular}
\end{table}